\documentclass[11pt]{article}
\usepackage{fullpage}
\usepackage{appendix}
\usepackage[sort,authoryear,round]{natbib}
\usepackage{fancyhdr}
\usepackage[
    format=plain,
    labelfont={it,it},
    textfont=it
]{caption}
\usepackage[stable, bottom, flushmargin]{footmisc}
\usepackage{graphicx}
\usepackage{xcolor}
\usepackage{breakcites}

\usepackage{amsmath,amssymb,amsthm}
\usepackage{aliascnt}
\usepackage{nicefrac}

\definecolor{lightblue}{RGB}{0,102,204}
\definecolor{refblue}{RGB}{0,88,185}

\usepackage{pgfplots}
\pgfplotsset{compat=newest}

\usepackage{wrapfig}
\usepackage{enumitem}
\usepackage{xparse,expl3}
\usepackage{indentfirst}
\usepackage{float}

\usepackage{tikz}
\usetikzlibrary{
    arrows.meta,
    calc,
    fit,
    positioning
}

\usepackage{microtype}
\usepackage{bbm,bm}
\usepackage{xfrac}

\usepackage{algorithm}
\usepackage{algpseudocode}

\usepackage[
    pagebackref,
    colorlinks=true,
    linkcolor=refblue,
    citecolor=lightblue,
    urlcolor=lightblue
]{hyperref}

\usepackage[capitalize]{cleveref}

\newtheorem{theorem}{Theorem}[section]

\newaliascnt{proposition}{theorem}
\newtheorem{proposition}[proposition]{Proposition}
\aliascntresetthe{proposition}

\newaliascnt{lemma}{theorem}
\newtheorem{lemma}[lemma]{Lemma}
\aliascntresetthe{lemma}

\newaliascnt{corollary}{theorem}
\newtheorem{corollary}[corollary]{Corollary}
\aliascntresetthe{corollary}

\newaliascnt{definition}{theorem}
\newtheorem{definition}[definition]{Definition}
\aliascntresetthe{definition}

\newaliascnt{example}{theorem}
\newtheorem{example}[example]{Example}
\aliascntresetthe{example}

\newaliascnt{remark}{theorem}
\newtheorem{remark}[remark]{Remark}
\aliascntresetthe{remark}

\newaliascnt{conjecture}{theorem}

\aliascntresetthe{conjecture}

\newaliascnt{observation}{theorem}

\aliascntresetthe{observation}

\newtheorem{problem}{Open Problem}

\crefname{theorem}{Theorem}{Theorems}
\Crefname{theorem}{Theorem}{Theorems}
\crefname{proposition}{Proposition}{Propositions}
\Crefname{proposition}{Proposition}{Propositions}
\crefname{lemma}{Lemma}{Lemmas}
\Crefname{lemma}{Lemma}{Lemmas}
\crefname{corollary}{Corollary}{Corollaries}
\Crefname{corollary}{Corollary}{Corollaries}
\crefname{definition}{Definition}{Definitions}
\Crefname{definition}{Definition}{Definitions}
\crefname{example}{Example}{Examples}
\Crefname{example}{Example}{Examples}
\crefname{remark}{Remark}{Remarks}
\Crefname{remark}{Remark}{Remarks}
\crefname{conjecture}{Conjecture}{Conjectures}
\Crefname{conjecture}{Conjecture}{Conjectures}
\crefname{observation}{Observation}{Observations}
\Crefname{observation}{Observation}{Observations}
\crefname{problem}{Open Problem}{Open Problems}
\Crefname{problem}{Open Problem}{Open Problems}
\crefname{equation}{Equation}{Equations}
\Crefname{equation}{Equation}{Equations}
\crefname{section}{Section}{Sections}
\Crefname{section}{Section}{Sections}
\crefname{subsection}{Section}{Sections}
\Crefname{subsection}{Section}{Sections}
\crefname{appendix}{Appendix}{Appendices}
\Crefname{appendix}{Appendix}{Appendices}

\newcommand{\nc}{\newcommand}
\newcommand{\rnc}{\renewcommand}

\rnc{\epsilon}{\varepsilon}

\nc{\R}{\mathbb R}
\nc{\Q}{\mathbb Q}
\nc{\Z}{\mathbb Z}
\nc{\N}{\mathbb N}

\let\E\relax
\DeclareMathOperator*{\E}{\mathbb{E}}
\let\P\relax
\DeclareMathOperator*{\P}{\mathbb{P}}

\DeclareMathOperator*{\argmin}{arg\,min}

\DeclareMathOperator*{\im}{im}
\DeclareMathOperator*{\VC}{VC}
\DeclareMathOperator*{\DS}{DS}

\DeclareMathOperator*{\Gdim}{Graph}

\nc{\CA}{\mathcal A}
\nc{\CC}{\mathcal C}
\nc{\CD}{\mathcal D}
\nc{\CF}{\mathcal F}
\nc{\CG}{\mathcal G}
\nc{\CH}{\mathcal H}
\nc{\CI}{\mathcal I}
\nc{\CP}{\mathcal P}
\nc{\CR}{\mathcal R}
\nc{\CS}{\mathcal S}
\nc{\CX}{\mathcal X}
\nc{\CY}{\mathcal Y}
\nc{\CV}{\mathcal V}

\nc{\X}{\mathcal X}
\nc{\Y}{\mathcal Y}
\nc{\Hc}{\mathcal H}
\rnc{\H}{\mathcal H}
\nc{\F}{\mathcal F}
\nc{\D}{\mathcal D}
\nc{\C}{\mathcal C}
\nc{\Rcal}{\mathcal R}
\nc{\Scal}{\mathcal S}
\nc{\A}{\mathcal A}
\nc{\V}{\mathcal V}

\nc{\Cantor}{\CH_\infty}
\nc{\Hpoison}{\mathcal H_{\mathrm{P}}}
\nc{\Hprop}{\mathcal H_{\mathrm{prop}}}
\nc{\Hinc}{\mathcal H_{\mathrm{I}}}
\nc{\Hor}{\mathcal H_{\mathrm{or}}}
\nc{\Hrf}{\mathcal H_{\mathrm{RF}}}

\nc{\defn}[1]{\textbf{#1}}
\nc{\edefn}[1]{\emph{\textbf{#1}}}
\rnc{\t}[1]{\text{#1}}
\nc{\ERM}{\mathrm{ERM}}
\nc{\Trans}{\mathrm{Trans}}
\nc{\h}{\widehat}
\nc{\Unif}{\mathrm{Unif}}
\nc{\xtest}{x_{\mathrm{test}}}
\nc{\xfool}{x_{\mathrm{fool}}}
\rnc{\arg}{\, \cdot \,}

\newcommand{\emp}{\widehat{L}}
\newcommand{\pop}{L}
\newcommand{\vs}{\mathrm{VS}}
\newcommand{\starlabel}{\star}
\newcommand{\dollarlabel}{\$}
\newcommand{\one}{\mathbf{1}}

\newcommand{\restr}[2]{{#1}\!\upharpoonright_{#2}}
\newcommand{\set}[1]{\left\{#1\right\}}

\newcommand{\brac}[1]{\left[#1\right]}
\newcommand{\abs}[1]{\left|#1\right|}

\makeatletter
\renewenvironment{proof}[1][\proofname]{
    \par
    \pushQED{\qed}
    \normalfont
    \topsep\dimexpr 6\p@-\parskip\relax
    \trivlist
    \item[\hskip\labelsep\itshape #1\@addpunct{.}]
    \ignorespaces
}{
    \popQED
    \endtrivlist
    \@endpefalse
}
\makeatother

\newcommand{\email}[1]{\textsf{#1}}

\title{
  \makebox[\textwidth][c]{
    \parbox{1.15\textwidth}{
      \centering
      \fontsize{14pt}{19pt}\selectfont
      \bfseries
    Algorithmic Principles For Multiclass Learning Are Hard To Come By: \\ Limits of Regularization and Proper Learning
    }
  }
}

\author{
Julian Asilis \\ USC \\ \email{asilis@usc.edu} \and 
Shaddin Dughmi \\ USC  \\ \email{shaddin@usc.edu} \and
Vatsal Sharan \\ USC \\ \email{vsharan@usc.edu} \and 
Alec Sun \\ University of Chicago \\ \email{alecsun@uchicago.edu} \and 
Shang-Hua Teng \\ USC \\ \email{shanghua@usc.edu} \and
Chang Wang \\ Northwestern University \\ \email{wc@u.northwestern.edu} 
}

\date{}

\begin{document}

\maketitle

\begin{abstract}
Two of the most fundamental questions in statistical learning theory are the following: which prediction problems are learnable, and how should they be learned? 
For the former, elegant answers often take the form of combinatorial dimensions, notably the VC dimension for binary classification and the DS dimension for multiclass. The latter question has proved considerably more elusive: all known general-purpose multiclass learners rely on intricate orientations of exponentially large one-inclusion structures, and familiar algorithmic principles such as proper learning and regularization remain poorly understood. Motivated by prior work, we ask whether learning reduces to proper learning---possibly over a larger hypothesis class---and whether proper or improper multiclass learning can ultimately be captured by suitable regularizers.

Our primary results answer both questions negatively, resolving three open problems from prior work. First, we exhibit a learnable multiclass problem that cannot be embedded in any properly learnable class, meaning learning cannot be reduced to proper learning by enlarging the hypothesis class. Second, we demonstrate that proper learning can require training error and characterize this phenomenon precisely: every properly learnable class admits a proper learner making $o(m)$ errors on samples of size $m$, but every prescribed sublinear scale $a_m=o(m)$ is necessary for some properly learnable problem. Third, regularization is not a general learner: we exhibit a properly learnable class that cannot be learned by any Structural Risk Minimization (SRM) learner, and a learnable class that cannot be learned by any (weighted) local regularizer. We complement these impossibility results with a positive theory that gives two sufficient conditions for SRM learnability and characterizes SRM representability through integrability of revealed preferences. Together, our results delineate the limits of proper learning and regularization, along with the structural conditions under which regularization succeeds.
\end{abstract}

\clearpage 
\begingroup
\hypersetup{linkcolor=black}
\setlength{\parskip}{0pt}
\setcounter{tocdepth}{2}
\linespread{1.0}\selectfont
\tableofcontents
\endgroup
\clearpage

\section{Introduction}\label{Section:Introduction}

Designing optimal, tractable learners for learnable problems stands as one of statistical learning theory's most central goals. Such a learnable problem is described by a domain $\CX$, label set $\CY$, and hypothesis class of functions $\CH \subseteq \CY^\CX$. A learner receives a training sample $S = \big((x_i, y_i)\big)_{i \leq n}$, with each $x_i$ drawn i.i.d.\ from a marginal distribution $\CD$ over $\CX$ and labeled by a hypothesis $h^* \in \CH$, i.e., $y_i = h^*(x_i)$. The purpose of a learner $A$ is to ingest only the sample $S$ and emit a predictor $A(S): \CX \to \CY$ with high probability of correctly classifying a freshly drawn test point $\xtest \sim \CD$. That is, the learner seeks to minimize 
\[ \pop_\CD \big (A(S) \big) = \P_{\xtest \sim \CD} \Big( A(S)(\xtest) \neq h^*(\xtest) \Big ). \]

\vspace{-0.2 cm}
Binary classification, where one employs the label set $\CY = \{0, 1\}$, enjoys a simple characterization of its (optimal) learners. Empirical Risk Minimization (ERM) learns whenever learning is possible, with nearly-optimal sample complexity, and a simple majority of just 3 ERM learners suffices to attain the optimal rate \citep{aden2024majority,rawal2026majority}.
Perhaps surprisingly, the landscape for multiclass classification, with $\CY$ arbitrary, is strikingly more complex. 

All known general-purpose multiclass learners appeal to abstract orientations of so-called \emph{one-inclusion graphs} (OIGs), which can be exponentially large in the size of the training set $S$, or even infinite should $\CY$ be infinite \citep{brukhim2022characterization,aden2023optimal,pabbaraju2026optimal}. Such reliance on intractable OIGs can be partly explained: Some natural algorithmic templates have been ruled out for multiclass learning. In particular, \citet{DS14} demonstrated the existence of learnable problems that cannot be learned by any \emph{proper} learner, i.e., a learner that always emits a function in the underlying class $\CH$. Subsequently, \citet{asilisunderstanding} exhibited the same failure for any aggregation of a bounded number of proper learners, such as a majority of 3 ERMs.
How, then, can one hope to design a simple algorithmic framework for multiclass learning?

Two paths are suggested by prior work. The first appeals to \emph{regularization}, one of the most fundamental algorithmic templates for learning with considerable successes on both theoretical and empirical fronts.\footnote{Empirically, regularized empirical-risk objectives---most prominently
$\ell_1$ penalties such as the lasso and $\ell_2$ penalties such as ridge
regression or weight decay---are routinely used to control effective model
complexity and encourage generalization
\citep{hoerl1970ridge,tibshirani1996lasso,krogh1991weightdecay,hastie2009elements}.
Theoretically, SRM has been shown to characterize non-uniform learnability
\citep{shalev2014understanding}.}
Recall that a regularizer $\psi \colon \CH \to \R_{\geq 0}$ assigns a score to each hypothesis $h \in \CH$, often thought of as measuring the ``complexity'' of $h$, and that a \emph{Structural Risk Minimization} (SRM) learner then trades off empirical risk with hypothesis complexity in an effort to avoid overfitting.
In multiclass learning, \citet{asilis2024regularization} exhibited optimal learners based upon \emph{unsupervised local regularization}. In this framework, a learner receives a labeled sample $S = \big( (x_i, y_i) \big)_{i \leq n}$ and uses only its unlabeled datapoints $(x_i)_{i \leq n}$ to learn a local regularizer $\psi: \CH \times \CX \to \R_{\geq 0}$. Intuitively, the local regularizer $\psi(h, x)$ measures the ``complexity'' of a function $h$ at the test point $x$, reflecting the fact that hypotheses may act simply in certain regions of the domain and more intricately in others. At test time, for an unlabeled datapoint $\xtest \in \CX$, the learner makes the prediction
\[ A(S)(\xtest) \in \Big \{ h(\xtest) : h \in \argmin_\CH \emp_S(h) + \psi(h, \xtest) \Big \}. \]
Note that the dependence of the regularizer upon the test point $\xtest$ is essential in order to express improper learners. (Intuitively, it permits a learner to ``stitch together'' various hypotheses in $\CH$ when producing a predictor.)

It is, however, not clear whether the unsupervised pre-training stage of the local regularizer can be omitted. Perhaps all that is required to learn a classification problem is an appropriate choice of local regularizer $\psi: \CH \times \CX \to \R_{\geq 0}$ encoding the ``complexity'' of hypotheses at particular regions in the domain.
Precisely this question was recently articulated by \citet{asilis-open-problem}.

\begin{problem}[{\citep{asilis-open-problem}}]\label{problem1:local-regularizer}
Can all learnable multiclass problems $\CH$ be learned by a local regularizer? If so, with (nearly) optimal sample complexity?
\end{problem}

A second path towards simplicity in multiclass learning is to invoke an additional assumption on the underlying hypothesis class $\CH$. Perhaps the most natural is that of proper learnability, i.e., to assume that $\CH$ can be learned while only emitting hypotheses in $\CH$ itself.
In this setting, the most ambitious goal would be to demonstrate that classic regularization is all one needs. This question, too, was previously raised and remains open.

\begin{problem}[{\citep{asilis25b-proper}}]\label{problem2:proper-SRM-equivalence}
Let $\CH$ be a properly learnable multiclass problem. Must $\CH$ be learnable by an SRM learner?
\end{problem}

Allow us to briefly remark upon a subtlety related to \cref{problem2:proper-SRM-equivalence}. In the design of improper learners for learnable multiclass problems, 
one makes crucial use of the fact that learnability amounts to finiteness of the DS dimension \citep{brukhim2022characterization,pabbaraju2026optimal}. That is, finiteness of $\CH$'s DS dimension provides an important handhold in analyzing successful learners, by way of $\CH$'s one-inclusion graphs. Proper learnability, however, is a fundamentally stranger beast. \citet{asilis25b-proper} demonstrated that it cannot be characterized by any combinatorial dimension, and that it can even be logically undecidable, i.e., independent of the ZFC axioms. This complicates the task of resolving \cref{problem2:proper-SRM-equivalence} in either direction, as merely establishing the proper learnability of a hypothesis class can be surprisingly involved (or even undecidable!).

Finally, there is a curious commonality shared by existing hypothesis classes designed to be improperly learnable yet not properly learnable: they can easily be made properly learnable with the addition of a small number of functions to the class. This includes the first Cantor class of \citet{DS14}, along with the EMX learning problem of \citet{ben2019learnability} when viewed as a classification problem as in \citet{asilis25b-proper}.
This raises an alluring possibility: perhaps all learnable classes can be embedded within properly learnable classes. 

\begin{problem}[{\citep{asilis25b-proper}}]\label{problem3:envelope-proper}
Is every learnable hypothesis class $\CH$ contained within a properly learnable class $\CH_{\mathrm{prop}} \supseteq \CH$?
\end{problem}

Note that a positive resolution to \cref{problem3:envelope-proper} would establish, roughly speaking, that general multiclass learning can be reduced to proper learning. If it were accompanied by a positive resolution to \cref{problem2:proper-SRM-equivalence}, demonstrating that proper learning reduces to SRM, it would then be the case that general multiclass learning reduces to SRM! 

Unfortunately, our results demonstrate that the picture is not so rosy: multiclass learning resists simple algorithmic principles, even under the promise of proper learnability. We elaborate upon our impossibility results, and upon our characterizations of SRM learnability, in \Cref{Subsection:results}.

\begin{figure}[t]
    \centering
    \resizebox{\linewidth}{!}{\begin{tikzpicture}[x=0.88cm,y=0.88cm,font=\small,
 line join=round,line cap=round]
\definecolor{learnblue}{RGB}{42,98,145}
\definecolor{embeddablegreen}{RGB}{23,119,91}
\definecolor{embeddablefill}{RGB}{224,239,229}
\definecolor{localblue}{RGB}{219,237,246}
\definecolor{weightedfill}{RGB}{251,237,211}
\definecolor{weightedline}{RGB}{172,111,32}
\definecolor{consistentred}{RGB}{177,63,70}
\path[use as bounding box] (-0.1,-0.1) rectangle (17,8.6);
\fill[gray!3,rounded corners=7pt] (0,0) rectangle (16.95,8.5);
\fill[embeddablefill,rounded corners=6pt] (0.66,0.4) rectangle (13.25,8.08);
\fill[white,rounded corners=5pt] (0.94,0.9) rectangle (10.95,6.60);
\fill[weightedfill,rounded corners=5pt] (0.22,2.25) rectangle (16.7,7.38);
\fill[localblue,rounded corners=5pt] (0.44,3.65) rectangle (16.45,7.04);
\draw[weightedline,line width=1.15pt,rounded corners=5pt]
 (0.22,2.25) rectangle (16.7,7.38);
\draw[learnblue,line width=1pt,rounded corners=5pt]
 (0.44,3.65) rectangle (16.45,7.04);
\draw[gray!70,line width=1.15pt,rounded corners=7pt]
 (0,0) rectangle (16.95,8.5);
\draw[embeddablegreen,line width=1pt,rounded corners=6pt]
 (0.66,0.4) rectangle (13.25,8.08);
\draw[black!80,line width=1pt,rounded corners=5pt]
 (0.94,0.9) rectangle (10.95,6.60);
\draw[consistentred,line width=1.15pt] (4.6,3.65) ellipse (2.4 and 2.35);
\draw[violet!70!black,line width=1.15pt] (7.1,3.65) ellipse (2.4 and 2.35);
\draw[black!75,fill=localblue,line width=0.85pt]
 (5.85,4.55) ellipse (0.62 and 0.5);
\node[anchor=east,font=\bfseries,text=black!60] at (16.4,8.15)
 {$\mathsf I$: learnable};
\node[anchor=east,font=\bfseries,text=embeddablegreen] at (13.0,7.72)
 {$\mathsf B$: properly embeddable};
\node[anchor=west,font=\bfseries] at (1.27,6.32)
 {$\mathsf P$: properly learnable};
\node[text=consistentred,font=\footnotesize\bfseries,align=center] at (4.10,5.00)
 {$\mathsf C$\\consistent\\proper};
\node[text=violet!70!black,font=\footnotesize\bfseries,align=center] at (7.70,5.00)
 {$\mathsf W$\\weighted\\SRM};
\node[align=center,text=learnblue,font=\bfseries] at (14.95,5.45)
 {$\mathsf L$\\hard local};
\node[align=center,text=weightedline,font=\bfseries] at (14.95,2.94)
 {$\mathsf{WL}$\\weighted local};
\node[align=center,font=\scriptsize\bfseries] at (5.85,4.62)
 {$\mathsf E$\\ERM};
\end{tikzpicture}}
    \caption[Containments among multiclass learning principles]{Multiclass hypothesis classes in 
    realizable PAC learning. A point in space is a hypothesis class, i.e., one learning problem.
    Each colored region corresponds to the set of all learning problems that can be learned by a
    particular class of learners. 
    The regions are: $\mathsf I$, learnable by an arbitrary improper learner;
    $\mathsf B$, contained in a properly learnable class; $\mathsf P$, properly
    learnable; $\mathsf C$, learnable by an interpolating proper learner;
    $\mathsf E$, learned by every ERM; $\mathsf W$, learnable by weighted SRM;
    $\mathsf L$, by local regularization; and
    $\mathsf{WL}$, by weighted local regularization. Details
    are given in \cref{Subsection:regularization-models}.
    Our results demonstrate that each cell in the diagram is non-empty. See 
    \cref{app:figure-witnesses}.}
    \label{fig:multiclass-containments}
\end{figure}

\subsection{Results}\label{Subsection:results}

Our primary results resolve \Cref{problem1:local-regularizer,problem2:proper-SRM-equivalence,problem3:envelope-proper}
in the negative. Furthermore, they provide a complete account of the set-theoretic relationships 
between classes of problems that can be solved by proper and/or regularization-based learners. See \Cref{fig:multiclass-containments}; our results demonstrate that every cell in the figure is non-empty.
Overall, we divide our contributions into three groups.

\begin{itemize}[leftmargin=1.35em,itemsep=0.9em,topsep=0.5em]

    \item \textbf{Proper learning and noninterpolation.}
    We first show that general multiclass learning does not reduce to proper learning, even after enlarging the learner's output space. 
    That is, there exists a learnable class that cannot be embedded into any properly learnable envelope, resolving \cref{problem3:envelope-proper} in the negative (\cref{thm:no-proper-completion}).  We next construct a topologically well-behaved hypothesis class that is properly learnable but cannot be learned by any interpolating proper learner (\cref{thm:no-erm}).  Finally, we give a tight quantitative account of this phenomenon: every properly learnable class admits a proper learner making $o(m)$ errors on every realizable sample of size $m$, and every prescribed sublinear scale $a_m=o(m)$ is necessary for some properly learnable problem (\cref{thm:empirical-upper,thm:empirical-lower}).  Thus proper learners never need to sacrifice a constant fraction of their training sample, but within the sublinear regime essentially any number of errors can be forced.

    \item \textbf{Limits of global and local regularization.}
    We first demonstrate that regularization-based learners can fail even on problems that are properly learnable. Namely, we exhibit a properly learnable class for which every weighted SRM objective of the form $h \longmapsto \emp_S(h)+\lambda(S) \cdot \psi(h)$ fails to learn, resolving \cref{problem2:proper-SRM-equivalence} in the negative (\cref{thm:weighted-srm-fails}).  We next consider whether \emph{local regularizers}, which may combine different hypotheses at different test points and thereby express improper learners, succeed whenever improper learning is possible.  Here we again give a negative answer: we construct a learnable class that lies beyond the reach of local regularization, resolving the PAC version of \cref{problem1:local-regularizer} (\cref{thm:local-regularization-fails}). We note that the concurrent, independent work of \citet{hou2026local} also exhibited such a counterexample; we compare and contrast our approach in \cref{Subsection:related-work}.
    In \cref{app:figure-witnesses}, we even consider a more powerful \emph{weighted} form of local regularization and demonstrate that it, too, falls short of characterizing multiclass learnability. 

    \item \textbf{Sufficient conditions for SRM.}
    We complement these impossibility results with two positive routes to SRM learnability.  First, we show that hard SRM succeeds whenever non-trivial samples admit a consistent hypothesis from a fixed finite family of hypotheses; more generally, the same principle applies when the relevant hypotheses lie inside a statistically simple sublevel of the regularizer (\cref{thm:fallback-core,thm:finite-ambiguous-range,thm:ambiguity-localized}).  Second, we introduce an order-sensitive disagreement dimension that similarly measures the complexity of sublevel sets of the regularizer. We prove that finiteness of this dimension suffices for hard-SRM learnability, though it is not necessary (\cref{thm:odd-suffices,ex:revealing-fallback}).  Finally, we characterize when the behavior of a particular learner can be witnessed by a regularizer. Here we find that a consistent learner is representable by SRM precisely when the pairwise preferences revealed by its choices are acyclic, while weighted SRM representability over finite systems is equivalent to positivity of every revealed-preference cycle (\cref{thm:hard-integrability,thm:weighted-integrability}).

\end{itemize}

\subsection{Techniques}\label{Subsection:techniques}

Our proofs appeal to diverse combinatorial tools: diagonal arguments on rooted trees, projective-plane expansion, and directed-triangle counting.  Let us describe the key ideas.

\begin{figure}[t]
    \centering
    \resizebox{0.9\textwidth}{!}{%
    \begin{tikzpicture}[
        x=1cm,
        y=1cm,
        >=Latex,
        every node/.style={font=\small},
        treevertex/.style={
            circle,
            draw=black!60,
            fill=white,
            minimum size=5.5mm,
            inner sep=0pt
        },
        pathvertex/.style={
            circle,
            draw=blue!65!black,
            fill=blue!10,
            minimum size=6.2mm,
            inner sep=0pt,
            line width=0.9pt
        },
        ghostvertex/.style={
            circle,
            draw=black!35,
            fill=black!3,
            minimum size=4.4mm,
            inner sep=0pt
        },
        path/.style={
            -{Latex[length=2mm]},
            draw=blue!65!black,
            line width=1.25pt
        },
        branch/.style={
            -{Latex[length=1.6mm]},
            draw=black!35,
            line width=0.65pt
        },
        block/.style={
            rounded corners=3pt,
            draw=black!55,
            fill=black!1,
            minimum height=6mm,
            inner sep=3pt
        },
        chosen/.style={
            circle,
            draw=blue!65!black,
            line width=1pt,
            minimum size=4.2mm,
            inner sep=0pt
        },
        dot/.style={
            circle,
            fill=black!80,
            minimum size=1.8mm,
            inner sep=0pt
        },
        guide/.style={
            draw=black!35,
            densely dotted,
            line width=0.7pt
        }
    ]

    \node[pathvertex] (v0) at (0,3.75) {$v_0$};
    \node[pathvertex] (v1) at (1.55,2.50) {$v_1$};
    \node[pathvertex] (v2) at (3.10,1.25) {$v_2$};
    \node[pathvertex] (v3) at (4.65,0.00) {$v_3$};

    \draw[path]
        (v0)
        --
        node[above right=-1pt] {$b_1$}
        (v1);

    \draw[path]
        (v1)
        --
        node[above right=-1pt] {$b_2$}
        (v2);

    \draw[path]
        (v2)
        --
        node[above right=-1pt] {$b_3$}
        (v3);

    \node[ghostvertex] (a01) at (-1.65,2.42) {};
    \node[ghostvertex] (a02) at (-0.25,2.12) {};
    \node[ghostvertex] (a03) at (0.65,2.02) {};

    \draw[branch] (v0) -- (a01);
    \draw[branch] (v0) -- (a02);
    \draw[branch] (v0) -- (a03);

    \node[ghostvertex] (a11) at (0.25,1.08) {};
    \node[ghostvertex] (a12) at (1.30,0.78) {};
    \node[ghostvertex] (a13) at (2.15,0.66) {};

    \draw[branch] (v1) -- (a11);
    \draw[branch] (v1) -- (a12);
    \draw[branch] (v1) -- (a13);

    \node[ghostvertex] (a21) at (2.13,-0.18) {};
    \node[ghostvertex] (a22) at (3.02,-0.48) {};
    \node[ghostvertex] (a23) at (3.80,-0.58) {};

    \draw[branch] (v2) -- (a21);
    \draw[branch] (v2) -- (a22);
    \draw[branch] (v2) -- (a23);

    \node at (5.42,-0.62) {$\ddots$};
    \draw[path,dashed] (v3) -- (5.25,-0.48);

    \node[
        block,
        minimum width=4.6cm
    ] (X0) at (8.15,3.75) {};

    \node[
        anchor=west,
        font=\footnotesize
    ] at (6.00,4.22) {$X_{v_0}$};

    \foreach \dx in {-1.80,-1.30,-0.80,-0.30,0.20,0.70,1.20,1.70} {
        \node[dot] at ([xshift=\dx cm]X0.center) {};
    }

    \node[
        block,
        minimum width=4.6cm
    ] (X1) at (8.55,2.50) {};

    \node[
        anchor=west,
        font=\footnotesize
    ] at (6.15,2.97) {$X_{v_1}$};

    \foreach \dx in {-1.80,-1.30,-0.80,-0.30,0.20,0.70,1.20,1.70} {
        \node[dot] at ([xshift=\dx cm]X1.center) {};
    }

    \node[
        block,
        minimum width=4.6cm
    ] (X2) at (8.95,1.25) {};

    \node[
        anchor=west,
        font=\footnotesize
    ] at (6.35,1.72) {$X_{v_2}$};

    \foreach \dx in {-1.80,-1.30,-0.80,-0.30,0.20,0.70,1.20,1.70} {
        \node[dot] at ([xshift=\dx cm]X2.center) {};
    }
    
    \node[chosen] at ([xshift=-0.80cm]X0.center) {};
    \node[font=\footnotesize,anchor=south]
        at ([xshift=-0.80cm,yshift=0.18cm]X0.center) {$u_1$};

    \node[chosen] at ([xshift=0.70cm]X1.center) {};
    \node[font=\footnotesize,anchor=south]
        at ([xshift=0.70cm,yshift=0.18cm]X1.center) {$u_2$};
    
    \node[chosen] at ([xshift=-0.80cm]X2.center) {};
    \node[font=\footnotesize,anchor=south]
        at ([xshift=-0.80cm,yshift=0.18cm]X2.center) {$u_3$};

    \draw[guide] (v0.east) -- (X0.west);
    \draw[guide] (v1.east) -- (X1.west);
    \draw[guide] (v2.east) -- (X2.west);

    \node[font=\normalsize] at (-0.9,4.55) {$T_d$};
    \node[font=\normalsize] at (8.6,4.75) {$|X_{v_i}|=d$};

    \end{tikzpicture}
    }
    \caption{
        The tree underlying the no-envelope construction.
        Each node $v$ is associated with a block $X_v$, and its
        children are indexed by all binary labelings of that block.
        A target follows the highlighted finite path, selecting one public
        coordinate $u_i$ in each visited non-terminal block and using a private label
        elsewhere. In the lower bound, each continuation $b_i$ is a $d$-bit pattern missing from the learner's image at the current block.
    }
    \label{fig:no-envelope-tree}
\end{figure}

\vspace{-0.4 cm}
\paragraph{No properly learnable envelope.}
We first ask whether an arbitrary learnable multiclass problem $\H$ can be embedded into a properly learnable envelope $\H_{\mathrm{prop}}\supseteq\H$.  Were this true, then general learning would reduce to proper learning in a precise sense: one could properly learn $\H_{\mathrm{prop}}$, thereby obtaining a learner $A$ for $\H$ whose image is contained in a properly learnable class.  We rule out this possibility by constructing a learnable class $\H$ for which every successful learner $A$ must satisfy $\DS\bigl(\operatorname{im}(A)\bigr)=\infty$.
In particular, the image of $A$ cannot be contained in \emph{any} learnable class, even improperly so.

The construction (illustrated in \Cref{fig:no-envelope-tree}) employs, for each $d \geq 1$, an infinite complete $2^d$-ary tree $T_d$.\footnote{$T_d$ shares some features of a Vapnik-Chervonenkis-Littlestone (VCL) tree \citep{bousquet2021theory}. Each node carries a block whose binary labelings index its children. However, its shape is slightly different, as within $T_d$ it has a fixed block size and a fixed branching factor at every depth.}
To every node
\[
    v=(b_1,\ldots,b_k),
    \qquad
    b_i\in\{0,1\}^d,
\]
at depth $k$, we associate a fresh block of domain points $X_v=\{x_{v,1},\ldots,x_{v,d}\}$.
The node $v$ has one child $(v,b)$ for every binary labeling $b\in\{0,1\}^d$ of $X_v$.  A target hypothesis chooses a finite path
\[
    \varnothing=v_0\prec v_1\prec\cdots\prec v_\ell
\]
and one coordinate $u_i\in[d]$ from each non-terminal block along that path.  At the selected point $x_{v_{i-1},u_i}$, it predicts the public bit $b_i(u_i)\in\{0,1\}$; at every other point, it predicts a private label identifying the entire target.

First note that the class can be improperly learned: if a private label is present in the sample, then the target function has been revealed. Otherwise, the sample is supported on public labeled points, which reveal a branch prefix---as each node $v_k$ records the earlier vectors $b_1,\ldots,b_k$. An improper learner can simply stitch together the public predictions from the longest such prefix, emitting a default label elsewhere. Informally, this emitted classifier can err only on an unobserved private region or the as-yet-unseen tail of the path, both of which have small mass whenever the sample has missed them.

Now let $A$ be an arbitrary successful learner for $\Hc$. We show that $\DS\bigl(\operatorname{im}(A)\bigr)=\infty$. To this end, it suffices to show that for each $d \in \N$,  $\im (A)$ contains a $d$-dimensional binary hypercube of functions, e.g., the behaviors $\{0, 1\}^d$ on a set of $d$ unlabeled datapoints.\footnote{In fact, it would suffice merely to exhibit \emph{pseudocubes} of arbitrarily large dimension, rather than the (binary) hypercubes we exhibit.}
Fix one such $d \in \N$, and consider the $d$ points in the block associated with each node of $T_d$. There are two possibilities. If at some node, every binary labeling of the block is realized by some hypothesis in $\operatorname{im}(A)$, then the image of $A$ contains a full $d$-dimensional binary cube, meaning $\DS\bigl(\operatorname{im}(A)\bigr) \geq d$. Otherwise, every node has some binary labeling of its block that is not realized by any hypothesis in $\operatorname{im}(A)$. Since the children of a node are indexed by all binary labelings of its block, we may follow the corresponding child.
Repeating this procedure produces an infinite path $P$ along which the image of $A$ misses a (prescribed) labeling at every depth.

The second possibility is incompatible with the success of $A$. To see why, fix a sample size $m = |S|$ and take a prefix of $P$ of length $2m$, say $P_{2m}$. At each depth of $P_{2m}$, choose one of the $d$ points uniformly at random as public points. Then construct the hard distribution by distributing mass uniformly across the chosen public points. Clearly, a sample of size $m$ leaves at least half of the depths in $P_{2m}$ unseen. Then at any unseen depth, the learner's output must disagree with the prescribed labeling on at least one point in the block. Since the selected public point at that depth remains uniformly random among those $d$ points, the learner errs with probability at least $\frac 1d$. As $m$ is arbitrary, this implies that $A$ does not attain vanishing expected error, contradicting the PAC guarantee.

Thus only the first possibility can occur. That is, for every $d$, some block must contain a full $d$-dimensional binary cube in the image of $A$, meaning $\DS\bigl(\operatorname{im}(A)\bigr)=\infty$ as desired.

\vspace{-0.4 cm}
\paragraph{Proper learning beyond interpolation.}
To separate proper learning from interpolation, our starting point is the first Cantor class of \citet{DS14}.  Each finite block $\CX_n$ supports the class
\[
    \CH_n
    =
    \left\{
        h_A:
        A\subseteq\CX_n,\ 
        |A|=\frac{|\CX_n|}{2}
    \right\},
    \qquad
    h_A(x)
    =
    \begin{cases}
        \lambda_A, & x\in A,\\
        \starlabel, & x\in\CX_n\setminus A,
    \end{cases}
\]
where $\lambda_A$ is a private label unique to $A$.  Taking $\CX=\bigsqcup_{n\in\N}\CX_n$
and extending every hypothesis by the fresh label $\dollarlabel$ outside its own block gives the full class. Thus each hypothesis outputs a unique identifier on one half of its block and the uninformative label $\starlabel$ on the other.  The class is easily learned improperly: either the sample contains a private label that definitively reveals the target, or the learner may safely output the simple predictor that is constantly $\starlabel$ on the active block and $\dollarlabel$ elsewhere.  A proper learner, however, can be forced to guess an arbitrary half-set $A\subseteq\CX_n$ using only draws from its complement.  Choosing $|\CX_n|$ large relative to $|S|$ therefore destroys proper learnability.

We augment this construction by adding a small number of ``poisoned'' hypotheses to $\CH_n$ that each output the $\star$ label on all of $\CX_n$, save for a distinguished marker point. 
In particular, we first carve a small marker region $P_n$ within $\CX_n$, 
\[
    P_n=\{z_1,\ldots,z_{r_n}\}\subseteq\CX_n,
    \qquad
    r_n\longrightarrow\infty,
    \qquad
    |P_n|\ll|\CX_n|.
\]
We then require each private half-set $A$ to avoid $P_n$, so that every Cantor hypothesis $h_A$ assigns the common label $\starlabel$ to the marker region.  For every marker $z_i\in P_n$, we then introduce an additional hypothesis $h_{z_i}$ satisfying
\[
    h_{z_i}(z_i)=\lambda_{z_i},
    \qquad
    h_{z_i}(x)=\starlabel
    \quad
    \text{for every }x\in\CX_n\setminus\{z_i\}.
\]
Each $h_{z_i}$ is thus an almost-all-$\starlabel$ hypothesis, deliberately poisoned at one marker.

\begin{figure}[!t]
    \centering
    \begin{tikzpicture}[
        x=1cm,
        y=1cm,
        line cap=round,
        line join=round,
        sample/.style={
            circle,
            fill=black!90,
            inner sep=1.35pt
        },
        regionlabel/.style={font=\small},
        masslabel/.style={font=\small}
    ]

    \def\OuterPath{
        (0.55,2.18)
        .. controls (0.28,3.20) and (1.08,4.18) .. (2.62,4.42)
        .. controls (3.78,4.60) and (4.70,4.23) .. (5.72,4.24)
        .. controls (7.08,4.25) and (8.02,3.62) .. (8.40,2.65)
        .. controls (8.78,1.67) and (8.32,0.52) .. (7.28,-0.06)
        .. controls (6.08,-0.74) and (4.16,-0.64) .. (2.72,-0.46)
        .. controls (1.34,-0.29) and (0.66,0.77) .. (0.55,2.18)
        -- cycle
    }

    \def\ABoundary{
        (1.50,-0.72)
        -- (1.78,-0.46)
        .. controls (2.40,0.16) and (3.08,0.78) .. (4.10,1.44)
        .. controls (5.08,2.06) and (6.18,2.68) .. (7.58,3.07)
        -- (8.28,3.28)
    }

    \def\PBoundary{
        (0.42,3.95)
        -- (1.03,3.63)
        .. controls (1.58,3.48) and (1.92,3.12) .. (2.08,2.66)
        .. controls (2.22,2.22) and (2.16,1.82) .. (1.98,1.40)
        .. controls (1.82,1.06) and (1.52,0.82) .. (1.12,0.66)
        -- (0.80,0.52)
    }

    \begin{scope}
        \clip \OuterPath;

        \path[fill=blue!9]
            \ABoundary
            -- (9.40,3.28)
            -- (9.40,-1.20)
            -- (-0.30,-1.20)
            -- (-0.30,-0.72)
            -- cycle;
    \end{scope}

    \begin{scope}
        \clip \OuterPath;

        \path[fill=orange!18]
            \PBoundary
            -- (-0.70,0.52)
            -- (-0.70,5.00)
            -- cycle;
    \end{scope}

    \begin{scope}
        \clip \OuterPath;

        \draw[
            black!70,
            line width=0.85pt,
            densely dashed
        ]
            \ABoundary;

        \draw[
            black!75,
            line width=0.85pt
        ]
            \PBoundary;
    \end{scope}

    \foreach \x/\y in {
        1.20/3.30,
        1.58/3.02,
        0.96/2.82,
        1.74/2.54,
        1.02/2.02,
        1.66/1.76,
        1.26/1.46,
    } {
        \node[sample] at (\x,\y) {};
    }

    \foreach \x/\y in {
        2.62/3.55,
        4.06/3.56,
        5.78/3.56,
        6.58/3.68,
    } {
        \node[sample] at (\x,\y) {};
    }

    \foreach \x/\y in {
        2.42/2.98,
        3.98/2.84,
        6.50/3.02,
        4.62/2.42
    } {
        \node[sample] at (\x,\y) {};
    }

    \foreach \x/\y in {
        3.18/1.54,
        2.38/1.34,
        3.06/1.22,
        1.5/0.6,
    } {
        \node[sample] at (\x,\y) {};
    }

    \node[font=\normalsize] at (4.48,4.78) {$X_n$};
    \node[regionlabel] at (6.45,0.92) {$A$};
    \node[regionlabel] at (1.34,2.42) {$P_n$};

    \draw[
        black!82,
        line width=1.0pt
    ]
        \OuterPath;

    \end{tikzpicture}

    \caption{
        The hard distribution for the poisoned first-Cantor construction.
        The target $h_{A}$ predicts the private label $\lambda_{A}$ on the
        hidden half-set $A$, and the common label $\starlabel$ on
        $X_n\setminus A$.
        The marker region $P_n\subseteq X_n\setminus A$ receives total mass $1/4$,
        while the remaining common-label region receives mass $3/4$.
        The sample therefore hits every marker in $P_n$, ruling out the poisoned
        fallbacks, while revealing only a sparse portion of the common-label region.
    }
    \label{fig:poisoned-cantor-technique}
\end{figure}

When the sample contains only $\starlabel$ labels, a successful proper learner can output any hypothesis associated with one of the least frequently observed markers. If a marker is absent from the sample, i.e., $S \subseteq \CX_n \setminus P_n$ and all labels are $\starlabel$, then this choice is interpolating. However, if every marker $p \in P_n$ appears in $S$ with the label $\starlabel$, then the proper learner sacrifices the least frequently observed marker $p_{\min} \in P_n$, and outputs $h_{p_{\min}}$. (See \Cref{fig:poisoned-cantor-technique}.)
An elementary probabilistic argument can be used to show that, because $|P_n| = \omega(1)$, $p_{\min}$ has population mass $o(1)$ with high probability. On the other hand, when the sample contains any non-$\starlabel$, then the true labeling function has been identified by its private label, guaranteeing 0 test error. Thus the described proper learner, which does not interpolate, is indeed a PAC learner.

To see that such noninterpolation is necessary, consider the process which selects $h_A$ uniformly at random and assigns $1/4$ mass uniformly to $P_n$, with the remaining $3/4$ mass uniform on $A^C \setminus P_n$. 
For samples of order $|P_n| \log (|P_n|)$, it will occur with high probability that every marker in $P_n$ is observed in the training set $S$. Thus all marker hypotheses are noninterpolating for $S$. Then succeeding with an interpolating proper learner amounts to learning the usual first Cantor class --- the learner must guess the hidden half-set $A$ from only a negligible fraction of its complement. This is impossible.

\vspace{-0.4 cm}
\paragraph{The exact empirical-error scale.}
The same construction can be tuned to demonstrate that any number $a_m = o(m)$ of empirical errors is necessary for certain properly learnable problems. In particular, fix one such $a_m = o(m)$. We calibrate the number and masses of the markers so that each marker hypothesis incurs $\Theta(a_m)$ errors on an $m$-sample from a realizable distribution. (The right values happen to be roughly $|\CX_m|=\Theta(m^3)$ and
$|P_m|=\min\{e^{\Theta(a_m)},\sqrt{m/a_m}\}$, with probability $\Theta(a_m/m)$ placed on each marker.) As argued previously, a proper learner can only succeed by selecting such a marker hypothesis, thus forcing $\Omega(a_m)$ training errors with constant probability.  

To demonstrate that a sublinear number of training errors can always be achieved, we first fix an arbitrary (successful) proper learner $A$. We then describe a procedure by which $A$ yields a proper learner with $o(m)$ empirical errors. In particular, we train $A$ on a prefix of the sample---of carefully chosen length---and then validate its output on the remaining suffix.  If the candidate passes validation, it already makes few errors on the complete sample; otherwise, we simply ``forfeit'' and output an arbitrary interpolating hypothesis.  Crucially, it is possible to select the prefix length and validation threshold along a sufficiently slow diagonal to achieve an $o(m)$ bound while preserving PAC learnability (\cref{thm:empirical-upper,thm:empirical-lower}).

\vspace{-0.4 cm}
\paragraph{Global regularization.}
Our previous ``poisoned'' Cantor construction demonstrates that a successful proper learner may need to sacrifice empirical fit in order to generalize.  Recall that weighted SRM permits precisely such a trade-off by minimizing the objective $h\mapsto \emp_S(h)+\lambda(S) \, \psi(h)$ (\cref{Equation:2}).  Ruling it out therefore requires a different obstruction.  Rather than merely eliminating every consistent safe hypothesis, we must ensure that every scalar regularizer ranks some statistically bad hypothesis at least as favorably as a good one.

Our construction begins with three disjoint vertex sets $V_0,V_1,V_2$, each of size $N$, connected cyclically by copies of the same bipartite graph, as below. 
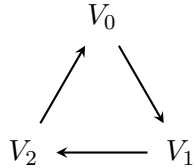
\begin{center}
\begin{tikzpicture}[
    >=stealth,
    edge/.style={
        ->,
        thick,
        shorten >=3pt,
        shorten <=3pt
    }
]
\node (V0) at (0,1.2) {$V_0$};
\node (V1) at ({0.6*sqrt(3)},-0.6) {$V_1$};
\node (V2) at ({-0.6*sqrt(3)},-0.6) {$V_2$};

\draw[edge] (V0) -- (V1);
\draw[edge] (V1) -- (V2);
\draw[edge] (V2) -- (V0);

\end{tikzpicture}
\end{center}
\vspace{-0.3 cm}
More precisely, for every $r\in\mathbb Z/3\mathbb Z$, we orient a $d$-regular bipartite graph from $V_r$ to $V_{r+1}$.  For a vertex $v \in V_r$, we write $\Gamma^+ (v) \subseteq V_{r+1}$ to denote its set of $d$ outgoing neighbors. We let $S(v):=\{v\}\cup\Gamma^+(v)$ denote the support of $v$.  Then only three graph-theoretic properties are needed: every vertex has large degree $d$; the supports of two distinct vertices intersect in at most one point; and any two subsets $U\subseteq V_r$ and $W\subseteq V_{r+1}$, each occupying a constant fraction of its respective part, have $\Omega(dN)$ edges between them. Informally, the graph has very small codegree but remains sufficiently dense and well mixed across large sets.  Projective-plane incidence graphs satisfy these requirements, with $N=q^2+q+1$ and $d=q+1=\Theta(\sqrt N)$, and we therefore use them in the formal construction.

For every vertex $v$, we introduce a Boolean-cube block $Z_v:=\{0,1\}^{\Gamma^+(v)}$, and the domain $\CX_N$ is the disjoint union of all such blocks. We refer to $v$ as the \emph{owner} of $Z_v$. Each vertex $u$ gives rise to a hypothesis $g_u$ in the following manner: on its own block $Z_u$, the hypothesis $g_u$ predicts $\starlabel$ everywhere. On a block $Z_v$ for which $u\in\Gamma^+(v)$, it acts like a first Cantor classifier, i.e., predicting $\starlabel$ on half of $Z_v$ and a private label $\lambda_u$ on the remaining half of $Z_v$. In particular, on input $z \in Z_v = \{0,1\}^{\Gamma^+(v)}$, $g_u$ predicts $\starlabel$ when $z_u = 0$ and $\lambda_u$ when $z_u = 1$. On all remaining blocks, for which $u \notin \Gamma^+(v)$, the hypothesis $g_u$ simply acts as the constant function emitting $\lambda_u$. Crucially, each block $Z_v$ witnesses three kinds of behaviors: the all-private-label behavior induced by any classifier $g_u$ with $u \notin S(v)$, the all-$\starlabel$ behavior induced by the owner $g_v$, and Cantor-like behaviors induced by the remaining hypotheses $\{g_u : u \in \Gamma^+(v)\}$.

We first argue that the class is properly learnable. If a private label $\lambda_u$ is observed in the sample, then the true classifier has simply been revealed as $g_u$. Otherwise, the sample must consist of only $\starlabel$-labeled points. If such points lie in two blocks $Z_u$ and $Z_v$, then the target function must belong to $S(u) \cap S(v)$. By the second of our 3 graph-theoretic properties, this set has cardinality 1, thus revealing the target function. Finally, if all the $\starlabel$-labeled datapoints lie in a single block $Z_v$, then we can safely predict its owner $g_v$, which acts as the all-$\starlabel$ classifier on $Z_v$.

To defeat a regularizer $\psi$, we first use the cyclic arrangement to find two adjacent parts $V_r$ and $V_{r + 1}$ whose median regularizer values are nonincreasing.  Let $U$ be the upper half of the first part and $W$ the lower half of the second, as measured by the regularizer.  Since these two large sets span $\Omega(dN)$ edges---by our third graph-theoretic requirement---some vertex $v\in U$ has $\Omega(d)$ outgoing neighbors $w\in W$. Each such neighbor satisfies $\psi(g_w)\leq\psi(g_v)$. We then take $g_v$ as the target and sample unlabeled data uniformly from $Z_v$. 
Crucially, every hypothesis other than $g_v$ incurs population error $\geq \frac 12$ on this distribution. However, each neighboring $g_w$ will happen to interpolate an $m$-sample with probability $2^{-m}$. By taking $d$ exponential in $m$, some such neighbor will interpolate with constant probability. On this event, either $g_v$ is not a minimizer of the weighted objective, in which case every minimizer is bad, or $g_v$ is a minimizer and the bad neighbor ties it. The argument therefore survives arbitrary ties and arbitrary sample-dependent choices of $\lambda(S)$, including $\lambda(S)=0$ (\cref{thm:weighted-srm-fails}).

\vspace{-0.4 cm}
\paragraph{Local regularization.}
The preceding lower bound ultimately exploits a single global ranking of the hypotheses.  Local regularization can evade this obstruction by changing its preferences with the test point: at each $x$, it may choose some $h_{S,x}\in\argmin_{h\in\vs(S)}\psi(h,x)$ and predict $h_{S,x}(x)$ (\cref{Equation:3}).  Indeed, on the incidence class above, a local regularizer could simply rank the owner $g_v$ first throughout its block $Z_v$.  To obtain a lower bound, we therefore need a collection of pointwise preferences that cannot all be satisfied simultaneously.  The cyclic structure of the preceding construction remains useful, but the projective-plane geometry is no longer needed: it suffices to simply consider a complete tripartite graph.

In particular, for each sample size $m$, set $N=2^{m-1}$ and consider the complete tripartite graph $G_m=K_{N,N,N}$.  An unlabeled datapoint is an orientation of every edge of $G_m$, and each edge $e=\{a,b\}$ defines a hypothesis $h_e$ that outputs the endpoint toward which $e$ is directed.  Thus every hypothesis uses only two labels, and improper learning is immediate: two observed labels identify the target edge, while a sample containing only one label permits the corresponding constant predictor.  The lower bound, meanwhile, is witnessed by directed cyclic triangles.  Given a test orientation, let $f$ be the earliest edge of such a triangle in the local regularizer's order, let $e$ be the preceding edge around the cycle, and let $a$ be the tail of $f$.  Then the target $h_e$ predicts $a$, whereas the locally preferred hypothesis $h_f$ does not.  The learner must therefore err whenever $f$ survives the sample and every incident edge ranked ahead of it is eliminated, an event of probability at least $2^{-m}(1-2^{-m})^{2N}$.  A random orientation contains $N^3/4$ directed cyclic triangles on average, compared with only $3N^2$ possible target edges. Taking $N=2^{m-1}$ and averaging over these witnesses yields a target-distribution pair on which the learner has constant error (\cref{thm:local-regularization-fails}).

\vspace{-0.4 cm}
\paragraph{Positive results.}
Our positive results are organized around the complementary ideas of \emph{localization} and \emph{integrability}.  Localization asks whether the ambiguity remaining after observing the sample is substantially simpler than the full hypothesis class.  The cleanest example is a finite fallback core $\CF$: every realizable sample either uniquely identifies the target or leaves at least one member of $\CF$ consistent.  An SRM rule may therefore rank $\CF$ first, after which an ordinary finite-class generalization bound completes the argument.  More generally, it suffices that every sample $S$ with $|\vs(S)|>1$ retain a consistent hypothesis inside a regularizer sublevel $\CH_{r_m}:=\{h:\psi(h)\le r_m\}$ whose complexity grows sufficiently slowly with $m$.
Ordered disagreement gives a target-dependent version of the same principle.  Since the target $h^\star$ is always consistent, every SRM output $g$ satisfies $\psi(g)\le\psi(h^\star)$.  We therefore need not control the full class $\CH$; it is enough to bound the VC dimension of the disagreement sets $\{x:g(x)\neq h^\star(x)\}$ generated by hypotheses lying below $h^\star$ in the regularizer order.  An $\epsilon$-net argument then ensures that every such disagreement set of substantial mass is hit by the sample, forcing the SRM output to have small population error (\cref{thm:fallback-core,thm:finite-ambiguous-range,thm:ambiguity-localized,thm:odd-suffices}).

Integrability instead begins with a learner and asks whether its sample-dependent choices can be represented by one regularizer.  Whenever a consistent learner chooses $h$ while another hypothesis $g$ remains feasible, it reveals a preference for $h$ over $g$.  In the hard-SRM setting, such choices admit a scalar representation exactly when the resulting preference relation is acyclic.  Weighted SRM also records the empirical-loss margin supporting each preference, producing a system of strict difference constraints on the values $\psi(h)$.  Summing these inequalities around a cycle eliminates the unknown potential; over finite systems, positivity of every cycle is both necessary and sufficient for reconstructing a regularizer realizing all of the learner's choices (\cref{thm:hard-integrability,thm:weighted-integrability}).

\subsection{Related work}\label{Subsection:related-work}

\paragraph{Multiclass learnability.}
\citet{DS14} introduced the DS dimension and conjectured it to characterize improper multiclass learnability over arbitrary label sets. This characterization was later confirmed by the breakthrough work of \citet{brukhim2022characterization}, using an intricate learner based upon list-learning reductions and orientations of one-inclusion graphs.
More recently, \citet{pabbaraju2026optimal} settled the sample complexity of realizable multiclass learning by proving a quantitative version of a Daniely--Shalev-Schwartz conjecture, building upon combinatorial work of \citet{hanneke2026optimal}.
The separation between proper and improper multiclass learning originates in the first Cantor construction of \citet{DS14}, which exhibits a learnable class admitting no proper learner.  \citet{asilisunderstanding} strengthened this obstruction by ruling out every aggregation of a bounded number of proper learners for general multiclass problems.
Furthermore, \citet{asilis25b-proper} demonstrated that proper learnability can be a poorly-behaved condition, by proving that it cannot be characterized by any combinatorial dimension, and that it can even be logically undecidable \citep{ben2019learnability}.

\vspace{-0.4 cm}
\paragraph{Regularization in multiclass learning.}
Structural risk minimization is among the classical mechanisms for trading empirical fit against hypothesis complexity \citep{shalev2014understanding}.  In multiclass learning, \citet{asilis2024regularization} showed that optimal learners can be expressed through substantially more flexible notions of regularization: the regularizer may be inferred from unlabeled data and may depend upon the test point, thereby combining different hypotheses across different regions of the domain.  The necessity of these relaxations motivated the question of \citet{asilis-open-problem}, who asked whether a fixed local regularizer can nevertheless learn every multiclass problem.  \citet{jafar2025local} gave a negative answer in the transductive setting while leaving the PAC model open.
We exhibit a tripartite construction that furthermore rules out local regularization in the PAC model, settling the problem completely. We note that the concurrent and independent work of \citet{hou2026local} obtains a closely related counterexample derived from tournaments on complete graphs. Both proofs exploit directed cyclic triangles as certificates of incompatible local preferences, but the analyses proceed differently: \citet{hou2026local} first expresses the learner's risk as a weighted count of inversions in the local edge order, whereas our tripartite construction isolates from each cyclic triangle a concrete bad edge whose survival forces an error, allowing for a somewhat more direct argument. On the proper side, prior work asked whether the promise of proper learnability is enough to recover classical SRM as a general learner \citep{asilis25b-proper}.  We again answer this question in the negative, even when the coefficient on the regularization term is permitted to depend arbitrarily upon the complete labeled sample.

\vspace{-0.4 cm}
\paragraph{Revealed preferences and integrability.}
Our representation results connect SRM with the classical theory of rational choice.  A deterministic learner choosing from sets of feasible hypotheses defines a choice rule, and asking for an SRM representation amounts to asking whether these choices can be rationalized by one scalar ordering.  Acyclicity, contraction consistency, and path independence are familiar themes in this literature \citep{plott1973path}.
Weighted SRM introduces sample-dependent empirical-loss offsets, leading instead to systems of difference constraints and corresponding cycle conditions.  This perspective bears a broad resemblance to revealed-preference and cyclic-monotonicity arguments in economics and mechanism design \citep{afriat1967utility,rochet1987rationalizability}, although the formal settings and questions are different.

\section{Preliminaries}\label{Section:Preliminaries}

Our setting is realizable multiclass PAC learning. $\CX$ and $\CY$ are, respectively, an arbitrary unlabeled data domain and label set. We permit $\CY$ to be infinite throughout the paper. A \defn{hypothesis class} is a set $\Hc\subseteq \Y^\X$, and a labeled \defn{example} is a pair $(x,y)\in \X\times\Y$. Given a sample $S=((x_1,y_1),\ldots,(x_m,y_m))$,
its \defn{empirical loss} is $ \emp_S(h):=\frac1m\sum_{i=1}^m \one\{h(x_i)\neq y_i\}$.
For a distribution $D$ on $\X$ and a target $h^\star\in\Hc$, the population loss of $h$ is $\pop_D(h,h^\star):=\P_{x\sim D}[h(x)\neq h^\star(x)]$.
A sample is \defn{realizable} if there exists an $h^\star\in\Hc$ such that $y_i=h^\star(x_i) \, \, \forall i$. A \defn{learner} is a function $(\CX \times \CY)^{< \omega} \to \CY^\CX$, and is \defn{proper} if its output always lies in $\Hc$, otherwise \defn{improper}. We say that $A$ PAC learns $\Hc$ if for every $\epsilon,\delta\in(0,1)$ there exists $n_A(\epsilon,\delta)$ such that for all $m\ge n_A(\epsilon,\delta)$, every $h^\star\in\Hc$, and every marginal $D$,
\[
\P_{S\sim D^m}\brac{\pop_D(A(S),h^\star)>\epsilon}\le \delta,
\]
where the sample is labeled by $h^\star$. 

The \defn{version space} of a realizable sample $S$ is
\[
\vs(S):=\set{h\in\Hc:\emp_S(h)=0}.
\]
A \defn{consistent} proper learner, or \emph{interpolating} proper learner, is one whose output always lies in $\vs(S)$. A realizable sample $S$ is said to be \defn{ambiguous} if $|\vs(S)| \geq 2$.

\subsection{Regularization models}\label{Subsection:regularization-models}

A \defn{regularizer} for $\CH$ is a function $\psi:\Hc\to\R_{\ge 0}$, thought of as encoding an inductive bias over hypotheses in $\CH$. Roughly, hypotheses $h \in \CH$ with smaller values under $\psi$ are thought of as being preferred by $\psi$, perhaps owing to their ``simplicity.'' 
A \defn{local regularizer} for $\CH$ is a function $\psi: \CH \times \CX \to \R_{\geq 0}$, likewise thought of as expressing preferences over hypotheses in $\CH$, but at a more granular, location-dependent level.

We now describe three manners in which (local) regularizers give rise to Structural Risk Minimization (SRM) learners. In particular, Hard SRM minimizes $\psi$ among interpolating hypotheses for $S$ (i.e., among the version space $\vs(S)$), Weighted SRM minimizes the sum of empirical risk and $\psi$ across all of $\CH$, and local regularization minimizes $\psi( \, \cdot \, , \xtest)$ locally at $\xtest$ among interpolating hypotheses. 

\begin{definition}[Hard SRM]
Let $\psi$ be a regularizer for $\CH$. A learner $A$ is said to be \defn{induced} by $\psi$ as a hard-SRM learner if the following condition holds over all realizable samples $S$,
\begin{equation}
A(S) \in \argmin_{h \in \CH} \big \{ \psi(h) : h \in \vs(S) \big \}.
\label{Equation:1}
\end{equation}
Should all hard-SRM learners induced by $\psi$ successfully learn $\CH$, then $\psi$ itself is said to \defn{learn} $\CH$.
\end{definition}

\begin{definition}[Weighted SRM]
Let $\psi$ be a regularizer for $\CH$ and let $\lambda : (\CX \times \CY)^{< \omega} \to \R_{\geq 0}$. A learner $A$ is \defn{induced} by the pair $(\psi, \lambda)$ as a weighted-SRM learner if for all samples $S$,
\begin{equation}
A(S) \in \argmin_{h\in\Hc}
\set{\emp_S(h)+\lambda(S) \, \psi(h)}. 
\label{Equation:2}
\end{equation}
Should all weighted SRM learners induced by the pair $(\psi, \lambda)$ be PAC learners for $\CH$, then they are said to \defn{learn} $\CH$.
\end{definition}

\begin{definition}[Local Regularization]
Let $\psi: \CH \times \CX \to \R_{\geq 0}$ be a local regularizer. A learner $A$ is \defn{induced} by $\psi$ if for all samples $S$ and points $x \in \CX$, 
\begin{equation}
A(S)(x)\in\big\{h(x):h\in\argmin_{g\in\vs(S)}\psi(g,x)\big\}.
\label{Equation:3}
\end{equation}
When all learners induced by $\psi$ are PAC learners for $\CH$, then $\psi$ is said to \defn{learn} $\CH$.
\end{definition}

A notable feature of our definitions is that we only crown a (local) regularizer $\psi$ as having successfully learned a hypothesis class $\CH$ if \emph{all} learners it induces are PAC learners for $\CH$. In short, we take a worst-case perspective with respect to tie-breaking of the $\mathrm{arg min}$ in \cref{Equation:1,Equation:2,Equation:3}: any fixed choice of tie-breaking must produce a successful learner. 
Semantically, this amounts to requiring that the success of a regularizer $\psi$ owe itself only to discernments made by $\psi$, not a particular sequence of good tie-breaking decisions. 
The regularization regions in \cref{fig:multiclass-containments} use this
same universal tie-breaking convention.

\subsection{Topological conditions}\label{Subsection:Topological-conditions}

Several of our forthcoming results take the form of impossibility theorems, prohibiting families of regularization-based learners from succeeding on certain learnable problems. We have found that such learnable counter-examples can often be found --- in a somewhat artificial way --- by exploiting hypothesis classes that are not topologically closed, i.e., that do not contain all their limit points. Briefly, the idea is that one can design $\CH$ lacking a limit point $f$ which could easily be learned by a regularizer favoring $f$ over all other hypotheses, if only $f$ were in $\CH$.
To exclude this artifact, we require each counterexample in
\cref{thm:no-erm,thm:empirical-lower,thm:weighted-srm-fails,%
thm:combined-global-negative,thm:figure-tagged,thm:figure-hypergraph,%
thm:figure-geometric} to satisfy the following two regularity conditions.

The first requirement is that $\CH$ should be closed in the product topology on $\CY^\CX$, when $\CY$ is endowed with the discrete topology. (I.e., the topology given by the 0-1 loss function on $\CY$.) In the event that $\CY$ is infinite, however, closedness of $\CH$ can be ``cheated'' by simply appending a single point to the domain $\CX$ and having each hypothesis emit a unique label on that point. As our second condition, we thus require that $|\CH(x)| = |\{h(x) : h \in \CH\}|$ be finite at each $x \in \CX$, which we refer to as having \emph{pointwise finite range}.\footnote{Note that this is a pointwise condition, not a uniform one --- $|\CH(x)|$ is permitted to attain arbitrarily large (but finite!) values across $x \in \CX$.}

\begin{definition}
Let $\Hc\subseteq \Y^\X$ be a hypothesis class.
\begin{enumerate}[label=(\roman*),topsep=0.3em,itemsep=0.25em]
    \item We say that $\Hc$ is \edefn{closed} if it is closed in the product topology on $\Y^\X$ when each copy of $\Y$ carries the discrete topology.
    \item We say that $\Hc$ has \edefn{pointwise finite range} if for every $x\in\X$ the following set is finite:
    \[
    \Hc(x):= \big \{ h(x):h\in\Hc \big \}.
    \]
    \item A finite set $I\subseteq \X$ is an \edefn{identifier} for $\Hc$ if the restriction map $h\mapsto \restr{h}{I}$ is an injection on $\Hc$.
\end{enumerate}
\end{definition}

It is useful to observe that the product-topological definition of closedness admits an elementary finite-projection reformulation. The following lemma arises directly from the definition of the product topology. 

\begin{lemma}\label{lem:closed-fpp}
A class $\Hc\subseteq\Y^\X$ is closed if and only if the following holds: whenever $f:\X\to\Y$ has the property that for every finite $I\subseteq\X$ there exists $h_I\in\Hc$ with 
$\restr{f}{I}=\restr{h_I}{I}$,
then $f\in\Hc$.
\end{lemma}

It is not difficult to see from Lemma~\ref{lem:closed-fpp} that any class $\CH$ with a finite identifier $I$ is automatically rendered closed, by considering sets of the form $I \cup \{x\}$ for each $x \in \CX$. Conversely, if $\CH$ is infinite and has pointwise finite range, then it cannot admit a finite identifier, as for any finite $I \subseteq \CX$ one has that $\abs{\restr{\Hc}{I}}\le \prod_{x\in I} \abs{\Hc(x)}<\infty$. Thus the pointwise finite range condition precludes this ``identifier trick'' so long as $\CH$ is infinite.

Finally, we note that the combination of both closedness and pointwise finite range --- which we employ in these negative results --- is sufficiently strong so as to ensure that all projections of $\CH$ to a subdomain $A \subseteq \CX$ are likewise closed (and furthermore compact).

\begin{lemma}\label{lem:restriction-compact}
If $\Hc$ is closed and has pointwise finite range, then the restriction $\restr{\Hc}{A}:=\set{\restr{h}{A}:h\in\Hc}$ is closed and compact in $\CY^A$ for each $A \subseteq \CX$.
\end{lemma}
\begin{proof}
Let $K:=\prod_{x\in\CX}\Hc(x)$. As $\Hc$ has pointwise finite range, each factor $\Hc(x)$ is finite and thus compact. By Tychonoff's theorem, $K$ is thus compact. Then $\Hc$ is a closed subset of $K$, hence compact. The restriction map $\pi_A: h\mapsto\restr{h}{A}$ is continuous, so $\pi_A (\Hc) = \restr{\Hc}{A}$ is compact in $\CY^A$. Finally, as $\CY^A$ is Hausdorff, its compact subsets are closed.
\end{proof}

\section{Proper Learning and Noninterpolation}\label{Section:Barriers-to-Proper-Learning}

We begin by establishing two limitations of proper learning. First, we demonstrate in \cref{Subsection:no-proper-completion} that general multiclass learning does not reduce to proper learning in the manner posited by \cref{problem3:envelope-proper}. In particular, \cref{thm:no-proper-completion} exhibits a learnable multiclass problem $\H$ that cannot be enveloped by any properly learnable class $\Hprop$, i.e., $\H \subseteq \Hprop$. Subsequently, in \cref{subsec:poisoned}, we demonstrate that proper learning requires one to tolerate empirical error: That is, there exist properly learnable problems for which all successful proper learners must err on realizable training samples. In \cref{subsec:tight-empirical}, we give a precise quantitative account of this phenomenon, by showing that there always exists a proper learner making a sublinear number of errors $o(n)$ on training sets of size $n$, and conversely that any sublinear error rate $a_n = o(n)$ can be made necessary on certain properly learnable problems.

\subsection{Learning does not reduce to proper learning}\label{Subsection:no-proper-completion}

The following theorem resolves \cref{problem3:envelope-proper}, first articulated (in the converse direction) as Conjecture~18 of \citet{asilis25b-proper}.

\begin{theorem}\label{thm:no-proper-completion}
There exists a learnable multiclass hypothesis class $\Hc$ that cannot be embedded into any properly learnable class. That is, there does not exist any properly learnable class $\Hprop$ such that $\Hc \subseteq \Hprop$. 
\end{theorem}

The remainder of this subsection is devoted to proving \cref{thm:no-proper-completion}.

\begin{proof}
We start with the construction of the data domain $\mathcal X$, the label set $\mathcal Y$, and the hypothesis class $\Hc$.

\vspace{-0.4 cm}
\paragraph{Construction.}

For each $d \geq 1$, let $T_d$ denote the infinite complete $2^d$-ary tree whose children at every depth are indexed by $\{0,1\}^d$. A node $v$ of $T_d$ at depth $k$ can be identified as a sequence $v=(b_1, \dotsc, b_k)$ where $b_i \in \{0, 1\}^d$:
\begin{itemize}
    \item The root at depth 0 is $\varnothing$.
    \item A node $v=(b_1, \dotsc, b_k)$ at depth $k$ has $2^d$ children of the form $(v, b_{k+1})$ for $b_{k+1} \in \{0, 1\}^d$.
\end{itemize}

Using the tree $T_d$, we now define the domain $\mathcal{X}$. Letting $|v|$ denote the depth of $v$, for each node $v \in T_d$ define $d$ domain points $x_{d, v, j}$ for $j \in [d]$. Let
\[ \mathcal X = \bigcup_{d \geq 1} \bigcup_{v \in T_d} \{x_{d, v, j}:j \in [d]\}. \]
We call $\{x_{d, v, j}:j \in [d]\}$ the \defn{block} of node $v$.

Each hypothesis in our constructed hypothesis class will correspond to choosing a finite path in $T_d$ and choosing one domain point in every non-terminal block along the path to label with 0 or 1. All other domain points are assigned a fixed label (other than 0 and 1) that can immediately identify the hypothesis. Formally, for any $\ell \geq 1$, any $b=(b_1, \dotsc, b_\ell) \in \prod_{i=1}^\ell \{0, 1\}^d$, and $u=(u_1, \dotsc, u_\ell) \in [d]^\ell$, define the hypothesis
\[ h_{d, b, u}(x) = \begin{cases}
    b_i(u_i) & \text{if $x=x_{d, (b_1, \dotsc, b_{i-1}), u_i}$ for some $i \in [\ell]$}, \\
    \alpha_{d, b, u} & \text{if $x \notin \{x_{d, (b_1, \dotsc, b_{i-1}), u_i}: i\in [\ell]\}$},
\end{cases} \]
where $b_i(u_i)$ denotes the $u_i$-th component of $b_i$. Here $u$ specifies which domain points are labeled with 0 or 1 on nodes along the path. Call the unique labels $\alpha_{d, b, u}$ that identify the hypothesis \defn{private} labels, and call $b_i(u_i)\in \{0,1\}$ the \defn{public} labels.

We finish the construction by defining the hypothesis class
\[ \Hc = \left\{h_{d, b, u}:d \geq 1, \ell \geq 1, b \in \prod_{i=1}^\ell \{0, 1\}^d, u \in [d]^\ell \right\}. \]
The label set $\mathcal{Y}$ consists of the public labels $\{0,1\}$ along with all private labels defined above:
\[ \mathcal Y = \{0, 1\} \cup \left\{\alpha_{d, b, u}:d \geq 1, \ell \geq 1, b \in \prod_{i=1}^\ell \{0, 1\}^d, u \in [d]^\ell \right\}. \]

\paragraph{An improper learner.}
As mentioned earlier, we construct an improper learner that either sees the private label and recovers the target exactly, or uses a deepest public example to reconstruct all preceding blocks.

Given a sample $S$, the learner $A$ proceeds as follows:
\begin{enumerate}
    \item If a private label $\alpha_{d, b, u}$ appears, then output the uniquely identified hypothesis $h_{d, b, u}$ (note that $\ell$ is also identified with the length of the vectors $b$ and $u$).
    \item Otherwise, all the labels are in $\{0, 1\}$. Choose any labeled example $(x_{d, v, j}, y)$ where $|v|$ is maximal in the data. Let $v=(b_1, \dotsc, b_{i^*-1})$ where $i^*=|v|+1$ and $u_{i^*}=j$. Output the classifier
    \[ h_S(x) = \begin{cases}
        y & \text{(if $x=x_{d, (b_1, \dotsc, b_{i^\ast-1}), u_{i^*}}$)}, \\
        b_i(j) & \text{(if $x=x_{d, (b_1, \dotsc, b_{i-1}), j}$ for some $i<i^*, j \in [d]$)}, \\
        0 & \text{otherwise}.
    \end{cases} \]
    In words, choose the domain point that has the largest depth and use that as the reference example to recover the labels in all preceding blocks.
\end{enumerate}
$A$ is an improper learner because $h_S(x)$ produced by the second case assigns public labels to the whole block of the node at depth $i-1$ where $i<i^*$ according to $b_i(\cdot)$, but any hypothesis in $\Hc$ assigns a public label to only one element of the block.

We next show that $A$ is a PAC learner.

\begin{lemma}
    For every target hypothesis $h \in \Hc$, distribution $D$ on $\mathcal X$, error $\varepsilon \in (0, 1)$, and sample size $m \geq 1$, we have
    \[ \P[L_D(A(S), h)>\varepsilon] \leq e^{-m \varepsilon}. \]
\end{lemma}
\begin{proof}
    The proof first handles the easy private-label case. For the remaining case we bound the probability that every example lies in a public prefix whose mass is less than $1 - \varepsilon$.

    Let $h=h_{d, b, u}$. If $S$ ever contains a private label, then by construction the learner recovers $h$ exactly.

    For the case where every example in $S$ has a public label, let $i^*=|v|+1$ be the index of the reference example. For $j=0, 1, \dotsc, \ell$ let $P_j$ be the set of domain points to which $h_{d, b, u}$ has given public labels in the first $j$ depths. Formally,
    \[ P_j:=\{x_{d, (b_1, \dotsc, b_{i-1}), u_i}:1 \leq i \leq j\}, \quad P_0:=\varnothing. \]
    Observe that the output $h_S(\cdot)$ is correct on $P_{i^*}$. Therefore, even if $h_S(\cdot)$ is wrong everywhere else, we still have $L_{D}(A(S), h) \leq 1 - D(P_{i^*})$. Here $D(Z):=\P_{x \sim D}[x \in Z]$.

    If $D(P_{i^*}) \geq 1 - \varepsilon$ then we are done. If not, let $i_\varepsilon$ be the largest $i \in \{0, 1, \dotsc, \ell\}$ such that $D(P_i)<1 - \varepsilon$. By assumption, $i_\varepsilon$ exists and $i^* \leq i_\varepsilon$. Since every example has a public label and has depth at most $i^*-1$, every example is also in $P_{i_\varepsilon}$. Therefore
    \[ \P[L_{D}(A(S), h)>\varepsilon] \leq D(P_{i_\varepsilon})^m<(1 - \varepsilon)^m \leq e^{-m \varepsilon} \]
    as desired.
\end{proof}

\vspace{-0.3 cm}
\paragraph{No properly learnable envelope exists.}
We prove the stronger statement that every (possibly randomized) PAC learner $A$ for $\mathcal H$ satisfies $\DS\bigl(\operatorname{im}(A)\bigr)=\infty$, as promised in \cref{Subsection:techniques}. This immediately proves \cref{thm:no-proper-completion}. Indeed, if $\Hc \subseteq \Hprop$ and $\Hprop$ has a proper PAC learner $A$, then $\operatorname{im}(A) \subseteq \Hprop$, so $\DS(\Hprop)=\infty$, which contradicts the fact that $\Hprop$ is learnable.

Fix an arbitrary PAC learner $A$ for $\Hc$, and let $n_A(\varepsilon, \delta)$ be its sample complexity bound. Fix any $d \geq 1$ and consider the corresponding tree $T_d$. We first prove the following dichotomy, which says that either $\operatorname{im}(A)$ is very rich or it systematically misses certain patterns.

\begin{lemma}
    For a node $v \in T_d$, identify a vector $c \in \{0, 1\}^d$ as a function on the block
    \[ E_{d, v}:=\{x_{d, v, j}:j \in [d]\} \]
    that maps $x_{d, v, j}$ to $c(j)$, where $c(j)$ denotes the $j$th component of $c$. Then at least one of the following holds:
    \begin{description}
        \item[Condition 1] There exists a node $v \in T_d$ such that
        \[ \forall c \in \{0, 1\}^d\ \exists f_c \in \operatorname{im}(A) \quad \restr{f_c}{E_{d, v}}=c. \]
        \item[Condition 2] There exists an infinite sequence $b_1, b_2, \dotsc$ with $b_i \in \{0, 1\}^d$ such that for every $i \geq 1$
        \[ b_i \notin \restr{\operatorname{im}(A)}{E_{d, (b_1, \dotsc, b_{i-1})}}. \]
    \end{description}
\end{lemma}
\begin{proof}
    We prove the dichotomy by assuming Condition 1 fails and recursively following a missing $d$-bit pattern down the tree.

    If the first condition does not hold, then
    \[ \forall v \in T_d\ \exists c \in \{0, 1\}^d\ \forall f_c \in \operatorname{im}(A) \quad \restr{f_c}{E_{d, v}} \neq c. \]
    In words, for every node $v$ there exists a function $c$ such that no hypothesis in $\operatorname{im}(A)$ realizes $c$ on $E_{d, v}$. Therefore, we can start with the root (labeled as $\varnothing$) and choose a missing vector $b_1 \notin \restr{\operatorname{im}(A)}{E_{d, \varnothing}}$. After choosing $b_1, \dotsc, b_{i-1}$, we are at node $v=(b_1, \dotsc, b_{i-1})$ and by the condition above we can choose a missing vector $b_i \notin \restr{\operatorname{im}(A)}{E_{d, v}}$ and move to the next child $(b_1, \dotsc, b_i)$. This finishes the proof.
\end{proof}

In the remaining text, we show that (1) Condition 1 supplies the promised pseudocube in \cref{Subsection:techniques} and hence $\DS\bigl(\operatorname{im}(A)\bigr)=\infty$, and (2) Condition 2 contradicts the PAC guarantee of $A$ and hence is impossible.

\paragraph{Condition 1.}
First recall that a non-empty finite class $\Hc \subseteq \mathcal Y^E$ where $|E|=d$ is a \defn{pseudocube} if, for every $f \in \Hc$ and every $x \in E$, there exists $g \in \Hc$ such that $g(x) \neq f(x)$ and $g(x')=f(x')$ for every $x' \in E \setminus \{x\}$. In particular, a full binary cube is a pseudocube. Also recall that a set $E \subseteq \mathcal X$ is \defn{DS-shattered} by a class $\Hc$ if $\restr{\Hc}{E}$ contains a $|E|$-dimensional pseudocube, and the \defn{DS dimension} of $\Hc$ is the supremum of the sizes of its finite DS-shattered sets.

Now, Condition 1 says that the restrictions $\{\restr{f_c}{E_{d, v}}:c \in \{0, 1\}^d\}$ where $f_c \in \operatorname{im}(A)$ contain the full binary cube on $E_{d, v}$, and hence a $d$-dimensional pseudocube. Therefore, $\DS\bigl(\operatorname{im}(A)\bigr) \geq d$. Since $d$ is arbitrary, we conclude that $\DS\bigl(\operatorname{im}(A)\bigr)=\infty$.

\paragraph{Condition 2.}
Let $m=n_A \left(\frac{1}{4d}, \frac{1}{8d} \right)$ and $\ell=2m$.
Fix the path $b_1, \dotsc, b_\ell$ that is shown to exist by assumption. We will use the probabilistic method to show that, when $A$ is given $m$ examples, there exists a hard instance $h_u \in \Hc$ such that $A$ cannot satisfy its PAC guarantee. The crux is that we will draw $m$ public samples from a distribution supported uniformly on one public point at each of the $2m$ depths in the tree, but $m$ examples cannot cover all $2m$ depths, so at every unseen depth, membership in $\operatorname{im}(A)$ forces at least one error coordinate in the block, and the lower bound follows because a random example hits such a coordinate with probability at least $\frac{1}{d}$.

Choose independently $u_1, \dotsc, u_\ell \sim U([d])$ and consider the random target $h_u:=h_{d, (b_1, \dotsc, b_\ell), u}$. Let $D_u$ be the uniform distribution over $h_u$'s public points
\[ x_i(u):=x_{d, (b_1, \dotsc, b_{i-1}), u_i} \quad i \in [\ell]. \]
Recall that the label for $x_i(u)$ is $b_i(u_i)$, so sampling $m$ i.i.d. examples (which form the sample $S$) from $D_u$ can also be done by first sampling $m$ i.i.d. depth indices $i_1, \dotsc, i_m \sim U([\ell])$ and then returning $(x_{i_t}(u), b_{i_t}(u_{i_t}))$. We call a depth $i-1$ \defn{unseen} if $i \notin \{i_1, \dotsc, i_m\}$.

Conditional on the $m$ examples for learning which form the set $S$ and $A$'s internal randomness, the output function $g := A(S) \in \operatorname{im}(A)$ is determined. No hypothesis in $\operatorname{im}(A)$ agrees with $b_i$ on block $E_{d, (b_1, \dotsc, b_{i-1})}$. So on $E_{d, (b_1, \dotsc, b_{i-1})}$ we have $g \neq b_i$, and in particular there exists at least one $j \in [d]$ such that
\[ g(x_{d, (b_1, \dotsc, b_{i-1}), j}) \neq b_i(j). \]
Observe crucially that, if depth $i-1$ is unseen, then $u_i$ is uniformly distributed over $[d]$ conditional on the sample and $A$'s internal randomness. Therefore, for an unseen depth $i-1$, $u_i=j$ with probability at least $\frac{1}{d}$, and $g$ makes an error. Together with the fact that the depths are uniformly randomly selected according to $D_u$, we conclude that an unseen depth $i-1$ contributes to conditional expected loss at least $\frac{1}{d \ell}$.

Finally, because for each fixed $i \in [\ell]$, depth $i-1$ is unseen with probability at least $\left(1 - \frac 1\ell \right)^m$, we have
\[ \operatorname*{\mathbb E}_{u, S, A}[L_{D_u}(A(S), h_u)] \geq \ell \cdot \frac{1}{d \ell} \left(1 - \frac 1\ell \right)^m = \frac{1}{d} \left(1 - \frac{1}{2m} \right)^m \geq \frac{1}{2d} \]
by summing over the contribution from all $\ell$ depths. Hence there exists a $u$ such that $\operatorname*{\mathbb E}_{S, A}[L_{D_u}(A(S), h_u)] \geq \frac{1}{2d}$. However, if we look at the PAC learning guarantee of $A$, because $m = n_A \left(\frac{1}{4d}, \frac{1}{8d} \right)$ and $h_u \in \Hc$, the expected error is at most
\[ \frac{1}{8d}+ \left(1-\frac{1}{8d} \right) \cdot \frac{1}{4d} \leq \frac{3}{8d}< \frac{1}{2d}, \]
which is a contradiction. Thus Condition 2 is impossible.
\end{proof}

\begin{corollary}
    Our proof of \cref{thm:no-proper-completion} provides a simpler construction of a partial concept class with VC dimension 1 that cannot be disambiguated to a total concept class of finite VC-dimension: replace every private label $\alpha_{d,b,u}$ with $\star$. The original result \cite[Theorem 6]{alon2022theory} uses graphs whose chromatic number is large compared to their biclique partition number.
\end{corollary}

\subsection{Proper learning requires empirical error}
\label{subsec:poisoned}

We now restrict our attention to properly learnable classes and ask whether
interpolation is always possible.  That is, must every properly learnable class
admit a successful proper learner whose output incurs no error on the training
sample?  Our answer is negative.

The construction begins from the first Cantor class of
\citet{DS14}.  Given a finite domain \(X\), this class contains one hypothesis
\(h_A\) for each half-set \(A\subseteq X\) with $|A| = |X| / 2$: the hypothesis predicts a private label
identifying \(A\) on the points of \(A\), and predicts a common label
\(\starlabel\) everywhere else.  Thus the appearance of a private label reveals
the target immediately, whereas an all-\(\starlabel\) sample reveals only that
the sample has avoided the target's private-label region.  We use a version of this class that is augmented with a family of hypotheses that output the $\star$ label on nearly the entire domain, but are deliberately inconsistent at rare marker points.

\begin{definition}[Poisoned first-Cantor class]
\label{def:poisoned-first-cantor}
For every \(n\ge1\), let
\[
    X_n=C_n\sqcup P_n,
    \qquad
    P_n=\{p_{n,1},\ldots,p_{n,n}\},
    \qquad
    |C_n|=2n^3+n.
\]
Thus
\[
    |X_n|=2n^3+2n,
    \qquad
    a_n:=n^3+n=\frac{|X_n|}{2}.
\]
Set the final domain to the disjoint union of these blocks, i.e., $\CX^{\mathrm P}:=\bigsqcup_{n\ge1}X_n$.

The label set contains two common labels \(\starlabel\) and \(\dollarlabel\).
For every \(A\subseteq X_n\) with \(|A|=a_n\), introduce a private label
\(\lambda_{n,A}\), and for every \(i\in[n]\), introduce a poison label
\(\rho_{n,i}\).  All private and poison labels are distinct.

For every
\(A\subseteq X_n\) with \(|A|=a_n\), define the Cantor hypothesis
\[
    c_{n,A}(x)
    =
    \begin{cases}
        \starlabel, & x\in A,\\
        \lambda_{n,A}, & x\in X_n\setminus A,\\
        \dollarlabel, & x\notin X_n.
    \end{cases}
\]
For every \(i\in[n]\), define the poisoned fallback
\[
    s_{n,i}(x)
    =
    \begin{cases}
        \rho_{n,i}, & x=p_{n,i},\\
        \starlabel, & x\in X_n\setminus\{p_{n,i}\},\\
        \dollarlabel, & x\notin X_n.
    \end{cases}
\]
Finally, let \(h_{\dollarlabel}\) denote the all-\(\dollarlabel\) hypothesis.  
The \edefn{poisoned first-Cantor class} is
\[
    \Hpoison
    :=
    \{h_{\dollarlabel}\}
    \cup
    \left\{
        c_{n,A}:
        n\ge1,\ A\subseteq X_n,\ |A|=a_n
    \right\}
    \cup
    \left\{
        s_{n,i}:
        n\ge1,\ i\in[n]
    \right\}.
\]
\end{definition}

Note that the poisoned hypotheses \(s_{n,i}\) are almost constantly
\(\starlabel\), and are therefore natural fallbacks when an all-\(\starlabel\)
sample leaves the target unresolved. Their isolated errors at the marker
points \(p_{n,i}\), however, can make them unavailable to an interpolating learner. This forms the central idea behind the following theorem.

\begin{theorem}
\label{thm:no-erm}
The class \(\Hpoison\) is countable, closed, and has pointwise finite range.
It is furthermore properly PAC learnable, but cannot be learned by any interpolating
proper learner.
\end{theorem}
\begin{proof}
We establish the four claims in turn.
\vspace{-0.4 cm}
\paragraph{Proper learnability.}
For a sample size \(m\), put
\[
    K_m:=\lfloor m^{1/5}\rfloor,
\]
and let \(\Hpoison^{\le K_m}\) consist of \(h_{\dollarlabel}\) together with
all hypotheses supported on blocks \(X_n\) for which \(n\le K_m\).

The learner proceeds as follows.  If every observed label is
\(\dollarlabel\), it outputs \(h_{\dollarlabel}\).  Otherwise, all
non-\(\dollarlabel\) labels occur in a unique block \(X_n\).  If a private
Cantor label \(\lambda_{n,A}\) appears, the learner outputs \(c_{n,A}\); if a
poison label \(\rho_{n,i}\) appears, it outputs \(s_{n,i}\).  The only
remaining case is that the sample contains non-\(\dollarlabel\) points from
\(X_n\), all labeled \(\starlabel\).  If \(n\le K_m\), the learner outputs an
arbitrary member of \(\Hpoison^{\le K_m}\) consistent with the entire sample.
If \(n>K_m\), it chooses a least-observed marker
\[
    \hat i
    \in
    \argmin_{i\in[n]}
    \#\{(x,y)\in S:x=p_{n,i}\}
\]
and outputs \(s_{n,\hat i}\).

The small-block rule is uniformly sound.  Indeed,
\[
    \log|\Hpoison^{\le K_m}|
    =
    O\!\left(
        \sum_{n\le K_m}|X_n|
    \right)
    =
    O(K_m^4)
    =
    O(m^{4/5})
    =
    o(m).
\]
Hence the standard finite-class realizable bound implies that every consistent
choice from \(\Hpoison^{\le K_m}\) has vanishing population error, uniformly
over the target and marginal distribution.

Consider now a large-block Cantor target \(c_{n,A}\), where \(n>K_m\), and
write
\[
    Q:=X_n\setminus A
\]
for its private-label region.  If the sample meets \(Q\), the label
\(\lambda_{n,A}\) reveals the target exactly.  If the sample misses \(Q\),
then for every \(\epsilon>0\),
\[
    \P\!\left[
        D(Q)>\frac{\epsilon}{4}
        \ \text{and}\
        S\cap Q=\varnothing
    \right]
    \le
    e^{-m\epsilon/4}.
\]
Likewise, if the sample misses the active block \(X_n\) entirely, the learner
outputs \(h_{\dollarlabel}\); this can incur error greater than \(\epsilon\)
only when \(D(X_n)>\epsilon\), in which case the probability of missing
\(X_n\) is at most \(e^{-m\epsilon}\).

In the remaining case the learner outputs \(s_{n,\hat i}\), and its
disagreement with the target is contained in $\{x:s_{n,\hat i}(x)\neq c_{n,A}(x)\} \subseteq Q\cup\{p_{n,\hat i}\}$.
The least-observed-marker lemma, proved in
\Cref{lem:least-observed-marker}, asserts that among a growing family of
distinguished atoms, a marker of minimum empirical frequency has small true
mass with high probability.  Since \(n>K_m\to\infty\), it follows that $ D(p_{n,\hat i})\le \frac{\epsilon}{4}$
with probability tending to one, uniformly over the marker marginal.
Together with the preceding missed-mass bounds, this proves that the output
has population error at most \(\epsilon\) with high probability.

Fallback targets are handled similarly.  Fix \(s_{n,j}\).  If the marker
\(p_{n,j}\) is sampled, the poison label \(\rho_{n,j}\) reveals the target
exactly.  If it is missed, then
\[
    \P\!\left[
        D(p_{n,j})>\frac{\epsilon}{4}
        \ \text{and}\
        p_{n,j}\notin S
    \right]
    \le
    e^{-m\epsilon/4}.
\]
If the active block is observed but the poison label is not, the learner
outputs \(s_{n,\hat i}\), whose disagreement with \(s_{n,j}\) is contained in $\{x:s_{n,\hat i}(x)\neq s_{n,j}(x)\} \subseteq \{p_{n,\hat i},p_{n,j}\}$.
The first mass is controlled by the least-observed-marker lemma and the second
by the preceding missed-mass estimate.  The all-\(\dollarlabel\) target is
learned exactly.  This proves distribution-free proper learnability.

\vspace{-0.4 cm}
\paragraph{Countability and pointwise finite range.}
Each block \(X_n\) is finite and supports only finitely many hypotheses, while
there are countably many blocks.  Hence \(\Hpoison\) is countable. To see that it has pointwise finite range, fix an $x \in X_n$. Each hypothesis supported on a different block takes the
value \(\dollarlabel\) at \(x\), and only finitely many hypotheses are
supported on \(X_n\).  Thus $|\Hpoison(x)|<\infty$.

\vspace{-0.4 cm}
\paragraph{Closedness.}
We use the finite-projection characterization of closedness from
\Cref{lem:closed-fpp}.  Let \(f:\CX^{\mathrm P}\to\CY\) be finitely
consistent with \(\Hpoison\).  If \(f\equiv\dollarlabel\), then
\(f=h_{\dollarlabel}\). Otherwise, choose \(x\in X_n\) with \(f(x)\neq\dollarlabel\).  We claim that
\(f\) is identically \(\dollarlabel\) outside \(X_n\).  Indeed, if
\(y\in X_{n'}\), \(n'\neq n\), also satisfied
\(f(y)\neq\dollarlabel\), then no hypothesis in \(\Hpoison\) could realize the
two-point restriction of \(f\) to \(\{x,y\}\), since every nontrivial
hypothesis is supported on a single block.  This would contradict finite
consistency.
The block \(X_n\) is finite, so finite consistency applied to all of \(X_n\)
yields some \(h\in\Hpoison\) satisfying $\restr{h}{X_n} = \restr{f}{X_n}$.
Since \(f(x)\neq\dollarlabel\), this hypothesis is supported on \(X_n\).
Both \(f\) and \(h\) equal \(\dollarlabel\) outside \(X_n\), and hence
\(f=h\in\Hpoison\).

\vspace{-0.4 cm}
\paragraph{Failure of every interpolating proper learner.}
Let \(B_m\) be an arbitrary deterministic interpolating proper learner.  Fix a
large sample size \(m\), and choose
\[
    n
    :=
    \left\lfloor
        \frac{m}{20\log m}
    \right\rfloor,
    \qquad
    k:=n^3.
\]
Choose \(T\subseteq C_n\) uniformly among all \(k\)-subsets, and put $A:=P_n\cup T$. Since $|A|=n+k=a_n$,
the Cantor hypothesis \(c_{n,A}\) belongs to \(\Hpoison\).

Now let \(D_A\) place mass \(1/4\) uniformly on \(P_n\) and mass \(3/4\) uniformly on \(T\).  Take \(c_{n,A}\) as the true labeling function.  Each
point in the support of \(D_A\) lies in \(A\), and is therefore labeled
\(\starlabel\). Furthermore, each marker has mass \(1/(4n)\), so a union bound gives
\[
    \P[\text{some marker in }P_n\text{ is missed}]
    \le
    n\left(1-\frac1{4n}\right)^m
    \le
    n\exp\!\left(-\frac{m}{4n}\right)
    =
    o(1).
\]
On the complementary event, every poisoned fallback \(s_{n,i}\) is
inconsistent, since the sample contains \(p_{n,i}\) with label
\(\starlabel\), whereas \(s_{n,i}(p_{n,i})=\rho_{n,i}\).  The hypothesis
\(h_{\dollarlabel}\), all hypotheses supported on other blocks, and all
private-label hypotheses whose \(\starlabel\)-regions do not contain the
sample are likewise inconsistent.  Thus the learner must output a Cantor
hypothesis \(c_{n,B}\).

Let \(R\subseteq T\) be the set of distinct sampled core points.  Consistency
forces $P_n\cup R\subseteq B$.
Since \(|B|=n+k\) and \(P_n\subseteq B\), the set $B_C:=B\cap C_n$
has size \(k\) and contains \(R\).
Now condition on the complete observed sample and on the learner's internal
choice of \(B_C\).  Given this information, the posterior distribution of
\(T\) is uniform over all \(k\)-subsets of \(C_n\) containing \(R\). Moreover,
\[
    |R|\le m=o(k),
    \qquad
    |C_n|=2k+n=(2+o(1))k.
\]
Thus \(T\setminus R\) is a uniformly random \((k-|R|)\)-subset of
\(C_n\setminus R\), while \(B_C\setminus R\) is a fixed set of the same
cardinality.  A standard hypergeometric concentration bound therefore yields
\[
    \P\!\left[
        |T\setminus B_C|\ge \frac{k}{5}
        \,\middle|\,
        S,B_C
    \right]
    =
    1-o(1).
\]
On this event,
\[
    \pop_{D_A}(c_{n,B},c_{n,A})
    \ge
    \frac34
    \frac{|T\setminus B_C|}{|T|}
    \ge
    \frac{3}{20}.
\]

Averaging over the random choice of \(T\), we conclude that there is a fixed
\(k\)-subset \(T\subseteq C_n\) for which \(B_m\) has population error at
least \(3/20\) with probability bounded away from zero.  Since the
construction applies for arbitrarily large \(m\), the learner \(B\) is not
PAC.

The same argument rules out randomized interpolating proper learners: one
simply includes the learner's internal randomness in the averaging and then
fixes a target \(T\) witnessing the resulting constant failure probability.
\end{proof}

We now observe an immediate consequence of \cref{thm:no-erm}: hard SRM learners, which select interpolating hypotheses using an a-priori inductive bias, are in general insufficiently expressive for properly learnable problems.

\begin{corollary}
\label{cor:hard-srm-not-characterize}
There exists a properly learnable hypothesis class, namely $\Hpoison$, that cannot be learned by hard SRM.
\end{corollary}
\begin{proof}
Every hard-SRM learner is an interpolating proper learner, whereas
\(\Hpoison\) is properly learnable but admits no successful interpolating
proper learner by \Cref{thm:no-erm}.
\end{proof}

Let us briefly make two observations concerning \cref{thm:no-erm}. First, note that the presence of only a single poisoned fallback would not suffice. In order to force every interpolating learner to reject the fallback with high probability, its poisoned point would need appreciable population mass. This would immediately imply, however, that the fallback is no longer statistically safe at test time. In contrast, a collection of $\omega(1)$ many poisoned fallbacks ensures that the entire collection of markers can carry constant mass, while nevertheless containing some markers that are successful at test time. (In particular, the ones whose poisoned points are least likely.) 

Second, we note that \cref{thm:no-erm} separates proper learning from interpolation, but does not by itself rule out weighted SRM. In particular, weighted SRM can express improper learners by selecting a preferred hypothesis with positive training error. We will, however, rule out weighted SRM in \cref{Section:Limits-of-Regularization} using a separate construction.

\subsection{The exact scale of empirical noninterpolation}\label{subsec:tight-empirical}

We next establish a sharp characterization concerning how many training errors can be required in proper learning: a sublinear rate $o(m)$ can be attained for any properly learnable problem, and conversely every sublinear rate $a_m = o(m)$ can be made necessary for some properly learnable problem.

\begin{theorem}\label{thm:empirical-upper}
Let $A$ be a proper PAC learner for a class $\mathcal{H}$ with a sample complexity $n_A(\epsilon,\delta)$. Then there exists a proper PAC learner $\widetilde A$ for $\mathcal{H}$ and a sequence $\gamma_m\to0$ such that for every realizable sample $S$ of size $m$,
\[
\emp_S(\widetilde A_m(S))\le\gamma_m.
\]
Moreover, one may take $n_{\widetilde A}(\epsilon,\delta) \le R(\epsilon,\delta)\, n_A\!\left(R(\epsilon,\delta)^{-2},R(\epsilon,\delta)^{-1}\right)$,
where $R(\epsilon,\delta) = \left\lceil\max\set{2,\epsilon^{-1/2},2/\delta}\right\rceil$.
\end{theorem}

\begin{proof}
For each $m$, let $r_m$ be the largest integer $r\ge2$ satisfying
\[
r\,n_A(r^{-2},r^{-1})\le m,
\]
if such an integer exists. Put $k_m=n_A(r_m^{-2},r_m^{-1})$. On an $m$-sample $S$, train $A$ on the first $k_m$ examples, obtaining $h$. If the empirical error of $h$ on the remaining $m-k_m$ examples is at most $r_m^{-1}$, output $h$; otherwise output an arbitrary hypothesis in $\Hc$ consistent with the full sample. On the finitely many $m$ for which $r_m$ is undefined, output a full-sample consistent hypothesis.

The learner is proper. On every realizable sample, the fallback has empirical error zero. If $h$ is accepted, it makes at most $k_m$ mistakes on the training prefix and at most $(m-k_m)/r_m$ mistakes on the validation suffix. Since $k_m/m\le1/r_m$,
\[
\emp_S(\widetilde A_m(S))\le\frac2{r_m}.
\]
For every fixed $r$, the defining inequality eventually holds, so $r_m\to\infty$ and we may take $\gamma_m=2/r_m$.

For the PAC guarantee, fix a target and marginal. With probability at least $1-r_m^{-1}$, the prefix learner has true error at most $r_m^{-2}$. Conditional on such an output, the validation suffix is independent and its expected empirical error is at most $r_m^{-2}$; Markov's inequality therefore bounds the probability of rejection by $r_m^{-1}$. Outside these two events, $\widetilde A_m$ accepts a hypothesis of true error at most $r_m^{-2}$. Thus
\[
\P\brac{\pop_D(\widetilde A_m(S),h^\star)>r_m^{-2}}
\le\frac2{r_m}.
\]
If $m\ge Rn_A(R^{-2},R^{-1})$, then $r_m\ge R$, yielding the displayed sample-complexity bound.
\end{proof}

\begin{theorem}\label{thm:empirical-lower}
Let $(a_m)_{m\ge1}$ be any nonnegative sequence with $a_m=o(m)$. There exists a countable, closed, pointwise finite-range, properly learnable class $\Hc_a$ and a universal constant $c>0$ such that every proper PAC learner $A$ for $\Hc_a$ satisfies the following: for all sufficiently large $m$, there is a realizable pair $(D_m,h_m^\star)$ for which
\[
\P_{S\sim D_m^m}\brac{m\emp_S(A_m(S))\ge c a_m}
\ge c.
\]
\end{theorem}

We defer the proof of \cref{thm:empirical-lower} to \Cref{app:tunable-poison}. Roughly speaking, the construction simply tunes the poisoned first Cantor class $\Hpoison$ (i.e., either increasing or decreasing the number of marker points) in order to ensure that every marker point receives the prescribed number of training points, while still having vanishing population mass.

\begin{remark}\label{rem:rare-sample-pathology}
One might ask whether the upper bound of \cref{thm:empirical-upper} is required to hold for optimal learners. The answer is negative for a simple, somewhat shallow reason: a learner can use anticoncentration in order to deliberately sabotage rare samples. In particular, consider the two constant binary hypotheses on $\N$. Modify the obvious learner so that, on the ordered unlabeled sample $(1,2,\ldots,m)$, it deliberately outputs the wrong constant. The probability of that exact sequence is
\[
\prod_{i=1}^m D(i)
\le
m^{-m}
\le
1/m!,
\]
so the learner remains PAC while occasionally misclassifying every training point. A three-hypothesis variant preserves optimal $\Theta(\epsilon^{-1}\log(1/\delta))$ sample complexity; see \Cref{app:pathological-learners}.
\end{remark}

\section{Limits of Global and Local Regularization}\label{Section:Limits-of-Regularization}

The preceding section shows that proper multiclass learning may require a sample-dependent departure from interpolation. We now ask whether this flexibility can nevertheless be encoded by regularization. We consider two templates: weighted SRM, which provides global scalar preferences over hypotheses, and local regularization, which provides pointwise preferences and can express improper learning.

\subsection{A closed incidence class beyond weighted SRM}
\label{subsec:incidence}

Our global lower bound is built from finite projective planes. Recall that a projective plane of order $q$ has $N=q^2+q+1$ points and the same number of lines. Every point lies on $d=q+1$ lines, every line contains $d$ points, and every two distinct points lie on a unique common line. If $M$ denotes its $N\times N$ incidence matrix, then $MM^\top=qI+J$. Consequently, the largest singular value of $M$ is $d$, while every nontrivial singular value is $\sqrt q$.

\begin{definition}[Projective-plane incidence class]
\label{def:incidence-class}
For every integer $m\ge2$, let
\[
    q_m:=2^{m+10},
    \qquad
    N_m:=q_m^2+q_m+1,
    \qquad
    d_m:=q_m+1.
\]
Since $q_m$ is a prime power, a projective plane of order $q_m$ exists.

Take three disjoint vertex sets $V_0^{(m)},V_1^{(m)},V_2^{(m)}$, each of cardinality $N_m$. For every $r\in\mathbb Z/3\mathbb Z$, place an oriented copy of the point-line incidence graph from $V_r^{(m)}$ to $V_{r+1}^{(m)}$. For $v\in V_r^{(m)}$, write $\Gamma^+(v)\subseteq V_{r+1}^{(m)}$ for its outgoing neighborhood and put $\mathsf S(v):=\{v\}\cup\Gamma^+(v)$. The projective-plane axioms imply
\begin{equation}
\label{eq:incidence-small-intersection}
    |\mathsf S(v)\cap\mathsf S(w)|
    \le 1
    \qquad
    \text{for every pair of distinct vertices }v,w.
\end{equation}
Indeed, two vertices in the same part have exactly one common outgoing neighbor, whereas the cyclic supports of vertices in different parts permit at most one intersection.

For every vertex $v$, create the Boolean-cube block $Z_v:=\{0,1\}^{\Gamma^+(v)}$. The $m$th incidence component and the full domain are
\[
    \CX_m^{\mathrm I}
    :=
    \bigsqcup_{v\in V_0^{(m)}\sqcup V_1^{(m)}\sqcup V_2^{(m)}} Z_v,
    \qquad
    \CX^{\mathrm I}
    :=
    \bigsqcup_{m\ge2}\CX_m^{\mathrm I}.
\]
Introduce a private label $\tau_u$ for every vertex $u$, with all such labels distinct across components. For a vertex $u$ in the $m$th component, define $g_u$ on that component by
\begin{equation}
\label{eq:incidence-hypothesis}
    g_u(v,z)
    =
    \begin{cases}
        \starlabel,
            & u=v,\\
        \starlabel,
            & u\in\Gamma^+(v)\text{ and }z_u=0,\\
        \tau_u,
            & \text{otherwise},
    \end{cases}
    \qquad
    (v,z)\in\CX_m^{\mathrm I},
\end{equation}
and extend $g_u$ by $\dollarlabel$ outside its component. Let $h_{\dollarlabel}$ be the all-$\dollarlabel$ hypothesis, and set
\[
    \Hinc
    :=
    \{h_{\dollarlabel}\}
    \cup
    \{g_u:u\text{ is a vertex in some incidence component}\}.
\]
We call $\Hinc$ the \edefn{projective-plane incidence class}.
\end{definition}

The intuition behind the incidence class of \cref{def:incidence-class} is easiest to see on a single block $Z_v$. Its owner $g_v$ predicts $\starlabel$ throughout the block, whereas the outgoing neighbors $u \in \Gamma^+(v)$ behave in a Cantor-like manner, i.e., agreeing with the owner on the half of $Z_v$ where $z_u=0$ and revealing the private label $\tau_u$ on the half where $z_u=1$. Every hypothesis whose index lies outside $\mathsf S(v)$ simply predicts its private label throughout $Z_v$.

\begin{theorem}
\label{thm:weighted-srm-fails}
The incidence class $\Hinc$ is countable, closed, and has pointwise finite range. It is properly PAC learnable, but no weighted-SRM objective $h\mapsto\emp_S(h)+\lambda(S)\,\psi(h)$ learns it.
\end{theorem}

\begin{proof}
We first establish proper learnability and the topological properties, and then prove the lower bound.

\vspace{-0.4 cm}
\paragraph{Proper learnability within one incidence component.}
We first describe a learner when the target and marginal are supported on one finite component. If the sample contains a private label $\tau_u$, simply output $g_u$. Otherwise, every observed label is $\starlabel$; let $I(S)$ denote the collection of blocks $Z_v$ appearing in the sample.
If $|I(S)|\ge2$, realizability implies that the target index $u$ belongs to $\mathsf S(v)$ for every $v\in I(S)$. By \eqref{eq:incidence-small-intersection}, the intersection $\bigcap_{v\in I(S)}\mathsf S(v)$ contains at most one vertex. Since it contains the target, it identifies $u$ exactly, and the learner outputs $g_u$. If $I(S)=\{v\}$, the learner outputs the owner $g_v$.

Fix a target $g_u$ and a marginal $D$ supported on the component. For every $v$ satisfying $u\in\mathsf S(v)$, define
\[
    A_{u,v}
    :=
    \{x\in Z_v:g_u(x)=\starlabel\},
    \qquad
    e_{u,v}
    :=
    \pop_D(g_v,g_u).
\]
The learner can output an incorrect owner $g_v$ only if every sampled point lies in $A_{u,v}$. On this set the target and owner agree, so $D(A_{u,v})\le 1-e_{u,v}$. Moreover, the sets $A_{u,v}$ are disjoint because they lie in distinct blocks. Hence, for a sample of size $k\ge1$ and every $\epsilon\in(0,1)$,
\begin{align}
\label{eq:component-learning-bound}
    \P\!\left[
        \pop_D(A(S),g_u)>\epsilon
    \right]
    &\le
    \sum_{v:e_{u,v}>\epsilon}D(A_{u,v})^k
    \notag\\
    &\le
    (1-\epsilon)^{k-1}
    \sum_vD(A_{u,v})
    \notag\\
    &\le
    (1-\epsilon)^{k-1}.
\end{align}
Crucially, this estimate is independent of the size of the component.

\vspace{-0.4 cm}
\paragraph{Proper learnability of the countable union.}
We now define the learner on the full class. If every observed label is $\dollarlabel$, output $h_{\dollarlabel}$. Otherwise the non-$\dollarlabel$ examples lie in one incidence component; discard the $\dollarlabel$ examples and run the component learner above on the active subsample.

Fix a target $g_u$, let $D$ be an arbitrary marginal on $\CX^{\mathrm I}$, and write $p$ for the mass of the target's active component. Outside that component, both the target and every hypothesis from the component predict $\dollarlabel$. If $p\le\epsilon$, every output of the preceding learner has population error at most $\epsilon$.

Suppose therefore that $p>\epsilon$, and let $K\sim\operatorname{Bin}(r,p)$ be the number of active observations in a sample of size $r$. If $K=0$, the learner outputs $h_{\dollarlabel}$, and $\P[K=0]=(1-p)^r\le e^{-r\epsilon}$. Conditional on $K=k\ge1$, the active examples are iid from the conditional marginal on the component. Apply \eqref{eq:component-learning-bound} with threshold $\alpha:=\epsilon/(2p)\le1/2$. If the global population error exceeds $\epsilon$, then the conditional component error exceeds $\epsilon/p$, and hence also exceeds $\alpha$. Consequently,
\begin{align*}
    \P\!\left[
        \pop_D(A(S),g_u)>\epsilon
    \right]
    &\le
    (1-p)^r
    +
    \sum_{k=1}^r
        \binom{r}{k}p^k(1-p)^{r-k}(1-\alpha)^{k-1}
    \\
    &\le
    e^{-r\epsilon}
    +
    \frac{(1-p+p(1-\alpha))^r}{1-\alpha}
    \\
    &\le
    e^{-r\epsilon}
    +
    2(1-\epsilon/2)^r
    \\
    &\le
    3e^{-r\epsilon/2}.
\end{align*}
The all-$\dollarlabel$ target is learned exactly. Thus $\Hinc$ is properly PAC learnable, with a component-size-independent sample complexity.

\vspace{-0.4 cm}
\paragraph{Countability, pointwise finite range, and closedness.}
Every incidence component is finite, and there are countably many components; hence $\Hinc$ is countable. At a fixed coordinate $x$, only the finitely many hypotheses from its own component can output a value other than $\dollarlabel$, so $\Hinc(x)$ is finite.
For closedness, invoke the finite-projection characterization from \Cref{lem:closed-fpp}. Let $f:\CX^{\mathrm I}\to\CY$ be finitely consistent with $\Hinc$. If $f\equiv\dollarlabel$, then $f=h_{\dollarlabel}$. Otherwise choose a point $x$ in some component $\CX_m^{\mathrm I}$ for which $f(x)\neq\dollarlabel$. Two-point restrictions force $f$ to equal $\dollarlabel$ outside $\CX_m^{\mathrm I}$, since no nontrivial hypothesis is active on two components. The component is finite, so finite consistency applied to the entire component produces a hypothesis $g_u$ agreeing with $f$ there. Both functions equal $\dollarlabel$ elsewhere, and hence $f=g_u$. Therefore $\Hinc$ is closed.

\vspace{-0.4 cm}
\paragraph{An incidence inversion for every scalar regularizer.}
Fix a sample scale $m$ and an arbitrary regularizer $\psi$. For each $r\in\mathbb Z/3\mathbb Z$, let $a_r$ be a median of the values $\{\psi(g_v):v\in V_r^{(m)}\}$. For some cyclic index $r$, one has $a_{r+1}\le a_r$. Let $U$ be the vertices in $V_r^{(m)}$ whose regularizer values are at least $a_r$, and let $W_0$ be the vertices in $V_{r+1}^{(m)}$ whose regularizer values are at most $a_{r+1}$. Both sets have cardinality at least $N_m/2$.
The projective-plane mixing estimate gives $e(U,W_0)\ge (d_m/N_m)|U||W_0|-\sqrt{q_m|U||W_0|}\ge d_mN_m/8$, where the final inequality uses $q_m\ge16$. Averaging over $U$ yields a vertex $v\in U$ with at least $d_m/8$ outgoing neighbors in $W_0$. For $W:=\Gamma^+(v)\cap W_0$, we therefore have
\begin{equation}
\label{eq:incidence-inversion}
    |W|\ge\frac{d_m}{8}
    \qquad\text{and}\qquad
    \psi(g_w)
    \le a_{r+1}
    \le a_r
    \le\psi(g_v)
    \quad
    \forall w\in W.
\end{equation}

\vspace{-0.4 cm}
\paragraph{Failure of weighted SRM.}
Fix an arbitrary regularizer $\psi$ and coefficient $\lambda$. By definition, a weighted-SRM objective must attain its minimum on every sample in order to induce a learner; an objective with an empty argmin is therefore already disqualified. The lower bound below does not exploit this technicality. On each hard sample, the version space is finite and contains the owner $g_v$ together with its surviving coordinate neighbors. If an attained global minimizer lies outside the version space, it is already population-bad; if a minimizer lies among the interpolators, we force a bad interpolator into the argmin set.

For every $m\ge2$, apply \eqref{eq:incidence-inversion} inside the $m$th component to obtain $v$ and $W\subseteq\Gamma^+(v)$ such that $|W|\ge d_m/8>2^{m+7}$ and $\psi(g_w)\le\psi(g_v)$ for every $w\in W$. Take target $g_v$, and let $D$ be uniform on the full cube $Z_v$. The owner $g_v$ is the unique population-perfect hypothesis. Every neighbor $g_w$, $w\in\Gamma^+(v)$, has population error $1/2$, while every remaining hypothesis has population error one.

For a fixed $w\in W$, the hypothesis $g_w$ interpolates an $m$-sample exactly when the $w$th coordinate is zero on every sampled vector, an event of probability $2^{-m}$. These events are independent over distinct $w$, and therefore
\begin{align}
\label{eq:bad-incidence-survivor}
    \P\!\left[
        \exists w\in W:\emp_S(g_w)=0
    \right]
    &=
    1-(1-2^{-m})^{|W|}
    \notag\\
    &\ge
    1-\exp(-|W|2^{-m})
    \notag\\
    &\ge
    1-e^{-128}.
\end{align}

On the event in \eqref{eq:bad-incidence-survivor}, choose an interpolating $g_w$ with $w\in W$. Then $\emp_S(g_w)=\emp_S(g_v)=0$ and $\psi(g_w)\le\psi(g_v)$, so $\emp_S(g_w)+\lambda(S)\psi(g_w)\le\emp_S(g_v)+\lambda(S)\psi(g_v)$. If $g_v$ is not an objective minimizer, every minimizer is population-bad because $g_v$ is the unique population-perfect hypothesis. If $g_v$ is a minimizer, the preceding inequality forces $g_w$ to be a minimizer as well. Hence the argmin set contains a nonowner of population error at least $1/2$ on the event \eqref{eq:bad-incidence-survivor}. 
Fixing an enumeration of $\Hinc$, choose the first nonowner minimizer on such samples whenever one exists, and the first minimizer otherwise; this defines an induced learner whose population error is at least $1/2$ with probability at least $1-e^{-128}$ for arbitrarily large $m$, so the weighted-SRM objective does not learn $\Hinc$.
\end{proof}

\begin{remark}[Existential versus universal selection]
\label{rem:incidence-existential-universal}
The incidence class does admit a successful consistent proper learner, namely the learner constructed in the proof of \Cref{thm:weighted-srm-fails}. Under an existential convention for selecting minimizers, one could therefore set $\lambda\equiv0$ and choose that learner from the version space. But the objective would then assign exactly the same score to every interpolating hypothesis: all of the learning power would reside in a hard-coded, sample-dependent sequence of tie-breaking decisions, rather than in any preference expressed by the objective. This is precisely what the universal convention in \cref{Equation:1,Equation:2,Equation:3} is designed to exclude.
\end{remark}

Combining the incidence class with the poisoned first-Cantor construction yields one class exhibiting both global obstructions.

\begin{theorem}
\label{thm:combined-global-negative}
There exists a countable hypothesis class $\Hc$ satisfying all of the following:
\begin{enumerate}[label=(\roman*),topsep=0.3em,itemsep=0.25em]
    \item $\Hc$ is properly PAC learnable;
    \item $\Hc$ is closed and has pointwise finite range;
    \item no interpolating proper learner PAC learns $\Hc$;
    \item no weighted-SRM objective learns $\Hc$.
\end{enumerate}
\end{theorem}
\begin{proof}
Take the disjoint union of $\Hpoison$ and $\Hinc$, identify their all-$\dollarlabel$ hypotheses, and extend every remaining hypothesis by $\dollarlabel$ on the opposite domain.  Proper learnability follows by running the two component learners and validating between their outputs on a fresh subsample, while closedness and pointwise finite range are preserved by the finite disjoint union.  Restricting the marginal to the poisoned component recovers \Cref{thm:no-erm}, whereas restricting it to the incidence component recovers \Cref{thm:weighted-srm-fails}; in either case, hypotheses from the opposite component predict $\dollarlabel$ throughout the hard distribution and therefore cannot rescue the constrained learner.
\end{proof}

\subsection{A PAC counterexample to hard local regularization}
\label{subsec:local-counterexample}

We now resolve \Cref{problem1:local-regularizer} in the negative. Recall that the local regularizers of \citet{asilis2024regularization} are first learned from the unlabeled sample, after which they may depend upon the test point. The open problem asks whether this unsupervised pre-training stage is truly necessary: perhaps every learnable multiclass problem admits a single local preference function, fixed in advance, whose favorite interpolating hypothesis makes the correct prediction at each test point. We demonstrate that this is not the case. Indeed, the obstruction already occurs in a class where every hypothesis uses only two labels; see also the related transductive separation of \citet{jafar2025local}.

\begin{definition}[Tripartite orientation class]
\label{def:tripartite-orientation-class}
For every integer $m\ge2$, set $N_m:=2^{m-1}$ and let $G_m=K_{N_m,N_m,N_m}$ be the complete tripartite graph whose three vertex parts each have cardinality $N_m$. Write $E_m$ for its edge set, so $|E_m|=3N_m^2$ and every vertex has degree $2N_m$.

Let $X_m:=\{0,1\}^{E_m}$ denote the set of all orientations of $G_m$, where the bit indexed by an edge determines which endpoint is its head. For $z\in X_m$ and $e=\{u,v\}\in E_m$, write $\operatorname{head}_z(e)\in\{u,v\}$ for the head of $e$ in the orientation $z$. The vertex sets of the graphs $G_m$ are pairwise disjoint and serve as labels, and the full domain is $\CX^{\mathrm{or}}:=\bigsqcup_{m\ge2}X_m$.

For every edge $e=\{u,v\}\in E_m$, fix one distinguished endpoint $\ell(e)\in\{u,v\}$ and define
\[
    h_e(x)
    :=
    \begin{cases}
        \operatorname{head}_x(e),
            & x\in X_m,\\
        \ell(e),
            & x\notin X_m.
    \end{cases}
\]
The \edefn{tripartite orientation class} is $\Hor:=\{h_e:e\in E_m,\ m\ge2\}$.
\end{definition}

Every hypothesis $h_e$ uses exactly the two endpoint labels of $e$, and distinct edges have distinct unordered endpoint pairs. Thus observing both labels used by the target identifies it exactly. Nevertheless, the class cannot be learned by a fixed local preference rule.

\begin{theorem}
\label{thm:local-regularization-fails}
\label{prop:orientation-learnable}
The tripartite orientation class $\Hor$ is PAC learnable, but it cannot be learned by any hard local regularizer.
\end{theorem}

\begin{proof}
We first prove that the class is learnable, and then establish the lower bound against an arbitrary local regularizer.

\vspace{-0.4 cm}
\paragraph{PAC learnability.}
Consider the following improper learner. If two distinct labels $a$ and $b$ appear in the sample, then the unordered pair $\{a,b\}$ identifies the unique target edge $e=\{a,b\}$, and the learner outputs $h_e$. If exactly one label $a$ appears in the sample, then a fresh test point is unlikely to be accompanied by a different label, and a learner can simply emit the constant-$a$ classifier. 

\vspace{-0.4 cm}
\paragraph{Reduction to a strict pointwise order.}
Fix an arbitrary local regularizer $\psi:\Hor\times\CX^{\mathrm{or}}\to\R_{\ge0}$. Since a local regularizer learns only if every induced learner succeeds, we may refine each pointwise tie according to a fixed enumeration of the countable class. This produces a strict total order $\prec_z$ on the hypotheses at every test point $z$. The associated local learner selects the $\prec_z$-least member of the version space and predicts with it at $z$. We prove the lower bound for an arbitrary collection of strict orders $(\prec_z)_z$.

\vspace{-0.4 cm}
\paragraph{A hard family of target-distribution pairs.}
Fix a sample size $m$ and work in the component $G_m$. For brevity, write $N:=N_m=2^{m-1}$ and $p:=2^{-m}$. Let $e\in E_m$ be incident to a vertex $a$. Define $D_{a,e}$ to be the uniform distribution on orientations $z\in X_m$ satisfying $\operatorname{head}_z(e)=a$, and take $h_e$ as the target. Under this target-distribution pair, every example has label $a$. Furthermore, for a realized sample $S=((z_1,a),\ldots,(z_m,a))$, the version space is exactly
\[
    \vs(S)
    =
    \Big \{
        h_f:
        f\in E_m,\ a\in f,\
        \operatorname{head}_{z_i}(f)=a
        \text{ for every }i\in[m]
    \Big \}.
\]
Indeed, an edge not incident to $a$ can never output the label $a$, while hypotheses from other components use disjoint labels. The target edge $e$ always survives. Every other edge $f$ incident to $a$ survives independently with probability $p=2^{-m}$, since its orientation is an independent fair bit in each sampled graph orientation.

\vspace{-0.4 cm}
\paragraph{Every directed triangle creates an error witness.}
Fix a test orientation $z\in X_m$, and consider a directed cyclic triangle in $z$. Let $f$ be the earliest of its three edges under the strict local order $\prec_z$. Let $a$ be the tail of $f$ in the directed cycle, and let $e$ be the edge immediately preceding $f$, which points into $a$. Then $\operatorname{head}_z(e)=a$, while $\operatorname{head}_z(f)\neq a$ and $f\prec_z e$. Thus $(a,e)$ defines one of the hard target-distribution pairs above, while the earlier edge $f$ predicts incorrectly at the test point $z$.

Consider the event that $f$ survives the training sample and every edge incident to $a$ that precedes $f$ under $\prec_z$ fails to survive. The target edge $e$ need not be eliminated, since $f\prec_z e$. A vertex has degree $2N$, and the non-target survival events are independent. Consequently, this event has probability at least
\begin{equation}
\label{eq:triangle-witness-probability}
    p(1-p)^{2N}.
\end{equation}
Whenever it occurs, $f$ is the first surviving edge incident to $a$, so the local learner predicts $h_f(z)=\operatorname{head}_z(f)\neq a=h_e(z)$.
For a fixed target incidence $(a,e)$, the events corresponding to distinct witness edges $f$ are disjoint: at most one edge can be the first survivor under $\prec_z$. Hence the probabilities in \eqref{eq:triangle-witness-probability} may be summed over all directed triangles producing that target incidence.

\vspace{-0.4 cm}
\paragraph{Averaging over the tripartite geometry.}
Draw a test orientation $z$ uniformly from $X_m$, choose an edge $e$ uniformly from $E_m$, set $a:=\operatorname{head}_z(e)$, and then draw the training sample independently from $D_{a,e}^m$. Conditional on $(a,e)$, the test point $z$ is itself distributed according to $D_{a,e}$.
Let $C(z)$ denote the number of directed cyclic triangles in $z$. Conditional on $z$, the expected test error averaged over the uniformly random target edge is at least $C(z)p(1-p)^{2N}/|E_m|$. The graph $G_m$ contains $N^3$ triangles, and a uniformly random orientation makes any fixed triangle cyclic with probability $1/4$. Therefore $\E_z[C(z)]=N^3/4$. Since $|E_m|=3N^2$, averaging over $z$ gives
\begin{align}
\label{eq:tripartite-average-error}
    \E_{z,e,S}
    \one\!\left\{
        A(S)(z)\neq h_e(z)
    \right\}
    &\ge
    \frac{N^3/4}{3N^2}
    \,p(1-p)^{2N}
    \notag\\
    &=
    \frac{Np}{12}(1-p)^{2N}.
\end{align}
With $N=2^{m-1}$ and $p=2^{-m}$, the right-hand side equals $\frac1{24}(1-2^{-m})^{2^m}$, which is at least $1/100$ for every $m\ge2$.
The expectation in \eqref{eq:tripartite-average-error} is an average over the finitely many target incidences $(a,e)$. Thus, for every $m$, some fixed pair $(h_e,D_{a,e})$ satisfies $\E_S\pop_{D_{a,e}}(A(S),h_e)\ge1/100$. This completes the argument.
\end{proof}

\begin{remark}[Weighted local regularization]
\label{rem:weighted-local-open}
Our proof is specific to an ``interpolation-first'' form of local regularization. On the hard samples, it identifies a locally preferred hypothesis among those having zero empirical loss and shows that this surviving interpolator predicts the test point incorrectly. One can imagine a \emph{weighted} local rule that is more permissive, by instead choosing
\[
    h_{S,x}
    \in
    \argmin_{h\in\Hc}
    \left\{
        \emp_S(h)+\lambda(S) \, \psi(h,x)
    \right\}
\]
at each test point $x$ and predicting the classifier $x \mapsto h_{S,x}(x)$. Such a rule may prefer a noninterpolating hypothesis that incurs some training error but nevertheless predicts $x$ correctly, so the survival argument above no longer applies.
However, in \Cref{thm:figure-hypergraph,thm:figure-geometric,thm:figure-layered}
we demonstrate that even this more powerful weighted local rule fails on certain
learnable problems. It is, nevertheless, strictly more powerful than ``interpolation-first''
local regularization; see \cref{thm:figure-tagged}.
\end{remark}

\section{Structural Conditions and Integrability} \label{Section:Structural-Conditions-and-Integrability} The preceding sections rule out several universal algorithmic templates. We now ask a more constructive question: when does regularization succeed? Two answers emerge. The first is statistical: though the full hypothesis class may be enormous, the hypotheses still relevant after observing the sample may occupy a simple region of the class. The second is algebraic: even when a successful learner makes genuinely sample-dependent choices, those choices may happen to fit together under one global potential.

\subsection{Fallback localization}
\label{subsec:fallback-cores}

Many pathological multiclass classes share a reveal-or-fallback structure.  In particular, the appearance of a private label---one associated with a unique hypothesis---instantly identifies the true labeling function, whereas samples containing only the common, public labels reveal less but suggest a simple default explanation. Roughly speaking, for such hypothesis classes, learning is possible because every sample $S$ either reveals the target function (i.e., $|\vs(S)| = 1$) or the version space $\vs(S)$ contains one of a few statistically safe ``fallback'' hypotheses.  We begin by isolating this principle.

\begin{definition}
A finite nonempty set $\F\subseteq\Hc$ is a \edefn{fallback core} for $\Hc$ if
\[
    \F\cap\vs(S)\neq\varnothing
\]
for every ambiguous realizable sample $S$ (i.e., with $|V(S)| > 1$).
\end{definition}

Thus, whenever the sample has failed to identify the target, at least one member of the same fixed finite family remains consistent. In this case, learning can be achieved by a simple regularizer that prioritizes the fallback core.

\begin{example}[Fallback completions]
The first Cantor class acquires a singleton fallback core after adjoining the all-$\starlabel$ hypothesis.  Indeed, a private label uniquely identifies the target, while every remaining ambiguous sample contains only $\starlabel$ labels and is therefore consistent with the added hypothesis.  The same observation applies to the EMX-derived class of \citet{asilis25b-proper} after taking its all-$\starlabel$ completion.

\vspace{-0.18 cm}
Our incidence and tripartite constructions exhibit the same reveal-or-fallback behavior after similarly natural completions, although their fallback families are now infinite.  For the incidence class, adjoin for each component $\CX_m^{\mathrm I}$ a hypothesis $f_m^\star$ that predicts $\starlabel$ throughout $\CX_m^{\mathrm I}$ and $\dollarlabel$ elsewhere.  A private label reveals the target; until one appears, either $h_{\dollarlabel}$ or the componentwise fallback $f_m^\star$ remains consistent.  For the tripartite orientation class, adjoin the constant hypothesis $c_a\equiv a$ for every vertex label $a$.  A sample containing two distinct labels identifies the target edge, while a sample containing only the label $a$ remains consistent with $c_a$.

\vspace{-0.18 cm}
These latter families are infinite, thus not fallback cores in the sense above, yet they are statistically very simple in another manner: each has graph dimension 1. We will soon develop techniques for handling such infinite yet low-dimensional fallback structures (\cref{thm:ambiguity-localized}).
\end{example}

\begin{theorem}
\label{thm:fallback-core}
If $\Hc$ has a fallback core $\F$ of size $K$, then $\Hc$ is learnable by hard SRM.  More precisely, there is a hard-SRM learner with sample complexity
\[
    O\!\left(
        \frac{\log K+\log(1/\delta)}{\epsilon}
    \right).
\]
\end{theorem}

\begin{proof}
Set $\psi(h)=0$ for $h\in\F$ and $\psi(h)=1$ otherwise.  On a nonambiguous sample, the version space is the singleton $\{h^\star\}$; on an ambiguous sample, every minimum-$\psi$ consistent hypothesis belongs to $\F$.  For each $f\in\F$ with $\pop_D(f,h^\star)>\epsilon$, the probability that $f$ remains consistent with an $m$-sample is at most $(1-\epsilon)^m$.  A union bound therefore gives
\[
    \P\!\left[
        \pop_D(A(S),h^\star)>\epsilon
    \right]
    \le
    K(1-\epsilon)^m
    \le
    K e^{-m\epsilon},
\]
simultaneously for every fixed choice of tie-breaking.
\end{proof}

It is not difficult to see that the existence of fallback cores admits an exact behavioral converse, by considering consistent proper learners with restricted range.

\begin{theorem}
\label{thm:finite-ambiguous-range}
A class $\Hc$ has a finite fallback core if and only if it admits a consistent proper PAC learner whose range on ambiguous realizable samples is finite.
\end{theorem}

\begin{proof}
If $\F$ is a fallback core, the hard-SRM learner constructed in \Cref{thm:fallback-core} is consistent, properly PAC learns $\Hc$, and emits only members of $\F$ on ambiguous samples.  Conversely, let $A$ be a consistent proper PAC learner whose range on ambiguous samples is a finite family $\F$.  For every ambiguous realizable sample $S$, consistency gives $A(S)\in\F\cap\vs(S)$, so $\F$ is a fallback core.
\end{proof}

The fallback core definition is quite restrictive: we now consider a broader localization principle, by which a regularizer forces its output into a statistically manageable region whenever the sample has not already identified the target. Notably, this fallback region may now grow with the sample size, provided its complexity grows sufficiently slowly. We use $\Gdim(\Hc)$ to denote the Graph dimension of $\Hc$. 

\begin{theorem}
\label{thm:ambiguity-localized}
Let $\psi:\Hc\to\R_{\ge0}$ be a regularizer whose minimum is attained on every realized version space, and define
\[
    \Hc_r:=\{h\in\Hc:\psi(h)\le r\}.
\]
Suppose that for every sample size $m$ there exists $r_m$ such that every ambiguous (realizable) $m$-sample satisfies $\vs(S)\cap\Hc_{r_m}\neq\varnothing$.
If $d_m:=\Gdim(\Hc_{r_m})$ satisfies $d_m\log(m+1)=o(m)$, then every hard-SRM selector induced by $\psi$ PAC learns $\Hc$.
\end{theorem}
\begin{proof}
Let $\widehat h$ be an arbitrary minimum-$\psi$ member of $\vs(S)$.  If $S$ is nonambiguous, then $\widehat h=h^\star$.  If $S$ is ambiguous, some consistent hypothesis belongs to $\Hc_{r_m}$, so minimality gives $\psi(\widehat h)\le r_m$ and hence $\widehat h\in\Hc_{r_m}$. The standard graph-dimension bound for consistent hypotheses gives
\[
    \P\!\left[
        \exists h\in\Hc_{r_m}:
        \emp_S(h)=0,\ 
        \pop_D(h,h^\star)>\epsilon
    \right]
    \le
    \exp\!\left(
        O(d_m\log(m+1))-\frac{\epsilon m}{2}
    \right).
\]
The right-hand side tends to zero for every fixed $\epsilon>0$, uniformly over all realizable target-distribution pairs.  Since every possible minimizer is confined to the same sublevel class, the conclusion holds for every fixed tie-breaking rule.
\end{proof}

Note that in \cref{thm:ambiguity-localized}, the full class $\Hc$ may have infinite graph dimension and exhibit no useful uniform convergence whatsoever. The theorem invokes uniform convergence only after the sample has failed to identify the target, and only within the low-complexity sublevel to which regularization confines its output.
A weighted analogue, in which the comparator may incur a controlled amount of empirical error, is given in \Cref{thm:weighted-approx}.

\subsection{Ordered disagreement complexity}
\label{subsec:odd}

Fallback localization is phrased in terms of the version space left by a sample.  We now consider a complementary viewpoint, which fixes the target $h^*$ and asks which disagreement sets can be generated by examining hypotheses that are preferred to $h^*$ (by an ambient regularizer under consideration).

A brief bit of notation: for $g,h\in\Hc$, write $\Delta(g,h):=\{x\in\CX:g(x)\neq h(x)\}$ for their disagreement set.
Given a regularizer $\psi$, define the lower contour and its disagreement family by
\[
    \mathcal L_\psi(h)
    :=
    \{g\in\Hc:\psi(g)\le\psi(h)\},
    \qquad
    \mathcal D_\psi(h)
    :=
    \{\Delta(g,h):g\in\mathcal L_\psi(h)\}.
\]
In words, $\mathcal L_\psi(h)$ records the hypotheses that are preferred to $h$, while $\mathcal D_\psi(h)$ records the differences of all such hypotheses with $h$.

\begin{definition}
For a regularizer $\psi$, define the \edefn{ordered disagreement dimension} $d_{\mathrm{OD}}(\psi)$ as $d_{\mathrm{OD}}(\psi) := \sup_{h\in\Hc} \VC\bigl(\mathcal D_\psi(h)\bigr)$.
The corresponding class parameter is
\[
    d_{\mathrm{OD}}^\star(\Hc)
    :=
    \inf_{\psi}d_{\mathrm{OD}}(\psi),
\]
where the infimum ranges over all regularizers attaining their minima on realized version spaces.
\end{definition}

\begin{theorem}
\label{thm:odd-suffices}
Suppose $\psi$ attains its minimum on every realized version space and
$d_{\mathrm{OD}}(\psi)=d<\infty$.  Then every hard-SRM selector induced by
$\psi$ PAC learns $\Hc$ with sample complexity
\[
    O\!\left(
        \frac{d\log(1/\epsilon)+\log(1/\delta)}{\epsilon}
    \right).
\]
\end{theorem}

\begin{proof}
Fix a target $h^\star$.  Since $h^\star$ is consistent, every hard-SRM output
$\widehat h$ satisfies $\psi(\widehat h)\le\psi(h^\star)$.  Hence
\[
    \Delta(\widehat h,h^\star)\in\mathcal D_\psi(h^\star).
\]
Consistency means that the unlabeled sample misses this disagreement set.  By
the VC $\epsilon$-net theorem, a sample of the stated size intersects every
member of $\mathcal D_\psi(h^\star)$ having $D$-mass greater than $\epsilon$,
with probability at least $1-\delta$.  Therefore
$\pop_D(\widehat h,h^\star)\le\epsilon$.
\end{proof}

The advantage of ordered disagreement is that it is target-relative.  For a fixed target $h^\star$, hard SRM can emit only hypotheses $g$ satisfying $\psi(g)\leq\psi(h^\star)$; hypotheses ranked above the target are therefore irrelevant, regardless of how complicated their disagreement patterns may be.  Consequently, $d_{\mathrm{OD}}(\psi)$ can remain finite even when the full class has infinite Graph dimension.  This sufficient condition is not necessary, however: a regularizer may rank many statistically complicated hypotheses below the target even though none of them can ever be selected from a realized version space.  The following example isolates precisely this gap.

\begin{example}
\label{ex:revealing-fallback}
Let $\CX=\bigsqcup_{n\ge1}X_n$, where $|X_n|=2n$.  Include the all-$\dollarlabel$ hypothesis and, for every $n$, the fallback
\[
    f_n(x)
    =
    \begin{cases}
        \starlabel, & x\in X_n,\\
        \dollarlabel, & x\notin X_n.
    \end{cases}
\]
For every nonempty $A\subseteq X_n$ with $|A|\le n$, introduce a private label $\lambda_{n,A}$ and the revealing hypothesis
\[
    c_{n,A}(x)
    =
    \begin{cases}
        \lambda_{n,A}, & x\in A,\\
        \starlabel, & x\in X_n\setminus A,\\
        \dollarlabel, & x\notin X_n.
    \end{cases}
\]
Let $\Hrf$ be the resulting class, and consider the hard SRM rule that ranks $h_{\dollarlabel}$ first, then all fallbacks, and then all revealing hypotheses. It is not difficult to see that this rule learns $\Hrf$. 

Consider, however, an arbitrary regularizer $\psi$.  For each $n$, choose a maximum-$\psi$ revealing hypothesis $c_{n,A^\star}$ in the $n$th block.  Since $|A^\star|\le n$, there exists $T\subseteq X_n\setminus A^\star$ with $|T|=n$.  For every nonempty $U\subseteq T$, maximality gives
$\psi(c_{n,U})\le\psi(c_{n,A^\star})$, while  $\Delta(c_{n,U},c_{n,A^\star})\cap T=U$.
The empty subset is realized by the center itself.  Thus the lower contour of
$c_{n,A^\star}$ shatters $T$, and $d_{\mathrm{OD}}(\psi)\ge n$.  Since $n$ is
arbitrary, $d_{\mathrm{OD}}^\star(\Hrf)=\infty$, despite the existence of a
successful hard-SRM rule.
\end{example}

The failure of \cref{ex:revealing-fallback} arises from hypotheses that lie below the target in the regularizer order but can never actually be selected from a version space.  This suggests restricting attention to genuine minimizers.

\begin{definition}
For a regularizer $\psi$ and target $h$, let $\operatorname{Sel}_\psi(h)$ denote the collection of all hypotheses in $\CH$ that are most-preferred by $\psi$ on an $h$-realizable sample. That is,
\[
    \operatorname{Sel}_\psi(h)
    :=
    \left\{
        g\in\Hc:
        \begin{array}{l}
        \text{there exists a realizable sample $S$ labeled by $h$}\\[-0.1em]
        \text{such that }
        g\in\argmin_{u\in\vs(S)}\psi(u)
        \end{array}
    \right\}.
\]
Then the \edefn{selected ordered disagreement dimension} of $\psi$ is 
\[
    d_{\mathrm{SOD}}(\psi)
    :=
    \sup_{h\in\Hc}
    \VC\{\Delta(g,h):g\in\operatorname{Sel}_\psi(h)\}.
\]
\end{definition}

We now demonstrate that finiteness of the SOD dimension suffices to ensure that $\psi$ learns $\CH$, using standard $\varepsilon$-net properties enjoyed by VC classes. Whether finiteness of the SOD dimension is necessary for $\psi$'s success, however, remains open.

\begin{proposition}
\label{prop:sod-suffices}
If $d_{\mathrm{SOD}}(\psi)<\infty$, then every hard-SRM selector induced by
$\psi$ PAC learns $\Hc$.
\end{proposition}

\begin{proof}
Under target $h^\star$, every possible hard-SRM output belongs to
$\operatorname{Sel}_\psi(h^\star)$ and is consistent.  Apply the same
$\epsilon$-net argument as in \Cref{thm:odd-suffices} to the selected
disagreement family.
\end{proof}

\subsection{Integrability of revealed preferences}
\label{subsec:integrability}

The preceding results ask when a particular regularizer learns.  We now reverse the question.  Suppose a successful learner $A$  is already given: when can its sample-dependent choices be represented by one regularizer? The answer turns out to be intimately related to the revealed preferences of $A$.

In particular, let $A$ be a deterministic consistent proper learner. Define its revealed-preference relation $\prec_A$ by
\[
    A(S)\prec_A g
    \qquad
    \text{whenever }
    g\in\vs(S)\setminus\{A(S)\}.
\]
Thus $h\prec_A g$ means that the learner has selected $h$ while $g$ remained feasible.

\begin{theorem}
\label{thm:hard-integrability}
Let $\Hc$ be countable.  An arbitrary consistent learner $A$ is representable by an injective hard-SRM regularizer if and only if $\prec_A$ is acyclic.
\end{theorem}

\begin{proof}
If $A$ is induced by an injective regularizer $\psi$, every revealed edge
$h\prec_A g$ satisfies $\psi(h)<\psi(g)$, so a directed cycle is impossible.
Conversely, suppose $\prec_A$ is acyclic.  Its transitive closure is a strict
partial order, which may be extended to a strict total order on the countable
set $\Hc$.  Every countable total order embeds into $\Q$; let
$q:\Hc\to\Q$ be an injective order-preserving embedding.  Composing with the
strictly increasing map
\[
    t\longmapsto \frac12+\frac1\pi\arctan(t)
\]
gives an injective regularizer $\psi:\Hc\to(0,1)$.  For every sample $S$,
the relation places $A(S)$ strictly below every other member of $\vs(S)$.
Hence $A(S)$ is the unique $\psi$-minimum of the version space.
\end{proof}

The same result has a choice-theoretic formulation.  Recall that a choice function $C$ on nonempty finite menus is \emph{contraction consistent} if $C(B)\in M\subseteq B \Longrightarrow C(M)=C(B)$.
We say that $C$ extends the learner $A$ if, for every realized sample $S$ and every finite menu $M$ satisfying $A(S)\in M\subseteq\vs(S)$,
one has $C(M)=A(S)$.

\begin{proposition}[Finite-menu extension]
\label{prop:menu-extension}
A consistent learner on a countable class is representable by hard SRM if and only if its choices admit a contraction-consistent extension to all nonempty finite menus.
\end{proposition}

\begin{proof}
If $A$ is induced by an injective regularizer, choose from each finite menu
its unique minimum.  This choice function is contraction consistent and
extends $A$.
Conversely, suppose $C$ is such an extension.  Define $h\prec g$ when
$C(\{h,g\})=h$.  This relation is a tournament.  It has no directed
three-cycle: if $h\prec g\prec k\prec h$, then the choice from
$\{h,g,k\}$ would contradict contraction consistency after restriction to
one of the three pairs.  A tournament without directed three-cycles is
transitive, and hence defines a strict total order.
For every realized sample $S$ and competitor
$g\in\vs(S)\setminus\{A(S)\}$, compatibility on the finite menu
$\{A(S),g\}$ gives $A(S)\prec g$.  Thus $A(S)$ is the unique order-minimum
of its version space.  Apply \Cref{thm:hard-integrability}.
\end{proof}

The previous results concern hard SRM learners, which select a most-favored hypotheses from the version space $\vs(S)$ of interpolating hypotheses for $S$. We now consider the more general weighted SRM learners, that minimized the sum of empirical risk and (weighted) regularization value. 
For finite hypothesis sets and finite collections of samples, we obtain an exact characterization: the learner's prescribed choices admit a weighted-SRM representation precisely when every directed cycle in a naturally associated weighted preference graph has positive total weight. We suspect that it should be possible to extend this cycle criterion to infinite systems, by invoking certain compactness assumptions, but that remains an interesting direction for future work. 

\begin{definition}
Let $H$ be a finite hypothesis class, let $A$ be a deterministic $H$-valued learner, and let $\lambda(S)>0$ be fixed for every sample $S$.  For a finite family of samples $\Scal$, the \edefn{weighted revealed-preference multigraph} $G_{A,\lambda}(\Scal)$ has vertex set $H$, and for every $S\in\Scal$ and every competitor $g\in H\setminus\{A(S)\}$, it contains an edge $g\longrightarrow A(S)$
of weight
\[
    w_{A,\lambda}(S,g)
    :=
    \frac{\emp_S(g)-\emp_S(A(S))}{\lambda(S)}.
\]
Different samples may generate parallel edges between the same pair of hypotheses.
\end{definition}

\begin{theorem}[Strict cycle characterization]
\label{thm:weighted-integrability}
Fix a finite hypothesis set $H$, a deterministic $H$-valued learner $A$, positive coefficients $\lambda(S)$, and a finite family of samples $\Scal$.  There exists a regularizer $\psi:H\to\R_{\ge0}$ for which $A(S)$ is the unique minimizer of
\[
    h\longmapsto\emp_S(h)+\lambda(S)\psi(h)
\]
for every $S\in\Scal$ if and only if every directed cycle in $G_{A,\lambda}(\Scal)$ has strictly positive total weight.
\end{theorem}

\begin{proof}
For a sample $S$ and competitor $g\neq A(S)$, strict preference for $A(S)$ is equivalent to
\[
    \psi(A(S))-\psi(g)
    <
    w_{A,\lambda}(S,g).
\]
Summing these inequalities around a directed cycle makes the potential terms telescope, so every cycle must have positive total weight.
Conversely, suppose every directed cycle has positive total weight.  Since the multigraph is finite, the minimum mean weight of a directed cycle is some $\mu>0$; if the graph is acyclic, choose any $\mu>0$.  Fix $0<\eta<\mu$ and subtract $\eta$ from every edge weight.  The adjusted graph has no negative directed cycle, so the classical feasibility theorem for difference constraints supplies a potential satisfying
\[
    \psi(v)-\psi(u)
    \le
    w(u,v)-\eta
\]
on every edge $u\to v$.  These inequalities are strict for the original weights, and hence make $A(S)$ the unique weighted-SRM minimizer for every $S\in\Scal$.  Adding a common constant makes $\psi$ nonnegative.
\end{proof}

\begin{remark}
One may weaken \Cref{thm:weighted-integrability} by requiring only that the learner's prescribed output $A(S)$ belong to the weighted-SRM argmin, rather than be its unique member.  The strict difference constraints then become weak, and the corresponding characterization requires every directed cycle to have nonnegative total weight.  This \emph{weak representation} is insufficient under our universal-selector convention: the argmin may contain additional hypotheses with different predictions, so the objective itself need not define a successful learner.  One could impose a separate, fixed rule for resolving ties, but this introduces an additional choice mechanism not encoded by the weighted objective, and we do not pursue that variant here.  Notice also that injectivity of $\psi$ does not preclude objective ties, since an empirical-loss difference may exactly cancel a regularizer gap.
\end{remark}

\section{Conclusion}\label{Section:conclusion}

Our results rule out several of the most natural algorithmic principles for multiclass learning. At the broadest level, learning does not reduce to proper learning: some learnable classes cannot be embedded into any properly learnable envelope.  Even when proper learning is possible, interpolation may fail, and the number of training errors required by a successful proper learner can occupy any prescribed sublinear scale. Furthermore, allowing a fixed notion of hypothesis complexity fails to repair the situation in general --- weighted SRM can fail on properly learnable classes, and local regularization can fail for improperly learnable classes.
On the positive side, we demonstrate that regularization succeeds whenever non-trivial version spaces are sufficiently simple, and at the learner level we provide exact characterizations of SRM using revealed preferences. 
Of course, our results leave open the possibility of some other algorithmic principle for multiclass learning that we have yet to consider. 
Identifying such principles---tractable learning rules beyond properness and regularization---is perhaps the most compelling direction left open by this work.

\subsection*{Acknowledgments}
\noindent 
Julian Asilis was supported by the National Science Foundation Graduate Research Fellowship under Grant No.\ DGE-1842487, and by a National Science Foundation award CCF-2239265.
Shaddin Dughmi was supported by NSF Grant CCF-2432219, and completed part of this work while on sabbatical as the Carter and
Tania Neild visiting professor at Northwestern University, as well as a visiting professor in the Data Science Institute at the University of Chicago. Vatsal Sharan was supported by National Science Foundation award CCF-2239265, an
Amazon Research Award, a Google Research Scholar Award, and an Okawa Foundation Research Grant. 
Shang-Hua Teng was supported in part by NSF Grant CCF-2308744. Alec Sun was supported by the National Science Foundation Graduate Research Fellowship.

\subsection*{AI Disclosure}

\noindent ChatGPT 5.6 Pro played a significant role in the development of all constructions in the paper. In particular, it first designed the poisoned first Cantor class of \cref{thm:no-erm} after the fundamental ideas of the construction were described by the (human) authors, i.e., augmenting the classic first Cantor class with a small collection of nearly constant $\star$ hypotheses.
Subsequent interactions used this construction to uncover the tight empirical-error scale of \cref{thm:empirical-upper,thm:empirical-lower}, and then sought a different obstruction capable of defeating weighted SRM. After several rounds of investigating possible constructions, with little substantive feedback, GPT 5.6 Pro discovered the projective-plane incidence class of \cref{thm:weighted-srm-fails}. In the same session, it was then instructed to attempt to solve the local regularization open problem of \citet{asilis-open-problem}, and derived the tripartite construction. The authors then discovered that even in a fresh conversation, GPT Pro was able to resolve the open problem. It also
assisted in finding and formalizing the $2^d$-ary tree construction used in the proof of \cref{thm:no-proper-completion}, and in the positive results of \cref{Section:Structural-Conditions-and-Integrability}. Finally, the model was used to draft and edit portions of the manuscript. The authors take full responsibility for all content in the paper.

\noindent  

\clearpage
\bibliographystyle{plainnat}
\bibliography{refs.bib}

\clearpage
\appendix
\crefalias{section}{appendix}

\section{Auxiliary Probabilistic and Structural Lemmas}
\label{app:tools}

We collect several probabilistic and structural tools used by the constructions
in the main text.  We conclude with a weighted extension of the localization
principle developed in \Cref{subsec:fallback-cores}.

\subsection{Selecting a light marker}

The proper learner for the poisoned first-Cantor class chooses a marker of minimum empirical frequency.  The following lemma shows that this simple rule is distribution-free: among sufficiently many distinguished atoms, an empirically least frequent one cannot carry substantial population mass.

\begin{lemma}[Least-observed marker]
\label{lem:least-observed-marker}
Let $p_1,\ldots,p_N$ be the masses of $N$ distinguished atoms under a probability distribution, and let $\widehat p_i$ denote their empirical frequencies in an iid sample of size $m$.  Choose $\widehat i\in\argmin_{i\in[N]}\widehat p_i$.
There is a universal constant $C>0$ such that, for every $\eta,\delta\in(0,1)$, if
\[
    N\ge \frac{8}{\eta}
    \qquad\text{and}\qquad
    m\eta\ge C\log\frac{1}{\eta\delta},
\]
then $\P[p_{\widehat i}>\eta]\le\delta$.
\end{lemma}

\begin{proof}
Choose $i_0$ with $p_{i_0}\le 1/N\le\eta/8$.  A Chernoff bound gives
$
    \P[\widehat p_{i_0}>\eta/2]
    \le e^{-c m\eta}.
$
On the other hand, for every $i$ with $p_i>\eta$,
$
    \P[\widehat p_i\le\eta/2]
    \le e^{-c m\eta}.
$
There are at most $1/\eta$ such heavy atoms.  Except on an event of probability at most $(1+1/\eta)e^{-c m\eta}$, the empirical minimum is therefore attained only by atoms of mass at most $\eta$.  Increasing the universal constant $C$ proves the claim.
\end{proof}

\subsection{Hidden-halfset lemma}

The lower bounds for the poisoned constructions use the same posterior-symmetry principle: after observing only $o(K)$ samples from a uniformly random $K$-subset of a set of size roughly $2K$, no procedure can cover much more than half of the hidden set.

\begin{lemma}[Hidden halfset]
\label{lem:hidden-halfset}
Let $K\to\infty$, let $N=o(K)$, and let $C$ be a set of size $2K+N$.  Choose $T\subseteq C$ uniformly among the $K$-subsets.  A possibly randomized procedure observes at most $m=o(K)$ iid draws from the uniform distribution on $T$ and outputs a set $B\subseteq C$ with $|B|\le K+N$.  Then
\[
    \P\brac{|T\setminus B|\ge K/5}=1-o(1),
\]
uniformly over the procedure.
\end{lemma}

\begin{proof}
Condition on the complete observed sequence, its set of distinct values $R\subseteq T$, and the procedure's internal randomness.  Every $K$-subset of $C$ containing $R$ assigns the same likelihood to the observed sequence.  Hence the posterior law of $T$ is uniform over these sets.  Write $T=R\cup U$, where $U$ is a uniformly random $(K-|R|)$-subset of $C\setminus R$.

Conditional on $R$ and $B$, the random variable $|U\cap(B\setminus R)|$ is hypergeometric.  Since $|R|\le m=o(K)$ and $|B|\le K+N=(1+o(1))K$, its mean is at most
\[
    (K-|R|)\frac{K+N}{2K+N-|R|}
    =
    \left(\frac12+o(1)\right)K.
\]
A hypergeometric Chernoff bound therefore gives
\[
    \P[|T\cap B|>4K/5\mid R,B]
    \le e^{-cK}
\]
for all sufficiently large $K$, uniformly in $R$ and $B$.  Since $|T\setminus B|=K-|T\cap B|$, the conclusion follows.
\end{proof}

\subsection{Projective-plane incidence estimates}

\begin{lemma}[Projective-plane mixing]
\label{lem:projective-plane-mixing}
Let $M$ be the point-line incidence matrix of a projective plane of order $q$, and write $N=q^2+q+1$ and $d=q+1$.  Then
\[
    MM^\top=qI+J.
\]
Consequently, the singular values of $M$ are $d$ and $\sqrt q$, and for every set $A$ of points and set $B$ of lines,
\[
    \left|
        e(A,B)-\frac{d}{N}|A||B|
    \right|
    \le
    \sqrt{q|A||B|}.
\]
\end{lemma}

\begin{proof}
Every row of $M$ contains $q+1$ ones, while two distinct rows have inner product one.  This gives $MM^\top=qI+J$.  The all-ones vector has eigenvalue $(q+1)^2$, and every vector orthogonal to it has eigenvalue $q$.  The asserted singular values follow.  Decomposing the indicator vectors of $A$ and $B$ into their all-ones and orthogonal components then yields the displayed mixing estimate.
\end{proof}

This is the estimate used in the incidence-inversion step \eqref{eq:incidence-inversion} of the proof of \Cref{thm:weighted-srm-fails}.

\subsection{Approximate interpolation for weighted SRM}
\label{app:further-srm}

Fallback localization extends naturally to weighted learners that make a controlled number of training errors.  The following theorem records the corresponding comparator argument.

\begin{theorem}[Approximate-interpolation weighted SRM]
\label{thm:weighted-approx}
Let $\psi:\Hc\to\R_{\ge0}$, let $\lambda_m>0$, and suppose that the weighted objective attains its minimum on every realizable $m$-sample.  Assume that for every such sample $S$ there exists $g_S\in\Hc$ satisfying
\[
    \emp_S(g_S)\le\eta_m,
    \qquad
    \psi(g_S)\le r_m.
\]
Set
$
    R_m:=r_m+\frac{\eta_m}{\lambda_m}.
$
If
\[
    \eta_m+\lambda_m r_m\to0
    \qquad\text{and}\qquad
    \Gdim(\Hc_{R_m})\log(m+1)=o(m),
\]
where $\Hc_R:=\{h \in \Hc:\psi(h)\le R\}$, then every minimizer of
$h\mapsto\emp_S(h)+\lambda_m\psi(h)$ PAC learns $\Hc$.
\end{theorem}

\begin{proof}
Let $\widehat h$ be any minimizer.  Comparison with $g_S$ gives
\[
    \emp_S(\widehat h)+\lambda_m\psi(\widehat h)
    \le
    \eta_m+\lambda_m r_m.
\]
Since both terms on the left are nonnegative,
\[
    \emp_S(\widehat h)
    \le
    \eta_m+\lambda_m r_m
    \qquad\text{and}\qquad
    \psi(\widehat h)
    \le
    r_m+\eta_m/\lambda_m
    =R_m.
\]
Thus every minimizer lies in the fixed sublevel class $\Hc_{R_m}$ and has vanishing empirical error.  Uniform convergence on that sublevel yields
\[
    \sup_{h\in\Hc_{R_m}}
    |\pop_D(h,h^\star)-\emp_S(h)|
    =o_{\P}(1)
\]
uniformly over realizable pairs.  Hence
$\pop_D(\widehat h,h^\star)\le\emp_S(\widehat h)+o_{\P}(1)\to0$.
The argument controls all minimizers simultaneously and is therefore robust to ties.
\end{proof}

\subsection{Witnesses for the containment diagram}
\label{app:figure-witnesses}
The regions in \cref{fig:multiclass-containments} concern nonempty countable
classes in realizable, distribution-free PAC learning. Here $\mathsf I$ allows
improper predictors, $\mathsf B$ denotes proper embeddability, $\mathsf P$
proper learnability, $\mathsf C$ interpolating proper learnability, $\mathsf W$
weighted SRM, $\mathsf E$ success of every proper ERM, and $\mathsf L$ and
$\mathsf{WL}$ hard and weighted local regularization. Weighted local
regularization uses one coefficient $\lambda(S)\ge0$ across test points and a
fixed score $\psi(h,x)$. Every regularization claim requires attained minima
and PAC success for all induced selectors.

\begin{lemma}\label{lem:figure-inclusions}
For nonempty countable classes, the following containments hold:
\[
  \mathsf E\subseteq\mathsf C\cap\mathsf W\cap\mathsf L,\qquad
  \mathsf C\cup\mathsf W\subseteq\mathsf P\subseteq\mathsf B
  \subseteq\mathsf I,\qquad
  \mathsf L\subseteq\mathsf{WL}\subseteq\mathsf I.
\]
Only the displayed inclusion $\mathsf E\subseteq\mathsf L$ uses countability.
\end{lemma}
\vspace{-0.4 cm}
\begin{proof}
The properness and embeddability implications follow from the definitions.
An ERM is interpolating on a realizable sample, and $\lambda\equiv0$ makes
every weighted-SRM selector an ERM.

For $\mathsf E\subseteq\mathsf L$, enumerate the hypotheses as
$h_1,h_2,\ldots$ and give them strictly increasing bounded scores, independent
of the test point. The first consistent hypothesis is the unique local
minimum and is a proper ERM. It learns by the definition of $\mathsf E$.

For $\mathsf L\subseteq\mathsf{WL}$, replace a successful local score $\psi$
by $\phi=\psi/(1+\psi)\in[0,1)$; this preserves its pointwise order. On a
realizable $m$-sample, set $\lambda(S)=1/(2m)$. Every consistent hypothesis
has weighted objective below $1/(2m)$, whereas every inconsistent hypothesis
has objective at least $1/m$. Thus the weighted and hard local argmin sets
coincide at every test point. On nonrealizable samples, $\lambda=0$ ensures
attainment of the empirical minimum.
\end{proof}

The first three additional placements use constructions already defined in
the main text.

\begin{lemma}\label{lem:figure-existing-local}
The no-envelope class of \cref{thm:no-proper-completion} belongs to
$\mathsf L\setminus\mathsf B$. The incidence class $\Hinc$ belongs to
$(\mathsf C\cap\mathsf L)\setminus\mathsf W$, and the poisoned first-Cantor
class $\Hpoison$ belongs to $\mathsf L\setminus\mathsf C$.
\end{lemma}
\vspace{-0.4 cm}
\begin{proof}
The cited theorems give the three negative claims. For the tree, set
$\psi(h,x)=\one\{h(x)\notin\{0,1\}\}$. A private label identifies the target.
Otherwise let $i^\star$ be the deepest sampled public depth. Every local
minimizer is correct on the target's public prefix $P_{i^\star}$: at an
earlier target-public point, the target supplies a public prediction, and
every consistent hypothesis predicting publicly there follows the branch
fixed by the deeper sample. The prefix-mass argument in the proof of
\cref{thm:no-proper-completion} therefore gives, for every selector,
\[
 \P\!\left[\pop_D(A(S),h^\star)>\epsilon\right]
 \le e^{-m\epsilon}.
\]

For $\Hinc$ and $\Hpoison$, use the same pointwise score: zero for the
background label $\dollarlabel$, one for $\starlabel$, and two for every
identifying label. An identifying observation fixes the target. If every
label is $\dollarlabel$, the all-background hypothesis is available and
every local minimum predicts $\dollarlabel$. Otherwise the observed
$\starlabel$ identifies an active component. At any point where the target
predicts $\starlabel$, the target itself is a consistent competitor of score
one, so every local minimum makes the correct prediction. Outside that
component, every consistent hypothesis predicts $\dollarlabel$. Error is
therefore confined to an identifying-label region missed by the sample, or
to an entirely missed active component. Each event has probability at most
$e^{-m\epsilon}$ when its mass exceeds $\epsilon$. Finally, the proper
learner constructed for $\Hinc$ in \cref{thm:weighted-srm-fails} interpolates:
its one-block owner and its uniquely identified outputs are consistent.
\end{proof}

\begin{lemma}\label{lem:figure-poison-weighted}
The poisoned class also belongs to $\mathsf W$; hence
$\Hpoison\in(\mathsf W\cap\mathsf L)\setminus\mathsf C$.
\end{lemma}
\vspace{-0.4 cm}
\begin{proof}
Give $h_{\dollarlabel}$ and every poisoned fallback score zero, and set
$\psi(c_{n,A})=n$. If a private or poison label appears, it identifies
$h^\star$; choose
$\lambda(S)=[2m(1+\psi(h^\star))]^{-1}$. The target's objective is below
$1/(2m)$, whereas every competitor has empirical error at least $1/m$.

On all other realizable samples, take $\lambda(S)=m^{-1/2}$ and
$K=m^{1/4}$. An all-$\dollarlabel$ sample has a zero-score background
comparator. Otherwise the active block $n$ supplies a consistent target of
objective at most $K/\sqrt m=m^{-1/4}$ if $n\le K$. If $n>K$, one of its
$n$ markers occurs at most $m/n$ times; the corresponding fallback has
objective at most $1/n<m^{-1/4}$. Consequently every minimizer
$\widehat h$ satisfies
\[
 \emp_S(\widehat h)\le m^{-1/4},
 \qquad
 \widehat h\in
 F\cup\{c_{n,A}:n\le m^{1/4}\},
\]
where $F$ is the background hypothesis together with all poisoned
fallbacks. On any $t$ points, $F$ has at most $2t+1$ output patterns,
while the displayed Cantor prefix has logarithmic cardinality
$O(m^{3/4})$. The standard growth-function uniform-convergence bound
therefore gives
\[
 \sup_{h\in F\cup\{c_{n,A}:n\le m^{1/4}\}}
 |\pop_D(h,h^\star)-\emp_S(h)|=o_{\P}(1)
\]
uniformly over realizable target-distribution pairs. This proves that
every minimizer learns. Minima attain: fallback empirical errors take
only $m+1$ values, and positive $\lambda$ leaves only finitely many Cantor
blocks below any fixed objective threshold. On nonrealizable samples set
$\lambda=0$.
\end{proof}

Let $X=\bigsqcup_{n\ge1}X_n$, with $|X_n|=2n$. Include
$h_{\dollarlabel}$ and, for every half-set $A\subseteq X_n$, the hypothesis
$c_{n,A}$ that predicts $\starlabel$ on $A$, its private label on
$X_n\setminus A$, and $\dollarlabel$ elsewhere. Call this class
$\mathcal C^-$. Let $\mathcal C^+$ additionally contain $f_n$, which
predicts $\starlabel$ throughout $X_n$ and $\dollarlabel$ elsewhere.

\begin{lemma}\label{lem:figure-cantor}
$\mathcal C^-\in(\mathsf B\cap\mathsf L)\setminus\mathsf P$ and
$\mathcal C^+\in(\mathsf C\cap\mathsf W\cap\mathsf L)\setminus\mathsf E$.
\end{lemma}
\vspace{-0.4 cm}
\begin{proof}
To rule out proper learning of $\mathcal C^-$, choose an $n$-point set
$A\subseteq X_n$ uniformly and draw $m=o(n)$ examples uniformly from
$A$, all labeled $\starlabel$. The background hypothesis and every
hypothesis supported on another block have error one.
Conditional on the observed support $R$, the unobserved part of $A$
is uniform among the $(n-|R|)$-subsets of $X_n\setminus R$. For any
proper output $c_{n,B}$, the posterior overlap of $B$ with $A$ is
approximately one half; hypergeometric concentration gives
$|A\setminus B|\ge n/5$ with probability tending to one. Thus every
proper learner has constant population error on a suitable fixed
target for arbitrarily large $m$.
The pointwise score in \cref{lem:figure-existing-local} learns
$\mathcal C^-$, and $\mathcal C^+$ is its properly learnable envelope.

For $\mathcal C^+$, give the background score zero, all $f_n$ score one,
and all $c_{n,A}$ score two. An all-background sample selects
$h_{\dollarlabel}$; an active $\starlabel$ sample selects its $f_n$;
and a private label identifies its target. These choices are consistent
and PAC by the missed-mass bound. They are unique hard minima, so the
same score learns locally. With $\lambda(S)=1/(3m)$, every weighted
minimum is consistent and coincides with a hard minimum, proving
weighted-SRM membership.

To see that not every ERM learns, take target $f_n$ and the uniform
distribution on $X_n$, with $n\ge m$. Some half-set contains all distinct
sampled points. A proper ERM may select its $c_{n,A}$, which has
population error $1/2$. Choosing such an ERM on these samples gives
failure at arbitrarily large $m$.
\end{proof}

The next construction combines the tripartite obstruction with
noninterpolating, nearly constant fallbacks.

\begin{theorem}\label{thm:figure-tagged}
There is a countable, closed, pointwise finite-range class
$\mathcal H_{\mathrm{tag}}$ such that
\[
 \mathcal H_{\mathrm{tag}}
 \in\mathsf W\cap\mathsf{WL}\cap\mathsf P
 \quad\text{but}\quad
 \mathcal H_{\mathrm{tag}}\notin\mathsf C\cup\mathsf L.
\]
\end{theorem}
\begin{proof}
For $j\ge2$, let $G_j=K_{N_j,N_j,N_j}$ with $N_j=2^{j-1}$ and edge set
$E_j$. Set
\[
 n_j=\max\{4,\lfloor\sqrt j\rfloor\},
 \qquad
 X_j=\{0,1\}^{E_j}\times[n_j],
 \qquad X=\bigsqcup_{j\ge2}X_j.
\]
The graph vertices serve as labels, with different components using
disjoint vertex sets. On its component, an edge hypothesis is
$h_e(z,t)=\operatorname{head}_z(e)$, independent of the tag $t$.
For each vertex $a$ and tag $i$, add a fallback
\[
 s_{j,a,i}(z,t)=
 \begin{cases}
  \rho_{j,a,i},&t=i,\\
  a,&t\ne i,
 \end{cases}
\]
where every $\rho_{j,a,i}$ is a fresh poison label. Extend each
hypothesis by $\dollarlabel$ outside its component and include
$h_{\dollarlabel}$. Each component has finitely many points and
hypotheses. The finite-projection argument used for $\Hinc$ proves
closedness; countability and pointwise finite range are immediate.

Choose injective global scores satisfying
\[
 \psi(h_{\dollarlabel})=0,\qquad
 1\le\psi(s_{j,a,i})\le1.1,\qquad
 2\le\psi(h_e)\le2.1.
\]
For a fixed sample, only countably many positive coefficients create an
objective tie, so every nonempty coefficient interval contains a choice
with a unique minimizer. Let $r=|S|$ and $J_r=\lfloor\sqrt r\rfloor$.
Use $\lambda=2$ on an all-$\dollarlabel$ sample, which uniquely selects
the background. If a poison label or two vertex labels occur, they
identify the target; choose a generic
$0<\lambda<1/(4r)$. The target's objective is below $1/r$ and every
competitor's empirical loss is at least $1/r$. Use the same small
coefficient when the active component $j\le J_r$; then every minimum
interpolates within a finite component.

It remains to consider $j>J_r$ with all $k\ge1$ active labels equal to
a vertex $a$. Choose a tie-free coefficient in
\[
 \frac{2k}{rn_j}<\lambda<\frac{2.1k}{rn_j}.
\]
Some tag occurs at most $k/n_j$ times. Its $a$-fallback has objective
below $3.31k/(rn_j)$; every edge hypothesis has objective above
$4k/(rn_j)$, and the background and every other fallback have empirical
loss $k/r$. Since $n_j\ge4$, the unique minimizer is an active
$a$-fallback whose empirical tag frequency is below
$3.31k/(rn_j)\le3.31/n_j$. These comparisons also give attainment on
every realizable sample. On a nonrealizable sample, take $\lambda=2$;
the background score is at most one and every other objective exceeds
two.

The finite prefix $j\le J_r$ has logarithmic cardinality $O(\sqrt r)$,
so its consistent selections generalize. In a tail component let
$p_i=D(X_j\cap\{t=i\})$. For fixed $\epsilon>0$ and sufficiently large
$r$, one has $n_j\ge8/\epsilon$ throughout the tail, so the selected
tag frequency is below $\epsilon/2$. At most $1/\epsilon$ tags have
$p_i>\epsilon$, and a Chernoff bound gives
\[
 \P[p_{\widehat i}>\epsilon]
 \le\epsilon^{-1}e^{-r\epsilon/8}.
\]
If an edge target produces an all-$a$ sample, the region on which it
predicts its other endpoint has been missed; if a fallback target
produces such a sample, its poison slice has been missed. Either region
has probability at most $e^{-r\epsilon}$ of being missed when its mass
exceeds $\epsilon$. An entirely missed active component obeys the same
bound. These estimates prove proper PAC learning for the unique
weighted-SRM selector. As the global minimum is unique, the
test-independent local score $\psi(h,x)=\psi(h)$ gives the same
predictions under weighted local regularization.

For the two lower bounds, take sample size $r=j$ and fix a vertex $a$.
Choose the target edge $e$ uniformly among the $2^j$ edges incident
to $a$, and draw from the distribution $D_{a,e}$ that is uniform over
orientations with $\operatorname{head}_z(e)=a$ and over tags. Every
training label is $a$. The event $F_j$ that every tag is observed has
\[
 \P(F_j)\ge1-n_j e^{-j/n_j}\longrightarrow1.
\]
On $F_j$, every fallback is inconsistent. The remaining version space
is exactly the original tripartite edge version space. Tags and edge
orientations are independent. Each non-target incident edge survives
with probability $2^{-j}$, independently, so the number of false
survivors is
$K\sim\operatorname{Bin}(2^j-1,2^{-j})$. Conditional on the sample,
the target is uniform among these $K+1$ edges. On $K\ge1$, every
interpolating proper learner selects a wrong edge with probability at
least $1/2$; a wrong edge has population error $1/2$. Thus its
expected loss is at least
\[
 \frac14\,\P(F_j)\P(K\ge1),
\]
bounded away from zero. This also covers randomized learners.

Finally fix any hard local score. Conditional on $F_j$, the training
orientation bits retain their independent law. At a test point $(z,t)$,
the local score induces an arbitrary order on the surviving edges,
possibly depending on $t$. The directed-triangle proof of
\cref{thm:local-regularization-fails} applies separately at each tag:
its averaged expected loss is at least $1/100$ conditional on $F_j$.
The unconditional lower bound $\P(F_j)/100$ excludes every hard local
regularizer.
\end{proof}

The next construction separates weighted local regularization from general
learnability. Its uniform hypergraph edges intersect sparsely, but each
vertex lies in many edges.

\begin{theorem}\label{thm:figure-hypergraph}
There is a countable, closed, pointwise finite-range class
$\mathcal H_{\mathrm{hg}}$ with sample complexity
$O(\epsilon^{-1}\log(1/\delta))$ such that
\[
 \mathcal H_{\mathrm{hg}}
 \in\mathsf B\setminus(\mathsf P\cup\mathsf{WL}).
\]
The nonlearnability claim still holds if the weighted-local coefficient
may depend on both the labeled sample and the test point.
\end{theorem}
\begin{proof}
For each sufficiently large integer $m$, set
\[
 k=m^6,\qquad D=k^{2m},\qquad M=D^4k^8,\qquad
 \epsilon=\beta=\frac1{1000m},\qquad
 \ell=\left\lfloor\frac{\epsilon k}{4}\right\rfloor.
\]
There exists a $k$-uniform hypergraph $G_m=(V_m,E_m)$ with $|V_m|=M$.
It is linear, meaning distinct edges intersect in at most one vertex;
it satisfies $|E_m|\ge MD/(2k)$ and has maximum vertex degree $2D$;
and for every $R\subseteq V_m$ with $|R|\le(1-\epsilon/2)M$,
 \[
  \#\{e\in E_m:|e\setminus R|<\ell\}\le\beta M.
  \tag{Expansion}
 \]
Here and below all parameters are component-dependent; their $m$-subscripts
are suppressed in finite-component calculations.

\par\smallskip\noindent\textbf{Constructing the hypergraph.}\enspace
Draw $L=MD/k$ independent uniformly random $k$-subsets of an $M$-element vertex
set.  Each vertex has expected degree $D$, so with probability tending to one
all degrees are at most $2D$, by a Chernoff bound and
$M\exp(-D/3)=o(1)$.

For a fixed $R$ of the specified size, a random edge has expected number of
vertices outside $R$ at least $\epsilon k/2$.  A hypergeometric lower-tail
bound therefore gives
\[
 \Pr[|e\setminus R|<\ell]\le\exp(-\epsilon k/16).
\]
A binomial upper-tail bound, followed by a union bound over at most $2^M$
sets $R$, shows that the probability that Expansion fails is at most
\[
 \exp\left\{
 M\log2+\beta M\log\frac{eD}{\beta k}
       -\frac{\beta\epsilon k M}{16}
 \right\}=o(1).
\]
Indeed, $\beta\epsilon k/16=m^4/(16\cdot10^6)$ whereas
$\beta\log(eD/(\beta k))=O(\log m)$.

The expected number of pairs of sampled edges intersecting in at least two
vertices is
\[
 O(L^2k^4/M^2)=O(D^2k^2).
\]
With positive probability it is at most
a constant times $D^2k^2$, simultaneously with the two preceding events.
Delete one edge from every intersecting pair until the hypergraph is linear.
This removes $O(D^2k^2)=o(L)$ edges, preserves Expansion and the upper degree
bound, and leaves at least $L/2=MD/(2k)$ edges. This gives the required
hypergraph.

\par\smallskip\noindent\textbf{The class and its envelope.}\enspace The vertex
label sets $V_m$ of different components are disjoint. Let
\[
 X_m=\prod_{e\in E_m}e
\]
be the set of all orientations of the hypergraph: $x\in X_m$ chooses a head
$x(e)\in e$ for every hyperedge.  Set $X=\bigsqcup_mX_m$, and introduce a
background label $\$$.  For every edge $e\in E_m$, include the hypothesis
\[
 h_e(x)=\begin{cases}x(e),&x\in X_m,\\\$,&x\notin X_m.\end{cases}
\]
Include also $h_\$\equiv\$$.  The resulting class $\mathcal H_{\mathrm{hg}}$ is countable,
closed, and pointwise finite.  Each component domain and its active hypothesis
family are finite, and every hypothesis from another component is $\$$ on
that domain.  The finite-projection closedness proof is the same as for the
incidence class in \cref{def:incidence-class}.

An improper learner proceeds as follows.  If all observed labels are $\$$,
output $h_\$$.  Otherwise the active component is identified.  If two distinct
non-$\$$ vertex labels appear, linearity identifies their unique common edge,
and the learner outputs the target $h_e$.  If only one active label $a$ appears,
output $f_{m,a}$, equal to $a$ on $X_m$ and $\$$ elsewhere. Write $p_a$
for the conditional probability of label $a$ in the active component.
This is uniformly PAC: conditional on $t\ge1$ active observations, the
probability of emitting a
constant with conditional component error exceeding $\alpha$ is bounded by
\[
 \sum_{a:p_a<1-\alpha}p_a^t\le(1-\alpha)^{t-1}.
\]
For completeness, let the active component have marginal mass $s$. If
$s\le\alpha$, the global error is at most $\alpha$. Otherwise let
$T\sim\operatorname{Bin}(r,s)$ and set $\eta=\alpha/(2s)\le1/2$.
Conditional on $T=t\ge1$, global error greater than $\alpha$ implies
conditional component error greater than $\eta$, whose probability is at
most $(1-\eta)^{t-1}$. Including the missed-component event yields
\[
 \Pr[L_D(A(S),h^\star)>\alpha]
 \le(1-s)^r+\frac{(1-s\eta)^r}{1-\eta}
 \le3e^{-r\alpha/2}.
\]
This bound is independent of the number of hyperedge vertices and of
component size. The background target is learned exactly.

Furthermore $\mathcal H_{\mathrm{hg}}$ is properly embeddable: adjoin all the $f_{m,a}$.
Assign hard penalties 0 to $h_\$$, 1 to each $f_{m,a}$, and 2 to the edge
hypotheses. An all-background sample has unique minimum $h_\$$. A sample
with exactly one active vertex label $a$ has unique minimum $f_{m,a}$, and
two distinct active labels identify a unique edge. The resulting unique
hard-SRM selector is precisely the learner above, now proper in the
completion. Targets $f_{m,a}$ are handled by the same missed-component bound.
Thus the completion has a consistent proper PAC learner.

Fix one component, an arbitrary test orientation $x$, and any strict total
order $\prec$ on its hyperedges.  This order may depend arbitrarily on $x$.
Take $m$ independent uniformly random training orientations
$U=(u_1,\ldots,u_m)$, independent of $x$, and write
\[
 C_a(f)=\#\{i\in[m]:u_i(f)=a\}.
\]
For $c\in\{1,\ldots,m\}$ let $F_c(a)$ be the $\prec$-first edge with
$C_a(f)=c$, or be undefined if there is none.  Necessarily such an edge is
incident to $a$.  Set
\[
 Z_c=\#\{a\in V_m:F_c(a)\text{ is defined and }x(F_c(a))=a\}.
\]
Then, uniformly over $x$, $\prec$, and $c\ge1$,
\[
 \mathbb E_U Z_c
 \le M B_m,\qquad
 B_m:=\frac1\ell+\beta+\epsilon+\frac{k}{M}.
 \tag{Coverage}
\]

\par\smallskip\noindent\textbf{Greedy coverage argument.}\enspace
Expose the training orientations one hyperedge at a time in the order
$\prec$.  Mark a vertex $a$ when its first edge with $C_a(f)=c$ is exposed.
Let $R$ be the already marked vertices.  Conditional on everything previously
exposed, the new edge's training heads are still independent uniform heads;
for every one of its vertices,
\[
 P_c:=\Pr[C_a(f)=c]
 =\binom mc k^{-c}(1-1/k)^{m-c}.
\]
Thus the conditional expected number of new marks is $P_c|f\setminus R|$,
whereas the expected number of newly correct marks is at most $P_c$, since
only the single vertex $x(f)$ can be correct.

Stop after $R$ first reaches size $(1-\epsilon)M$, if it ever does.  Before
stopping, call an edge good if $|f\setminus R|\ge\ell$.  The conditional expected
number of correct new marks on a good edge is at most $1/\ell$ times the
expected total new marks there.  Summing gives at most $M/\ell$ expected correct
marks from good edges.

Let $R_\mathrm{final}$ be the marked set at stopping, or after every edge if
stopping never occurs.  It has size at most $(1-\epsilon)M+k$, and for all
sufficiently large $m$ this is at most $(1-\epsilon/2)M$.  Every bad edge
processed before stopping satisfies
$|f\setminus R_\mathrm{final}|<\ell$.  Expansion therefore bounds the number
of such edges by $\beta M$, and they contribute at most $\beta M$ correct
marks.  After stopping at most $\epsilon M+k$ vertices remain available to
become correct first marks.  This proves Coverage.  The extra $k$ merely
absorbs harmless threshold rounding.

\par\smallskip\noindent\textbf{Hard distributions and coupling.}\enspace For
every incidence $a\in e$, let $D_{a,e}$ be the uniform distribution on
orientations satisfying $x(e)=a$, and take target $h_e$.  All training labels
are then $a$.  We average these hard pairs using the following equivalent
joint experiment:
\begin{enumerate}
 \item draw a uniform test orientation $x\in X_m$;
 \item choose a uniform target edge $e\in E_m$ and put $a=x(e)$;
 \item draw generic independent uniform training orientations $U$, then
 overwrite the head of edge $e$ in each one by $a$, obtaining $S$.
\end{enumerate}
Conditional on $(a,e)$, the test point and the overwritten training sample
are independent draws from $D_{a,e}$, exactly as required.  The mixture prior
is uniform over edge-vertex incidences.

Let $d_x^-(a)=\#\{e:x(e)=a\}$.  Since degrees are at most $2D$, a Chernoff bound
gives the event
\[
 Q_m:=\{\max_a d_x^-(a)\le4D/k\},\qquad
 \Pr_x[Q_m^c]\le\delta_m:=M e^{-2D/(3k)}=o(1).
\]
On $Q_m$, the conditional mixture weight of every vertex is at most
\[
 w_x(a)=\frac{d_x^-(a)}{|E_m|}\le\frac8M.
 \tag{Weight}
\]
Together with Coverage, this shows that a generic first-$c$ selector is
correct with probability at most $8B_m+\delta_m$, averaged over $x,e,U$.

Fix now the order $\prec$ at each $x$ and compare its first-$c$ selectors
before and after overwriting the target edge.

For $1\le c<m$, overwriting changes only the target edge's count,
setting it to $m$. The
first-$c$ selector is unchanged unless its generic target-edge count was $c$.
This has probability $P_c$.  Hence the overwritten first-$c$ selector is
correct with probability at most $8B_m+\delta_m+P_c$.  Moreover,
\[
 \sum_{c=1}^{m-1}P_c\le\Pr[\operatorname{Bin}(m,1/k)\ge1]\le m/k.
\]

For $c=m$, write $p=k^{-m}$. Forcing the target edge to become
$m$-eligible can change the first-$m$ selector only if no earlier
incident edge is $m$-eligible. Let
$r_a(e)$ be the number of edges incident to $a$ that precede $e$.  The
probability of this event is $(1-p)^{r_a(e)}$.  Averaged over the uniform target
edge and $a=x(e)$ it is at most
\begin{align*}
 \frac1{|E_m|}\sum_{e\in E_m}(1-p)^{r_{x(e)}(e)}
 &\le\frac1{|E_m|}\sum_{a\in V_m}
       \sum_{r=0}^{\deg(a)-1}(1-p)^r\\
 &\le\frac{M}{p|E_m|}
 \le\frac{2k}{Dp}=2k^{1-m}.
\end{align*}
Thus the overwritten first-$m$ selector is correct with probability at most
$8B_m+\delta_m+2k^{1-m}$.

For $c=0$, let $f_0$ be the first edge in the entire component order,
and let $Y(f_0)$ collect the labels appearing as its heads in the
generic training orientations. Unless the target edge is $f_0$,
overwriting leaves its counts
unchanged.  If $a\notin Y(f_0)$, then $f_0$ is still the first zero-count edge,
and it predicts $x(f_0)$; it can be correct only if $a=x(f_0)$.  Consequently,
the overwritten first-zero selector can be correct only if
\[
 a\in Y(f_0)\cup\{x(f_0)\},\qquad\text{or}\qquad e=f_0.
\]
By Weight, its correctness probability is at most
\[
 \frac{8(m+1)}M+\frac1{|E_m|}+\delta_m.
\]
Zero-count hypotheses include edges not incident to $a$; the global first
edge handles all of them at once.

\par\smallskip\noindent\textbf{The weighted-local lower bound.}\enspace
Fix any proposed local penalty $\psi(h,x)$ and any nonnegative coefficient
$\lambda(S,x)$.  The coefficient may depend on the entire sample and test
point. This is stronger than the shared coefficient $\lambda(S)$ in the
weighted-local rule of \cref{rem:weighted-local-open}.
Choose an induced learner that resolves weighted-objective ties first by
smaller $\psi(h,x)$ and then by a fixed enumeration of the countable class.
Use this rule on the hard all-$a$ samples, and on arbitrary other samples
whenever the coefficient is positive.  On other zero-coefficient samples,
use the enumeration-first objective minimizer instead; this extends the
hard-sample selector to one globally defined induced learner without requiring
a penalty minimum on an infinite version space.
Whenever the weighted argmin is nonempty this selector exists: at a fixed
sample the empirical loss takes only $m+1$ values, so among tied weighted
minimizers with positive coefficient there are only $m+1$ possible penalty
values.  At zero coefficient, the hard samples have a finite nonempty
version space inside the active component.

At a fixed test point $x$, let $\prec_x$ be the order on active component edges
given by penalty and enumeration.  For an all-$a$ sample, the empirical loss
of $h_f$ is $1-C_a(f)/m$.  If the selected hypothesis belongs to the active
component and has count $c$, it must be the $\prec_x$-first edge among
those with count $c$. Otherwise an earlier edge of the same count
would have a strictly smaller objective or win the prescribed tie-break.
This remains
true when $\lambda=0$, because then the secondary penalty tie-break chooses
the first edge with count $m$.

Hypotheses from other components and $h_\$$ never predict the desired vertex
label $a$ on the hard distributions.  Therefore, regardless of the chosen
coefficient, correct prediction implies that at least one of the $m+1$
first edges at those agreement counts predicts $a$.

Summing these bounds gives
\begin{align*}
 \Pr_{x,e,S}[A(S)(x)=a]
 &\le 8m B_m+\frac{m}{k}+2k^{1-m}
      +\frac{8(m+1)}M+\frac1{|E_m|}+(m+1)\delta_m\\
 & = 0.016+o(1).
\end{align*}
Here $m/\ell=O(m^2/k)=o(1)$ and
$m(\beta+\epsilon)=2/1000$.  In particular, the correctness probability is
less than $1/10$ for all sufficiently large $m$.

Thus the average expected population error over the hard target-distribution
pairs exceeds $9/10$.  Some fixed incidence $(a,e)$ consequently has expected
error exceeding $9/10$ after $m$ samples.  These failures occur for arbitrarily
large $m$, contradicting a uniform PAC guarantee.  If a proposed objective
ever lacks an attained argmin on a realizable sample, it already fails to
induce a learner.  Hence no weighted local objective learns $\mathcal H_{\mathrm{hg}}$.

It remains to rule out proper learning directly. Choose a vertex $a$ of
degree $d\ge D/2$, which exists by the edge-count bound. Draw the target
uniformly among its incident edges and sample from the corresponding
$D_{a,e}$. The target is uniform in the posterior over
$K=1+\operatorname{Bin}(d-1,p)$ surviving incident edges, where
$p=k^{-m}$. Every incorrect proper hypothesis has population error at
least $1-1/k$. Thus every proper learner has Bayes expected loss at least
\[
 (1-1/k)\left(1-\E\frac1K\right)
 =(1-1/k)\left(1-\frac{1-(1-p)^d}{dp}\right)
 \ge(1-1/k)\left(1-\frac1{dp}\right)
 \longrightarrow1.
\]
This also covers randomized proper learners and completes the theorem.
\end{proof}

For two classes on disjoint domains, write
$\mathcal H_1\oplus\mathcal H_2$ for the following marked sum. Relabel
their label sets disjointly, extend every hypothesis by a fresh outer
background label $\bot$ on the opposite domain, and add the
all-$\bot$ hypothesis. Each class's original background label remains
an ordinary label on its own component.

\begin{lemma}\label{lem:figure-unions}
Marked sums preserve PAC, proper PAC, interpolating proper PAC,
proper embeddability, and hard local learnability. The negative
results used in the table below persist under the indicated sums.
Moreover, $\mathcal H_{\mathrm{tag}}\oplus\mathcal H\in\mathsf{WL}$
for $\mathcal H\in\{\Hinc,\mathcal C^-,\mathcal H_{\mathrm{tree}}\}$.
\end{lemma}
\vspace{-0.4 cm}
\begin{proof}
For the first four positive properties, discard $\bot$-labeled
examples and run the appropriate learner on the active component,
defaulting to the all-$\bot$ hypothesis when no active example
appears. Binomial thinning converts a uniform component guarantee
into a uniform guarantee for the sum. The same construction on
properly learnable envelopes proves proper embeddability.

For hard local learning, transform each component's score into
$[0,1)$, give a hypothesis score $1+\psi_i(h,x)$ on its own domain
and zero on the opposite domain, and give the outer background score
zero. An active observation eliminates all opposite-component
hypotheses; on the active domain the local minima agree with the
component rule, while outside it every surviving hypothesis
predicts $\bot$. If all observations carry $\bot$, every pointwise
minimum predicts $\bot$.

For the three weighted-local sums, use a bounded hard-local score
in $[1,2)$ on the untagged component. On the tagged component give
its internal background score one, its fallbacks distinct scores in
$[2,2.1]$, and its edge hypotheses distinct scores in $[3,3.1]$.
Every component hypothesis has score zero outside its own domain;
the outer background has score zero everywhere. Let $r$ be the
sample size and $p=k/r>0$ the fraction of active observations.
If the hard-local component is active, take
$0<\lambda<1/(3r)$; every consistent own-component objective is
below $1/r$, while inconsistent hypotheses cost at least $1/r$.
This reproduces the hard-local predictions on that component and
predicts $\bot$ elsewhere.

If the tagged component is active and its target is identified, or
its index is at most $\lfloor\sqrt r\rfloor$, take a generic
$0<\lambda<1/(5r)$. The tagged minimizer is then a unique
interpolator in the active component. Otherwise the active sample
has one vertex label $a$ in a tail component with $n$ tags; choose
a tie-free coefficient in $(2p/n,2.1p/n)$. A least-frequency
$a$-fallback has objective at most $5.41p/n$, while every edge
costs more than $6p/n$ and every opposite-component hypothesis or
outer background costs at least $p$. For sufficiently large $r$,
all tail components have $n\ge6$. Hence every minimum on the active
domain is an $a$-fallback with empirical poison-slice mass at most
$5.41/n$. The heavy-slice Chernoff argument in
\cref{thm:figure-tagged} proves generalization. At a test point
outside the active domain, every active-component hypothesis
predicts $\bot$ and has score zero, so every local minimum predicts
correctly there. An all-$\bot$ sample similarly predicts $\bot$
at every test point. The tagged objective has attained minima
because its active components are finite; the bounded hard-local
objectives attain by assumption.

For the negative properties, concentrate the hard marginal on
the offending component. Opposite-component hypotheses predict
$\bot$ there, have population error one, and are inconsistent with
any non-$\bot$ observation. In particular, they cannot repair the
poisoned-Cantor interpolation obstruction or the tagged hard-local
obstruction. The incidence weighted-SRM inversion still gives a
bad minimizer: its owner remains uniquely population-perfect, and
all added hypotheses are population-bad. The hypergraph
weighted-local proof bounds all active-edge predictions of the
correct vertex label and permits arbitrary additional competitors
that predict only $\bot$. Finally, restriction of any properly
learnable envelope of a sum containing the tree to the tree domain
would contradict \cref{thm:no-proper-completion}.
\end{proof}

\subsubsection{Four geometric cells}\label{app:figure-geometric}

The first construction separates consistent proper learning from
weighted local regularization. We give the finite design and the two
estimates that control its learners. Throughout, a local score may
use the test point, and the lower bound even allows its coefficient
to depend on that point.

\begin{lemma}[Finite owner--alias design]\label{lem:fig-geometry}
Take $C=m^{30}$ and $q=2^{\lceil10m\log_2 C\rceil}$.
For all sufficiently large $m$, the construction below gives
unique owners for public-label/affine-line pairs and at most $200$
aliases per label--plane pair. The graph joining rows that share
at least two public labels has maximum degree one, and every label--line
pair has at least $q^2/2$ compatible nonparallel aliases. Its light
and heavy role thinnings satisfy the uniform expansion bounds below.
\end{lemma}
\vspace{-0.4 cm}
\begin{proof}
Set $M=q^{100}$, $Z=\mathbb F_q^5$,
$s_1=(q^5-1)/(q-1)$, $N=Ms_1$, and
$H=s_1\lceil Mq/C\rceil$. A row's owned lines will be
parallel, disjoint, and distinctly labeled, with load at most
$4Cq^3$; a nonparallel test line will meet their union in at
most $16C$ points. Each row has $C$ complete parallel classes
of alias planes with distinct labels. At an unowned point an
independent head coordinate chooses one of the $C$ eligible
aliases.

For the expansion bounds, set $\gamma=\beta=m^{-4}$.
At query geometry $z$, let $T$ be the retained role set and
$n=|T|$. Uniformly for $R\subseteq T$ with
$|R|\le(1-\gamma)n$, fewer than $\beta n$ rows have fewer
than $\ell$ alias neighbors in $T\setminus R$. For light
retention $\theta=1/q$, take
$B=\lceil1000\log q\rceil$ and $\ell=\gamma C/4$;
for heavy retention $\theta=C^{2m}/q$, take
$B=2C^{2m}$ and $\ell=\gamma C^{2m+1}/4$.
The pencil cap $B$ bounds retained roles in one alias pencil.
Each thinning fails with probability at most $q^{-20}$.

For each line direction, color the label--line bipartite graph
sequentially with $\lceil Mq/C\rceil$ colors, uniformly among those
not used at either endpoint. At least half the colors remain
available at every step. Each color class is a matching, so its
owned lines are parallel and its owner labels distinct. A fixed
row's load is dominated by a binomial variable of mean $2Cq^3$;
Chernoff bounds give load at most $4Cq^3$ simultaneously.
For $16C$ prescribed parallel lines, the probability that a row
owns all of them is at most $(2C/q)^{16C}$, yielding the stated
nonparallel intersection bound after a union bound.

Independently for each row, sample $C$ plane directions without
replacement and injectively assign labels to their $Cq^3$ planes,
avoiding owner labels. At every unowned point exactly one plane
from each class contains it. The expected multiplicity of a fixed
label--plane pair is at most $3/q$; a binomial upper tail bounds all
multiplicities by $200$. The probability of $t\le8$ prescribed
row--label incidences is at most $(4Cq^3/M)^t$. A union bound on
the six-to-eight incidences witnessing two overlap neighbors has
failure $O(C^5q^{39}/M)=o(1)$. For fixed $(a,L)$, the number of
rows whose alias plane contains $L$ has mean
$(1+o(1))q(q+1)$ and independent row indicators. Chernoff bounds
give at least $3q^2/4$ for every pair; at most $200$ have owner
direction parallel to $L$, leaving $q^2/2$.

For expansion, fix $z,T,R$. An alias neighborhood consists of
$C$ pencils $P(a,\Pi)$ of $q+1$ roles. Conditional on the pencils
already exposed, the next pencil contributes at least
$\gamma L/2$ expected fresh roles, where $L=1$ or $C^{2m}$
is the thinning's intensity. It contributes at most $B$.
The bounded-variable Chernoff inequality therefore bounds the
probability that a row has fewer than
$\ell=\gamma CL/4$ fresh neighbors by
\[
 \exp\!\left(-\frac{\gamma CL}{16B}\right).
\]
For a fixed $R$, rows are independent. Union-bounding over at
most $2^n$ sets $R$ and $\binom{H}{\lceil\beta n\rceil}$ row sets
proves the claim: for $A=\gamma CL/(16B)$, both regimes satisfy
$A\ge2\log(eH/(\beta n))$ and $\beta A\ge4\log2$.
The light and heavy values of $A$ are respectively
$\Omega(m^{25}/\log m)$ and $\Omega(m^{26})$.
Chernoff bounds control the pencil
caps. For each $z$, the joint design/thinning failure is at
most $q^{-300}$. Markov's inequality makes the conditional
thinning failure at most $q^{-20}$ for all but a
$q^{-280}$ fraction of designs. Union-bounding over the
$q^5$ query geometries and both schemes fixes one design
with all stated properties.
\end{proof}

Let $\mathcal G_m$ be the row class on
$Z\times[C]^H$: on an owned line the row emits its forced owner
label; elsewhere its own head coordinate selects one eligible
alias. Give different components disjoint label sets, extend each
row by a common background label outside its component, and include
the all-background hypothesis. This defines the countable class
$\mathcal G$. Finite components and the background limit make it
closed and pointwise finite-range.

\begin{lemma}\label{lem:fig-geometric-proper}
$\mathcal G$ has an interpolating proper PAC learner with sample
complexity $O(\epsilon^{-1}\log(1/\delta))$.
\end{lemma}
\vspace{-0.4 cm}
\begin{proof}
Within a component, two observed public labels leave the target
and at most its one fixed neighbor in the range-overlap graph;
choose any consistent row. With one label $a$, inspect the affine
span of the sample geometry. On a plane, choose a consistent
alias row for that label--plane pair (there are at most $d_0$).
On a line $L$, choose $o(a,L)$; at one point choose the owner
of a fixed line through it. Every choice interpolates.
If a bad fixed competitor survives, its probability is at most
$(1-\epsilon)^m$. The probability that all sample labels equal
a label of mass at most $1-\epsilon/2$ is at most
$(1-\epsilon/2)^{m-1}$ after summing over labels.
For the remaining possible heavy label, an owner selected from a
sample span can err substantially only when the sample missed
substantial mass outside that span. While that mass exceeds
$\epsilon/2$, each observation raises affine rank with probability
at least $\epsilon/2$; three raises span a plane. Thus the failure
probability is bounded by a $\operatorname{Bin}(m,\epsilon/2)$
lower tail. Combining these estimates gives
$(d_0+12)e^{-m\epsilon/4}$.
For the countable union, condition on the number of observations
from the target component and apply the same bound; a component
of mass at most $\epsilon$ is harmless and a heavier missed
component has probability at most $e^{-m\epsilon}$.
\end{proof}

\begin{theorem}\label{thm:figure-geometric}
The class $\mathcal G$ is in $\mathsf C$ but in neither
$\mathsf W$ nor $\mathsf{WL}$. Its weighted-local lower bound holds
even if the coefficient may depend on the test point.
\end{theorem}
\vspace{-0.4 cm}
\begin{proof}
For each role $(a,L)$ use owner $o(a,L)$ as target and draw geometry
uniformly on $L$, with independent uniform heads. All labels equal
$a$. Average over a uniform query geometry $z$, a uniform line
$L\ni z$, and a uniform label $a$; this is the same mixture of
target-distribution pairs. A repeated query geometry has probability
at most $m/q$, so condition on all training geometries differing
from $z$.

Fix any local penalty and coefficient rule. At each query, refine
penalty ties by a fixed enumeration. A selected row with $k$
agreements must be the first in this order among rows with at least
$k$ agreements; otherwise an earlier row has no larger empirical
loss or penalty. At zero coefficient the same rule holds with $k=m$.
Thus every coefficient's correct choice is among at most $m+1$
order records. A zero-agreement row is the order's first row,
correct for at most one of the $M$ uniform labels.

For positive agreement, a correct row is either a compatible alias
or the true owner. To bound aliases, lightly thin the $N$ query
roles and scan rows in penalty order, marking each role when its
first alias with at least $k$ agreements arrives. Except on the
auxiliary failure event of \cref{lem:fig-geometry}, a row has at
least $\ell$ fresh roles unless it belongs to fewer than $\beta n$
bad rows. At most $\ell/4$ neighbors are affected by forced owner
predictions, while a query can be correct in one pencil of at most
$B$ retained roles. The expected correct first records, divided
by $n$ and then unthinned, are bounded uniformly in $k$ by
\[
 A_k\le
 \frac{16t_0}{3\gamma C}+\beta t_0+2\gamma
 +O(q^{-19}),\qquad t_0=\lceil1000\log q\rceil .
\]
The affected-neighbor bound follows from the $16C$ line
intersection: its mean is at most $64mC^2\theta+C\theta$,
and a Chernoff bound at threshold $\ell/4$, unioned over rows,
has probability at most $q^{-200}$.

For a correct owner, use the heavy thinning and scan fully agreeing
aliases before each owner arrives. With $p_m=C^{-m}$, a good row
marks at least $p_m(3\ell/4)$ expected fresh roles. At most
$\beta n$ rows are bad and, after the scan covers
$(1-2\gamma)n$ roles, at most $2\gamma n$ remain. Unthinning gives
\[
 O_m\le
 \frac{16}{3\gamma C^{m+1}}+\beta+2\gamma
 +O(q^{-19})
\]
for the chance that a true owner wins its role. Combining the
zero record, $m$ alias thresholds, owner bound, and fresh-query
exception yields
\[
 \Pr[\text{correct prediction}]
 \le \frac mq+\frac1M+m\max_k A_k+O_m=o(1).
\]
The bound is uniform over query orders and coefficients. Hence
some fixed role in each component defeats an induced pointwise
selector at arbitrarily large sample sizes. When the penalty is
global, this selector chooses one global row, giving the
weighted-SRM failure directly. Proper learnability is
\cref{lem:fig-geometric-proper}.
\end{proof}

\begin{lemma}\label{lem:fig-tagged-geometric}
Add an independent tag $t\in[n_j]$ to component $j$, where
$n_j=\max\{4,\lfloor\sqrt j\rfloor\}$, and for each public label
$a$ and tag $i$ add a fallback predicting $a$ except at tag $i$,
where it emits a unique poison label. Call the class
$\mathcal G^{\mathrm{tag}}$. Then
\[
 \mathcal G^{\mathrm{tag}}\in
 \mathsf C\cap\mathsf W\cap\mathsf{WL}\setminus\mathsf L.
\]
\end{lemma}
\vspace{-0.4 cm}
\begin{proof}
The fallback family has graph dimension at most four: on $t$
points its agreement patterns with arbitrary target labels
number at most $4t+2$. If a fallback is consistent, choose it;
otherwise run the original consistent proper learner ignoring
tags. Uniform convergence of the fallback family proves
$\mathsf C$ membership.

Give the background penalty zero, fallbacks distinct penalties
in $[1,1.1)$, and original rows distinct penalties in $[2,2.1)$.
For sample size $r$, put
$J_r=\lfloor r^{1/4}\rfloor$ and
$b_r=\min\{1/8,r^{-1/16}\}$. Choose
$\lambda(S)\in(b_r,2b_r)$ outside the finite set of objective-tie
values among represented components. Unrepresented components
are dominated by the background, so the global minimum attains
and is unique. A fallback target is compared with a fallback
of empirical loss zero, forcing every minimum into the
uniformly convergent fallback family. For an original target,
every original minimum has empirical loss at most $0.1\lambda$;
small components $j\le J_r$ have $\log$ cardinality
$O(J_r\log J_r)=o(r)$. In a larger component, any different
original row shares at most one public target label $a$.
An $a$-fallback chosen at a least-observed tag has objective
at most $\widehat q_a+1/n_j+1.1\lambda$, while that row costs
at least $\widehat q_a+2\lambda$. Since
$0.9b_r n_{J_r+1}\to\infty$, it cannot minimize.
Only the target and its one fixed overlap companion remain;
a bad companion cannot have empirical loss $O(b_r)$ with
nonvanishing probability. This proves weighted-SRM learning
for its unique minimizer, which also proves weighted-local
learning using the same test-independent score.

On the all-$a$ hard samples for $\mathcal G_j$, every tag is
observed with probability at least $1-n_je^{-j/n_j}\to1$
at sample size $j$. All fallbacks then leave the version
space. Conditional on this event, the geometric observations
retain their original law; \cref{thm:figure-geometric} defeats
any pointwise hard-local order, including one that sees the
test tag.
\end{proof}

Marked disjoint sums preserve the positive PAC and consistent
proper properties by identifying the active summand from a
nonbackground label and applying its learner to that subsample.
They preserve hard negative results by restricting the hard
distribution to the offending summand. In particular
$\mathcal G\oplus\Hpoison$ has signature
$(\mathsf C,\mathsf W,\mathsf{WL})=(0,0,0)$.
For $\mathcal G^{\mathrm{tag}}\oplus\Hinc$, use the tagged
penalty on its domain and the incidence hard-local penalty
on the other. Shift both into bounded positive bands, set
off-summand penalties to zero, and multiply the active
component's coefficient by its sample fraction. Active
tagged comparators have normalized objective below one;
incidence comparators have objective below $1/t$ on an
active sample of size $t$. Foreign/background rows incur
the active sample fraction in empirical loss. At a query
outside the active summand, active rows predict background
and have zero penalty. All-background samples predict
background at every query. This proves
weighted-local learning of the sum; incidence excludes
weighted SRM, and the tagged component excludes hard local
regularization. Thus its signature is $(1,0,1)$.

\subsubsection{Weighted SRM without weighted local regularization}
\label{app:figure-layered}

For the remaining cells, amplify the finite geometric design in
component $j$ to
\[
 m=j,\quad C=\lceil e^{j^2}\rceil,\quad
 R=\lceil j^{1/4}\rceil+8,\quad
 q=2^{\lceil10j\log_2 C\rceil},\quad M=q^{100}.
\]
The construction of \cref{lem:fig-geometry} still applies:
$C=o(q)$, $C\gg\log q$, and $C^{2j}\le q^{1/5}$.
For thinning parameter $\sigma=C^{-1/4}$, the light and heavy
row exponents satisfy
$A_{\rm light}\sigma=\Omega(C^{1/2}/\log q)$ and
$A_{\rm heavy}\sigma=\Omega(C^{1/2})$.
These dominate the $O(\log q)$ union-bound entropy.
The owner-load, alias-multiplicity, overlap, and affected-neighbor
bounds therefore retain their stated margins; all record and
blocker errors below are $e^{-\Omega(j^2)}$.

Keep the finite base set $\mathcal B_j$. For every epoch $n\ge1$
and public label $a$, use one owner-label tag $t_{n,a}\in[R]$
shared by all bases owning $a$. For every epoch $n\ge0$ and base
$h$, use an independent head
$u_{n,h}\in\{\mathsf E\}\sqcup[C]$. The old row $h_{n,i}$,
$i\in[R]$, predicts $(n,a)$ on its owned $a$-line unless
$t_{n,a}=i$, in which case it predicts its private E-parent
label $\eta_h$. Off the owned union, $u_{n,h}$ selects
$\eta_h$ or one eligible alias $(n+1,a)$. Include
$f_h\equiv\eta_h$. These rows form $\mathcal H_j^-$.

For $\mathcal H_j^+$, also include $c_{h,n}$, which is
unpoisoned on every owned line with label $(n,a)$ and off
the owned union uses $u_{n-1,h}$ to select $\eta_h$ or an
alias $(n,a)$. Thus its owner and alias labels share epoch $n$.
Different components have disjoint nonbackground labels; extend
their rows by a common outer background and add its constant row.

The ambient coordinate product can be restricted to
$X_j=A_j\cup F_j$. The set $A_j$ is finite: it repeats a single
tag vector $(t_a)_a$ and head vector $(u_h)_h$ at every epoch.
The set $F_j$ has all tags one and all heads $\mathsf E$
except at finitely many epoch--base coordinates. It is countable.
Hence both full domains, label sets, and classes are countable.
The finite-support statistical laws below lie in $A_j$; the
attainment argument uses only $F_j$. To see closedness, for each
base $h$ choose an unowned geometry and a diagonal point
whose head for $h$ is alias one at all epochs and whose other
heads are $\mathsf E$. Values on these finitely many anchor
points restrict every base to at most two epochs; finite
consistency then leaves only finitely many rows and forces
an actual row. The outer background is the only limit
escaping across components.

\begin{lemma}\label{lem:fig-layered-positive}
Both $\mathcal H^-$ and $\mathcal H^+$ are in
$\mathsf P\cap\mathsf W$, and $\mathcal H^+\in\mathsf C$.
\end{lemma}
\vspace{-0.4 cm}
\begin{proof}
Give E parents and the outer background penalty zero, old rows
penalty $2^{-n}$, and completion rows penalty one. In a component
of sample size $s$, use coefficient one on a sole E label;
use $3\cdot2^n/R$ on a sole K label $(n,a)$ when
$s\le\tau_j:=\lceil j^2(1+\log(2D_j))\rceil$ and the
sample geometry spans at most a line, where
$D_j=(2R+1)|\mathcal B_j|$; otherwise use $1/(4s)$.
At the last coefficient all minima interpolate. In the early
line branch a current owner has a tag variant of loss at most
$1/R$ and objective at most $4/R$; previous aliases and
completion rows cost at least $6/R$, while nonincident rows
have loss one. Thus all minima are current owners of loss
at most $1/R$.

For a fixed K label, owner correctness sets have growth
at most $(s+1)^3$; one base's tag variants have at most
$s+1$ loss patterns. Two observed labels leave at most
ten fixed base/epoch families by the degree-one overlap
property. A noncollinear one-K sample identifies its
alias plane and at most $d_0$ candidates. Late one-K
interpolation has failure at most $D_j e^{-s\epsilon/2}$,
and $\sup_{j:s>\tau_j}\log(2D_j)/s\to0$.
These bounds and the missed-mass estimate for a heavy
single label prove PAC uniformly over components and targets.
Every positive objective attains: a row with no observed K
label is dominated by its zero-price E parent, leaving
finitely many competing epochs. On a marked union, multiply
the active coefficient by its sample fraction; missed
heavy components obey the same binomial bound.

For $\mathcal H^+$, a consistent proper learner uses the
E parent on a sole E label, any consistent row on
multiple labels, and a consistent alias on a one-K plane.
On a one-K line $L$, it uses $c_{o(a,L),n}$; on a point,
it uses the completion owner of a fixed line through it.
The fixed-family bounds above handle the multi-label
and plane branches. While mass outside the sample span
exceeds $\epsilon$, each new point raises affine rank
with probability at least $\epsilon$, so three raises
span the support. A binomial lower tail controls the
line and point branches. All choices interpolate.
\end{proof}

\begin{lemma}[Hidden realizability]\label{lem:fig-planting}
For a fixed role $(a,L)$ in a layered component, the artificial
law $P_{a,L}$ with uniform geometry on $L$, independent uniform
heads and tags, and constant label $(n+1,a)$ is within
\[
 \Delta_k\le\frac{16kC}{q}
   +\frac{e^{32kC/q}}{2\sqrt{B(C+1)^{-k}}}
   =e^{-\Omega(j^3)}
\]
in total variation on $k\le j+2$ iid queries of a mixture of
fixed-target realizable laws. Here $B\ge q^2/2$ is the number
of nonparallel bases with an $a$-alias plane containing $L$.
\end{lemma}
\vspace{-0.4 cm}
\begin{proof}
For each such base $h$, condition the uniform law on
$z\notin O_h$ and on its epoch-$n$ head selecting alias $a$;
the old row $h_{n,i}$ then predicts $(n+1,a)$ everywhere.
Choose $h$ uniformly before drawing $k$ iid queries.
Put $p=(C+1)^{-1}$ and
$\omega_h=|L\cap O_h|/q\le16C/q$.
Conditional on query geometries, the density of this
mixture relative to $P_{a,L}^{\otimes k}$ is the average of
\[
 p^{-k}\frac{\one\{z_\ell\notin O_h\ \forall\ell\}}
 {(1-\omega_h)^k}
 \one\{u_{n,h,\ell}\text{ selects }a\ \forall\ell\}.
\]
The head events are independent over bases, each of probability
$p^k$. The mixture density has conditional variance at most
$e^{64kC/q}/(Bp^k)$; its conditional mean differs from one
in $L^1$ by at most $32kC/q$. Half their combined $L^1$
bound gives $\Delta_k$. The parameter choice makes both
terms $e^{-\Omega(j^3)}$.
\end{proof}

At sample size $j$, an artificial constant-K sample sees every
owner-label tag with probability $1-o(1)$ and has two distinct
geometries with probability $1-o(1)$. No current owner variant
interpolates: a different line meets the true line at at most
one point, and every tag variant of its owner is poisoned
somewhere. Every remaining proper interpolator is a
previous-epoch alias, whose independent fresh head predicts
the label with probability at most $1/(C+1)$. The planted
mixture is realizable, so the artificial sample is
unrealizable with probability at most $\Delta_{j+1}$.
The comparison
in \cref{lem:fig-planting} transfers this constant error to a
mixture of fixed-target realizable laws. Consequently
$\mathcal H^-\notin\mathsf C$.

\begin{lemma}[Strict penalty floor]\label{lem:fig-floor}
If an attained weighted-local objective with one shared
coefficient universally learns either layered class, then,
outside finitely many components, every query $x$ satisfies
\[
 \delta(x):=\inf_{h\text{ active}}\psi(h,x)
 -a(x)>0,\qquad
 a(x):=\inf_{g\text{ foreign/background}}\psi(g,x).
\]
\end{lemma}
\vspace{-0.4 cm}
\begin{proof}
If infinitely many components contain an active row of price
at most $a(x)$ at some query, put the outer-background target
on $N$ such queries in distinct components. On an $m$-sample
at least $N-m$ components are unseen; there a wrong active
interpolator beats or ties every correct row for every
coefficient. Choosing such wrong minima and taking $N=m^2+1$
contradicts PAC.

In a remaining component, suppose all active prices exceed
$a(x)$ but their infimum equals it. Finitely many bases and
tag variants give cofinally cheap rows of one type at
distinct source-head epochs. Split them into frozen and
reserve subsequences. At a geometry outside that base's
owned union, freeze every nonreserve head to $\mathsf E$.
For sample size $m$, reserve one source for each $m$-subset
$T$ of an $N=m^3+1$ point support; set its head to
$\mathsf E$ exactly on $T$ and to an alias value elsewhere.
All support points lie in the countable finite-deviation
domain. Frozen cofinal rows interpolate the E-parent target,
and their objectives approach $\lambda(S)a(x)$ from above.
No active row attains this infimum and foreign rows have
empirical loss one, so any positive coefficient violates
attainment. At the forced coefficient zero, the reserve
row for a padded sample support interpolates yet errs on
at least $N-m$ points. A fixed bad induced selector again
contradicts PAC.
\end{proof}

\begin{theorem}\label{thm:figure-layered}
The countable closed classes above satisfy
\[
 \mathcal H^-\in\mathsf P\cap\mathsf W
       \setminus(\mathsf C\cup\mathsf{WL}),\qquad
 \mathcal H^+\in\mathsf C\cap\mathsf W
       \setminus\mathsf{WL}.
\]
Every hard law used below has finite support. The weighted-local
lower bound uses one sample-dependent coefficient shared by the
test queries.
\end{theorem}
\begin{proof}
Only weighted-local nonlearnability remains, by
\cref{lem:fig-layered-positive,lem:fig-planting}.
Fix a proposed score $\psi(h,x)$ and shared coefficient rule
$\lambda(S)$, and use the selector that breaks objective ties by
smaller penalty and then enumeration. At a positive coefficient
there are at most $m+1$ penalty values among objective minima.
On a constant-K sample with coefficient zero, a current owner
has positive agreement, so the empirical-minimum stratum is
finite. Thus this canonical selector is defined on the artificial
samples used below as well.

By \cref{lem:fig-floor}, take a sufficiently large component
with a strict positive floor at every query. Center penalties
at the foreign infimum $a(x)$ and put
\[
 \pi_n(h,x)=\min_{i\in[R]}\bigl(\psi(h_{n,i},x)-a(x)\bigr).
\]
At geometry $z$, let $\mathcal O(z)$ be the $N$ owner bases,
in bijection with the roles $(a,L\ni z)$. With
$\alpha=10^{-4}$, let $U_n(x)$ be the
$\lceil\alpha N\rceil$-th order statistic of
$\{\pi_n(h,x):h\in\mathcal O(z)\}$.
It is positive and finite. At least $\alpha N$ bases have
price at most $U_n$, while fewer than $\alpha N$ have price
strictly below it.

We record the two order-uniform estimates needed below.
On an artificial all-$(n,a)$ sample from a line $L$, a previous
epoch row can predict the current label only as an alias.
Its tag variants have identical empirical agreement, so replace
them by their cheapest variant. The light-thinning scan in
\cref{thm:figure-geometric}, with head probability $1/(C+1)$,
gives
\[
 \zeta_{\rm alias}\le
 m\left(\frac{16t_0}{3\sigma C}+\sigma t_0+2\sigma\right)
 +O(mq^{-19})+\frac mq
 =e^{-\Omega(j^2)},\quad t_0=\lceil1000\log q\rceil.
\]
For any pool $A\subseteq\mathcal O(z)$ of at least
$\alpha N$ previous-epoch bases, scan their fully agreeing
aliases under the heavy thinning. With
$p=(C+1)^{-m}$ and fresh-neighbor threshold
$\ell=\sigma C^{2m+1}/4$, each good row adds at least
$3p\ell/4$ expected marks. The scan processes at most
$n_T/(3p\ell/4)$ good rows in expectation; expansion
allows at most $\sigma n_T$ bad rows. If the pool is
exhausted early, at least
$\alpha N-\sigma n_T-1\ge\alpha N/2$ good rows were
processed. Markov's inequality and removal of the
independent thinning give an unblocked-role probability
\[
 \zeta_{\rm block}\le
 K\left(\sigma+\frac1{\alpha p\sigma Cq}\right)
 +O(q^{-19})+\frac mq
 =e^{-\Omega(j^2)}
\]
for an absolute $K$. This holds uniformly over query orders
and pools. Apply it to the bases with
$\pi_{n-1}(h,x)\le U_{n-1}(x)$; their fully agreeing rows
are wrong previous owners with centered price at most
$U_{n-1}(x)$.

Put $\eta=1/(4R)$, $A_0=\log(1/\eta)$, and
$r_0=\eta/(10A_0)$. Let $P$ be the finite reference law
with uniform geometry, independent tags across public labels,
and independent heads across bases, each vector repeated at
every epoch. For $n\ge2$ define the score event
\[
 B_n=\{x:U_n(x)>(1-\eta)U_{n-1}(x)\}.
\]
Suppose $\gamma=P(B_n)\ge r_0$. For each line $L$, write
$\beta_L=P_L(B_n)$; retain lines with
$\beta_L\ge\gamma/2$, choose one with prior proportional
to $\beta_L$, and condition its iid sample and query on $B_n$.
Conditioning costs $\log(1/\beta_L)$ in KL divergence.
Tensorization over the independent owner-label tags shows
that at most $2R^2\log(1/\beta_L)$ labels have a tag atom
below $1/(2R)$, by Pinsker's inequality. For every other
label, all empirical tag
frequencies exceed $\eta$, except with probability
$Re^{-m/(16R)}$; geometric collisions cost
$O(m^2/(q\gamma))$.

The line-prior factor cancels the one-query conditioning
normalizer. Except on query mass $2\kappa$, with
$\kappa=1/100$, at least $\kappa N$ original roles
remain available at the fresh geometry. Indeed, the
retained-line fraction has conditional mean at least
$1/2$, and the new query law weights it proportionally.
The conditioned
joint law of $m+1$ queries has density at most
$(2/\gamma)^{m+1}$ relative to the unconditioned
role experiment. Hence the alias-record and missing-blocker
probabilities remain $o(1)$, since their
$e^{-\Omega(j^2)}$ bounds dominate
$(2/r_0)^{m+1}=e^{O(j\log j)}$.
The tag bound uses conditional Chernoff directly.

For a correct current owner variant $g$ of empirical loss
$e_g\ge\eta$ and a cheap interpolating previous row $h$,
optimality against the foreign infimum and against $h$
gives, with centered prices $u_g,u_h$,
\[
 \lambda u_g\le1-e_g,\qquad
 e_g+\lambda u_g\le\lambda u_h,\qquad
 u_g\le(1-\eta)u_h\le(1-\eta)U_{n-1}(x).
\]
On $B_n$ this puts its owner strictly below $U_n(x)$,
which happens for fewer than $\alpha N$ roles.
Conditional accuracy is therefore at most
$2\kappa+\alpha/\kappa+o(1)=3/100+o(1)$.
The mixture in \cref{lem:fig-planting} remains a mixture
of iid realizable laws after conditioning; its total
variation cost is at most
$2\Delta_{m+1}(2/r_0)^{m+1}=o(1)$.
Universal PAC learning consequently forces
\begin{equation}\label{eq:fig-small-resets}
 P(B_n)<r_0\qquad(n\ge2)
\end{equation}
uniformly in the epoch.

Under the unconditioned line experiment, use one sample
and two independent fresh queries $X,Y$ on its line.
Outside the negligible alias, blocker, tag, and
lower-$\alpha$ role events, correct selection gives
$\lambda(S)U_n(X)\le1$ and
$\lambda(S)U_{n-1}(Y)\ge\eta$.
The same coefficient appears at both queries, so
\[
 P[\log U_n(X)>A_0+\log U_{n-1}(Y)]
 \le2\alpha+o(1)<1/25.
\]
Two uniform points on a random affine line differ from
two iid global geometric points in total variation by at
most $1/q$. The displayed bound therefore holds for
independent $X,Y\sim P$.

For completeness, let $Z_n=\log U_n$, let $q_n(t)$ be its
lower $t$ quantile, and put
$H_n=\int_{1/5}^{3/5}q_n(t)\,dt$.
The cross-query bound implies
$q_n(4/5)\le A_0+q_{n-1}(1/5)$.
If every reset event in \eqref{eq:fig-small-resets} had
mass below $\rho$, where
$\rho\le1/10$ and $\rho A_0<2d/5$ for
$d=-\log(1-\eta)$, the one-space quantile shift would give
\[
 H_n-H_{n-1}\le-2d/5+\rho A_0
   +\rho\bigl(q_{n-2}(1/5)-q_{n-1}(1/5)\bigr).
\]
Summation makes $H_n$ tend to $-\infty$.
But the strict floor gives $U_n(x)\ge\delta(x)>0$
for all $n$, so all lower $1/5$ quantiles, and hence
all $H_n$, have a common finite lower bound.
Thus some reset has mass at least
$\min\{1/10,d/(5A_0)\}>r_0$, contradicting
\eqref{eq:fig-small-resets}. This excludes
$\mathcal H^-$ from $\mathsf{WL}$.

For $\mathcal H^+$, any correctly selected completion
row of epoch $n$ is a record within that finite completion
copy. Apply the full geometric light-alias and heavy-owner
bounds to this copy; its E head changes
$C^{-m}$ to $(C+1)^{-m}$, at most a factor two.
In particular, its correctness probability is at most
\[
 \zeta_{\rm comp}\le \frac mq+\frac1M+
 m\left(\frac{16t_0}{3\sigma C}+\sigma t_0+2\sigma\right)
 +\frac{32}{3\sigma C^{m+1}}+3\sigma+O(mq^{-19})
 =e^{-\Omega(j^2)}.
\]
This bound survives the
same reset conditioning. The old owner prices, blocker
pool, tag entropy, scalar window, and quantile argument
are unchanged. Hence $\mathcal H^+\notin\mathsf{WL}$.
\end{proof}

In summary, the preceding results and main-text results give the
membership patterns in the following table. All classes listed are
improperly learnable, so the $\mathsf I$ coordinate is suppressed.
A one denotes membership and a zero nonmembership.
\begin{center}
\setlength{\tabcolsep}{4pt}
\begin{tabular}{lccccccc}
\hline
Witness & $\mathsf E$ & $\mathsf C$ & $\mathsf W$ &
 $\mathsf P$ & $\mathsf B$ & $\mathsf L$ & $\mathsf{WL}$\\
\hline
Any nonempty finite class                 &1&1&1&1&1&1&1\\
Completed Cantor $\mathcal C^+$           &0&1&1&1&1&1&1\\
Incidence $\Hinc$                         &0&1&0&1&1&1&1\\
Poisoned Cantor $\Hpoison$                &0&0&1&1&1&1&1\\
$\Hpoison\oplus\Hinc$                     &0&0&0&1&1&1&1\\
Tagged orientation $\mathcal H_{\mathrm{tag}}$ &0&0&1&1&1&0&1\\
$\mathcal H_{\mathrm{tag}}\oplus\Hinc$    &0&0&0&1&1&0&1\\
$\mathcal G^{\mathrm{tag}}$             &0&1&1&1&1&0&1\\
$\mathcal G^{\mathrm{tag}}\oplus\Hinc$  &0&1&0&1&1&0&1\\
Uncompleted Cantor $\mathcal C^-$         &0&0&0&0&1&1&1\\
$\mathcal H_{\mathrm{tag}}\oplus\mathcal C^-$ &0&0&0&0&1&0&1\\
No-envelope tree $\mathcal H_{\mathrm{tree}}$ &0&0&0&0&0&1&1\\
$\mathcal H_{\mathrm{tag}}\oplus\mathcal H_{\mathrm{tree}}$
                                          &0&0&0&0&0&0&1\\
Hypergraph $\mathcal H_{\mathrm{hg}}$     &0&0&0&0&1&0&0\\
$\mathcal H_{\mathrm{hg}}\oplus\mathcal H_{\mathrm{tree}}$
                                          &0&0&0&0&0&0&0\\
$\mathcal G$                           &0&1&0&1&1&0&0\\
$\mathcal G\oplus\Hpoison$             &0&0&0&1&1&0&0\\
$\mathcal H^-$                         &0&0&1&1&1&0&0\\
$\mathcal H^+$                         &0&1&1&1&1&0&0\\
\hline
\end{tabular}
\end{center}
These nineteen patterns exhaust the Boolean assignments within
$\mathsf I$ permitted by \cref{lem:figure-inclusions}.
Thus every cell displayed in \cref{fig:multiclass-containments}
is nonempty.

\section{Omitted Proofs for Empirical Noninterpolation}
\label{app:empirical-noninterpolation}

\subsection{Tunable lower bound}
\label{app:tunable-poison}

We now prove \Cref{thm:empirical-lower}.  Fix a nonnegative sequence $a_m=o(m)$ and define
\[
    r_m
    :=
    \max\set{
        \lceil a_m\rceil,
        \left\lceil\log\log(m+e^e)\right\rceil
    }.
\]
Then $r_m\to\infty$, $r_m=o(m)$, and $r_m\ge a_m$.  Put
\[
    N_m
    :=
    \max\left\{
        1,
        \left\lfloor
            \min\set{e^{r_m/32},\sqrt{m/r_m}}
        \right\rfloor
    \right\},
    \qquad
    K_m:=m^3.
\]
These parameters satisfy
\begin{equation}
\label{eq:tunable-parameters}
    N_m\to\infty,
    \qquad
    \frac{N_mr_m}{m}\to0,
    \qquad
    N_me^{-r_m/8}\to0.
\end{equation}
Indeed, both $e^{r_m/32}$ and $\sqrt{m/r_m}$ diverge; moreover,
$N_mr_m/m\le\sqrt{r_m/m}$ and
$N_me^{-r_m/8}\le e^{-3r_m/32}$ for all sufficiently large $m$.

For each block index $j\ge1$, let $P_j=\{p_{j,1},\ldots,p_{j,N_j}\}$ be a marker set, let $C_j$ be a core of size $2K_j+N_j$, and put $X_j=P_j\sqcup C_j$.  Thus $|X_j|=2K_j+2N_j$, and every half-set has size $K_j+N_j$.  Introduce global labels $\starlabel$ and $\dollarlabel$, a private label $\lambda_{j,A}$ for every half-set $A\subseteq X_j$, and a poison label $\rho_{j,i}$ for every marker.  Include the all-$\dollarlabel$ hypothesis, every Cantor hypothesis
\[
    c_{j,A}(x)
    =
    \begin{cases}
        \starlabel, & x\in A,\\
        \lambda_{j,A}, & x\in X_j\setminus A,\\
        \dollarlabel, & x\notin X_j,
    \end{cases}
\]
and every fallback
\[
    s_{j,i}(x)
    =
    \begin{cases}
        \rho_{j,i}, & x=p_{j,i},\\
        \starlabel, & x\in X_j\setminus\{p_{j,i}\},\\
        \dollarlabel, & x\notin X_j.
    \end{cases}
\]
Let $\Hc_a$ be the resulting class.

\begin{lemma}
\label{lem:tunable-learnable}
The class $\Hc_a$ is countable, closed, pointwise finite-range, and properly PAC learnable.
\end{lemma}

\begin{proof}
Countability and pointwise finite range follow because every block is finite and supports only finitely many hypotheses.  For closedness, let $f$ be finitely consistent with $\Hc_a$.  If $f\equiv\dollarlabel$, then $f=h_{\dollarlabel}$.  Otherwise choose $x\in X_j$ with $f(x)\neq\dollarlabel$.  Two-point restrictions force $f$ to equal $\dollarlabel$ outside $X_j$, since no nonbackground hypothesis is active on two blocks.  Finite consistency on the finite block $X_j$ then supplies one hypothesis agreeing with $f$ there, and hence everywhere.

For learnability, choose cutoffs $J_s\to\infty$ sufficiently slowly that
\[
    \log|\Hc_{a,\le J_s}|=o(s)
    \qquad\text{and}\qquad
    \inf_{j>J_s}N_j\to\infty,
\]
where $\Hc_{a,\le J_s}$ contains the background hypothesis and all hypotheses supported on the first $J_s$ blocks.  Such a sequence exists because every initial class is finite and $N_j\to\infty$.

On a sample of size $s$, use finite-class consistent learning on blocks $j\le J_s$.  On a larger block, use the same reveal-or-least-observed-marker learner as in the proof of \Cref{thm:no-erm}: a private or poison label identifies the target, while an all-$\starlabel$ sample triggers the fallback associated with a least-observed marker.  The initial class is controlled by the realizable finite-class bound because its logarithmic size is $o(s)$.  Uniformly over the tail blocks, the number of markers tends to infinity, so \Cref{lem:least-observed-marker} controls the selected marker mass.  The remaining failure events are precisely that the sample misses a private-label region or a target poison point of substantial mass, each of which has exponentially small probability.  This proves distribution-free proper learnability.
\end{proof}

\begin{proof}[Proof of \Cref{thm:empirical-lower}]
Fix an arbitrary proper PAC learner $A$ for $\Hc_a$ and a sufficiently large sample size $m$.  Work in block $X_m$.  Choose $T\subseteq C_m$ uniformly among the $K_m$-subsets, put $A_T:=P_m\cup T$, and take target $c_{m,A_T}$.  Define $D_T$ by
\[
    D_T(p_{m,i})=\frac{r_m}{m}
    \qquad(i\in[N_m]),
\]
and distribute the remaining mass $1-q_m$ uniformly on $T$, where
\[
    q_m:=\frac{N_mr_m}{m}=o(1).
\]
Every point in the support is labeled $\starlabel$.

Let $E_m$ be the event that every marker appears at least $r_m/2$ times.  Each marker count has mean $r_m$, so a Chernoff bound and \eqref{eq:tunable-parameters} give
\[
    \P(E_m^c)
    \le
    N_me^{-r_m/8}
    =o(1).
\]
On $E_m$, every fallback $s_{m,i}$ supported on the active block makes at least $r_m/2$ training mistakes.

We next show that every nonfallback output has constant population error with probability $1-o(1)$ over the random choice of $T$, the sample, and the learner's internal randomness.  The background hypothesis and all hypotheses supported on another block have error one.  Suppose the output is a Cantor hypothesis $c_{m,B}$, and put $B_C:=B\cap C_m$.  Since $|B|=K_m+N_m$, one has $|B_C|\le K_m+N_m$.  Conditional on the marker observations, the learner sees at most $m=o(K_m)$ iid uniform draws from the hidden set $T$; the marker observations themselves carry no information about $T$.  Applying \Cref{lem:hidden-halfset} with $K=K_m$, $N=N_m$, and $C=C_m$ gives
\[
    |T\setminus B_C|\ge K_m/5
\]
with probability $1-o(1)$.  Therefore
\[
    \pop_{D_T}(c_{m,B},c_{m,A_T})
    \ge
    (1-q_m)\frac{|T\setminus B_C|}{K_m}
    \ge
    \frac17
\]
for all sufficiently large $m$.

Fix $\epsilon_0=\delta_0=1/20$.  Since $A$ PAC learns, for every fixed $T$ and all sufficiently large $m$,
\[
    \P_{S,A}\brac{
        \pop_{D_T}(A_m(S),c_{m,A_T})>\epsilon_0
    }
    \le\delta_0.
\]
Averaging over the random choice of $T$ and comparing with the preceding lower bound shows that the learner outputs an active-block fallback with probability at least $1-\delta_0-o(1)$ on average over $T$.  Hence there is a fixed $T$ for which this fallback probability is at least $3/4$.  For this fixed target-distribution pair, $\P(E_m)=1-o(1)$ because the marker marginal is independent of $T$.  Thus, for all sufficiently large $m$,
\[
    \P\brac{
        m\emp_S(A_m(S))\ge r_m/2
    }
    \ge\frac12.
\]
Since $r_m\ge a_m$, the theorem follows with a universal constant.
\end{proof}

\subsection{Pathological learners}
\label{app:pathological-learners}

The following examples explain why the upper bound in \Cref{thm:empirical-upper} is existential: an otherwise successful learner may behave arbitrarily on a sufficiently rare sequence of samples.

\begin{proposition}
\label{prop:full-training-error-pathology}
There is a proper PAC learner for the two constant binary hypotheses on $\N$ that misclassifies every training point on some realizable sample of every size.
\end{proposition}

\begin{proof}
Let $h_0\equiv0$ and $h_1\equiv1$.  On an $m$-sample, if the ordered unlabeled sequence is exactly $(1,2,\ldots,m)$, output the constant hypothesis opposite to the observed label; otherwise output the correct constant hypothesis.  The exceptional sample is misclassified completely.  For every marginal $D$,
\[
    \P[(X_1,\ldots,X_m)=(1,\ldots,m)]
    =
    \prod_{i=1}^m D(i)
    \le
    m^{-m}
\]
by AM--GM.  The additional failure probability therefore vanishes superexponentially.
\end{proof}

\begin{proposition}
\label{prop:optimal-pathological-learner}
The preceding pathology can coexist with optimal realizable sample complexity $\Theta(\epsilon^{-1}\log(1/\delta))$.
\end{proposition}

\begin{proof}
Let $\CX=\{0\}\cup\N$ and
\[
    h_0\equiv0,
    \qquad
    h_1(x)=\one\{x=0\},
    \qquad
    h_2\equiv1.
\]
A finite-class learner gives the upper bound $O(\epsilon^{-1}\log(1/\delta))$.  For the matching lower bound, place mass $2\epsilon$ on $0$ and the remaining mass on $1$.  Under targets $h_0$ and $h_1$, the labeled samples are identical unless $0$ is observed.  Conditional on missing $0$, no learner can identify both targets; for at least one of them it incurs error $2\epsilon$ with conditional probability at least $1/2$.  Since $0$ is missed with probability $(1-2\epsilon)^m$, the usual $\Omega(\epsilon^{-1}\log(1/\delta))$ lower bound follows.

Now modify an optimal finite-class learner only on the ordered sample $(1,2,\ldots,m)$: if all labels are zero, output $h_2$, and if all labels are one, output $h_0$.  The exceptional probability is at most $m^{-m}$ under every marginal, so the asymptotic sample complexity is unchanged, while the emitted hypothesis misclassifies every training point on the exceptional realizable samples.
\end{proof}

\end{document}